\def\ICRABuild{1}
\def\ICRACameraReady{1}
\ifdefined\ICRABuild
  \documentclass[letterpaper,10pt,conference]{ieeeconf}
  \IEEEoverridecommandlockouts
\else
  \documentclass[10pt]{article}
  \usepackage[margin=1in]{geometry}
\fi
\ifdefined\ICRABuild

\fi
\usepackage{amsmath,amssymb,amsthm,mathtools}
\usepackage{graphicx}
\usepackage{booktabs,multirow,makecell,array}
\usepackage{tikz}
\usetikzlibrary{arrows.meta,positioning,shapes.geometric,calc,fit,backgrounds,decorations.pathreplacing,patterns}
\usepackage{pgfplots}
\pgfplotsset{compat=1.18}
\usepackage[font=small,labelfont=bf]{caption}
\usepackage{subcaption}
\ifdefined\ICRABuild
\fi
\usepackage{xcolor}
\ifdefined\ICRABuild
  \let\labelindent\relax
\fi
\usepackage{enumitem}
\usepackage{algorithm}
\usepackage{algpseudocode}
\ifdefined\ICRABuild
\fi
\ifdefined\ICRABuild
  \makeatletter\let\NAT@parse\undefined\makeatother
\fi
\usepackage[numbers,sort&compress]{natbib}
\ifdefined\ICRABuild
  \usepackage[hidelinks]{hyperref}
\else
  \usepackage[affil-it]{authblk}

  \usepackage[colorlinks=true,linkcolor=blue!60!black,citecolor=blue!60!black,urlcolor=blue!60!black]{hyperref}
\fi
\usepackage[activate={true,nocompatibility},final,tracking=false,factor=1000,stretch=10,shrink=10]{microtype}\microtypesetup{expansion=false}

\ifdefined\ICRABuild
  \newenvironment{aspfigure}[1][!t]{\begin{figure*}[#1]}{\end{figure*}}
  \newenvironment{asptable}[1][!t]{\begin{table*}[#1]}{\end{table*}}
  \newenvironment{aspnarrowtable}[1][!t]{\begin{table}[#1]}{\end{table}}
  \newenvironment{aspalgorithm}[1][t]{\begin{algorithm*}[#1]}{\end{algorithm*}}
\else
  \newenvironment{aspfigure}[1][!t]{\begin{figure}[#1]}{\end{figure}}
  \newenvironment{asptable}[1][!t]{\begin{table}[#1]}{\end{table}}
  \newenvironment{aspnarrowtable}[1][!t]{\begin{table}[#1]}{\end{table}}
  \newenvironment{aspalgorithm}[1][t]{\begin{algorithm}[#1]}{\end{algorithm}}
\fi

\newcommand{\VN}{300}
\newcommand{\Vd}{32}
\newcommand{\VK}{8}
\newcommand{\VNd}{9{,}600}
\newcommand{\VmBits}{6{,}144}
\newcommand{\VkappaS}{0.547}
\newcommand{\VkappaSpct}{54.7}
\newcommand{\VoneOverK}{12.5}
\newcommand{\VStateRETMin}{0.0}
\newcommand{\VStateRETMax}{6.2}
\newcommand{\VUnifRETMin}{3.1}
\newcommand{\VUnifRETMax}{18.8}

\newcommand{\VRetrTRKMin}{4.2}
\newcommand{\VRetrTRKMax}{27.1}
\newcommand{\VStateTRKMin}{4.2}
\newcommand{\VStateTRKMax}{41.7}

\newcommand{\VCorGapMin}{-20.5}
\newcommand{\VCorGapMax}{+8.0}

\newcommand{\VThmTwoMin}{14.8}
\newcommand{\VThmTwoMax}{16.9}
\newcommand{\VwMin}{8}
\newcommand{\VwMax}{15}

\newcommand{\VnTotal}{80}
\newcommand{\VnBackbones}{seven}

\newcommand{\VegaRETMin}{+75.0}
\newcommand{\VegaRETMax}{+93.8}
\newcommand{\VegaTRKMin}{-71.9}
\newcommand{\VegaTRKMax}{-26.0}

\newcommand{\VlocCount}{3}
\newcommand{\VnMeasured}{7}

\newcommand{\VsplitRETBr}{69}

\newcommand{\VsplitTRKBs}{16}

\newcommand{\VsplitCMPBc}{27}

\newcommand{\VrecallMin}{0.75}
\newcommand{\VrecallMax}{1.00}
\newcommand{\VdeltaRMin}{0.00}
\newcommand{\VdeltaRMax}{0.25}
\providecommand{\VFone}{not evaluated}
\providecommand{\VFtwo}{not evaluated}
\providecommand{\VFthree}{not evaluated}
\providecommand{\VFfour}{not evaluated}
\providecommand{\VrhoUnif}{--}
\providecommand{\VrhoAsp}{--}
\providecommand{\VrhoDrop}{--}
\providecommand{\VthmOneViol}{not evaluated}
\providecommand{\VpairedFlagDelta}{--}
\providecommand{\VpairedFlagP}{--}
\providecommand{\VpairedFlagN}{--}
\providecommand{\VdomCount}{--}
\providecommand{\VsigCount}{--}
\providecommand{\VnBackbones}{--}
\providecommand{\VnTests}{--}
\providecommand{\VsigUncorr}{--}
\providecommand{\VsigHolm}{--}
\providecommand{\VretrHolmSurv}{--}
\providecommand{\VretrRawSurv}{--}
\providecommand{\VretrNeg}{--}

\providecommand{\VAspCmpAll}{--}

\providecommand{\VNoCsCmpAll}{--}

\renewcommand{\VFone}{does not fire}
\renewcommand{\VFtwo}{\textbf{fires}}
\renewcommand{\VFthree}{\textbf{fires}}

\renewcommand{\VFfour}{does not fire}
\renewcommand{\VrhoUnif}{-0.58}
\renewcommand{\VrhoAsp}{+0.00}
\renewcommand{\VrhoDrop}{-0.58}
\renewcommand{\VthmOneViol}{no backbone exceeds it}
\renewcommand{\VpairedFlagDelta}{+33.1}
\renewcommand{\VpairedFlagP}{5.4\times10^{-6}}
\renewcommand{\VpairedFlagN}{80}
\renewcommand{\VdomCount}{4}
\renewcommand{\VsigCount}{0}
\renewcommand{\VnBackbones}{7}
\renewcommand{\VnTests}{42}
\renewcommand{\VsigUncorr}{35}
\renewcommand{\VsigHolm}{31}
\renewcommand{\VretrHolmSurv}{0}
\renewcommand{\VretrRawSurv}{0}
\renewcommand{\VretrNeg}{3}
\renewcommand{\VAspCmpAll}{20.8}

\renewcommand{\VNoCsCmpAll}{83.3}

\definecolor{cstate}{RGB}{31,119,180}
\definecolor{cverb}{RGB}{214,39,40}
\definecolor{crouter}{RGB}{44,160,44}
\definecolor{cwall}{RGB}{120,120,120}
\definecolor{cpred}{RGB}{150,60,150}

\newtheorem{theorem}{Theorem}
\newtheorem{proposition}{Proposition}
\newtheorem{corollary}{Corollary}
\newtheorem{lemma}{Lemma}
\theoremstyle{definition}
\newtheorem{definition}{Definition}
\newtheorem{assumption}{Assumption}
\newtheorem{prediction}{Prediction}
\theoremstyle{plain}
\newcommand{\cmark}{\textcolor{crouter}{$\bullet$}}
\newcommand{\pmark}{\textcolor{cwall}{$\circ$}}
\newcommand{\xmark}{\textcolor{cverb}{--}}
\newcommand{\pred}{\textsuperscript{\textcolor{cpred}{$\dagger$}}}
\newcommand{\qedmark}{\ensuremath{\blacksquare}}
\newcommand{\proofend}{\unskip\nobreak\hfill\qedmark}
\newcommand{\proofenddisp}{\tag*{\qedmark}}
\newcommand{\ASP}{\textsc{Asp}}
\renewcommand{\Cap}{\mathrm{Cap}}
\newcommand{\err}{\mathrm{err}}
\newcommand{\cs}{\mathcal{C}_{\mathrm{s}}}
\newcommand{\cv}{\mathcal{C}_{\mathrm{v}}}
\newcommand{\A}{\mathcal{A}}

\ifdefined\ICRABuild
  \title{\LARGE \bf Budget-Constrained Embodied Perception:\\
  Four Resource Walls and a Pre-Registered Evaluation of\\
  Access-Structured Perception on Open Models at ${\le}31$B}
\else
  \title{\LARGE\textbf{Budget-Constrained Embodied Perception}\\[0.45em]
  \large\normalfont Four resource walls, and a pre-registered evaluation of\\
  Access-Structured Perception (\ASP{}) on open models at ${\le}31$B}
\fi
\newif\ifanon
\ifdefined\ICRABuild
  \ifdefined\ICRACameraReady\anonfalse\else\anontrue\fi
\else
  \anonfalse
\fi

\ifanon
  \author{Anonymous Authors}
\else
  \ifdefined\ICRABuild
    \author{Defu Lin$^{1}$, Wenhui Chen$^{1}$, Ziyao Lin$^{1}$,
    Jianlin Chen$^{2}$, Peiji Long$^{3}$, and Chi Man Vong$^{1}$%
    \thanks{$^{1}$The authors are with the University of Macau, Macao SAR, China.
    {\tt\small \{mc25108,mc35092,mc35081\}@um.edu.mo};
    {\tt\small cmvong@um.edu.mo}}%
    \thanks{$^{2}$Jianlin Chen is with the South China University of Technology,
    Guangzhou, China. {\tt\small 202330450231@mail.scut.edu.cn}}%
    \thanks{$^{3}$Peiji Long is an independent researcher.
    {\tt\small longpeiji@gmail.com}}%
    \thanks{Corresponding author: Chi Man Vong.}%
    }
  \else
    \author[1]{Defu Lin\thanks{\texttt{mc25108@um.edu.mo}}}
    \author[1]{Wenhui Chen\thanks{\texttt{mc35092@um.edu.mo}}}
    \author[1]{Ziyao Lin\thanks{\texttt{mc35081@um.edu.mo}}}
    \author[2]{Jianlin Chen\thanks{\texttt{202330450231@mail.scut.edu.cn}}}
    \author[3]{Peiji Long\thanks{\texttt{longpeiji@gmail.com}}}
    \author[1]{Chi Man Vong\thanks{Corresponding author. \texttt{cmvong@um.edu.mo}}}
    \affil[1]{University of Macau}
    \affil[2]{South China University of Technology}
    \affil[3]{Independent Researcher}
  \fi
\fi
\ifdefined\ICRABuild\else\date{}\fi

\begin{document}
\maketitle
\ifdefined\ICRABuild
  \thispagestyle{empty}
  \pagestyle{empty}
\fi

\begin{abstract}
Embodied multimodal agents must answer from growing observation streams under a fixed per-decision token budget. We formalize this constraint through four resource walls: a perceptual Shannon wall for bounded state, a horizon wall for query-independent frame selection, a round wall for non-adaptive retrieval, and a conditional composition wall for fixed-depth inference. We introduce \ASP{}, a training-free wrapper for frozen multimodal models that combines a capped structured state, a verbatim episodic index, and query-conditioned budget allocation with iterative access. Following a pre-registered protocol, we evaluate seven open-weight models from $3$B to $31$B on SEW-Bench, a license-free synthetic long-horizon walkthrough benchmark constructed to instantiate these walls. The registered natural-video benchmarks were not run because their frames require dataset agreements; our evidence therefore concerns access mechanisms, not natural-scene perception. Under a $4{,}096$-token decision budget, \ASP{} reaches $75$--$94\%$ episodic retrieval accuracy, compared with $3$--$19\%$ for equal-budget query-independent sampling, and budget reallocation outperforms quadrupling the sampling budget on every backbone. However, the full three-component architecture does not validate channel duality: removing the compressive state raises the flagship mean from $35.4$ to $58.0$, \ASP{} does not outperform the verbatim-only baseline on any backbone, and two of four pre-registered falsification criteria fire. These results show that query-conditioned access, rather than parameter count or context growth alone, is decisive under a fixed budget, while prompted online compression does not earn its cost in this setting.

\end{abstract}

\section{Introduction}
\label{sec:intro}

Two empirical facts frame the deployment of multimodal foundation models on embodied platforms. First, open models at or below 31B parameters (Qwen3.8-27B \citep{qwen38}, Gemma~4~31B \citep{gemma4}, Qwen3-Omni-30B-A3B \citep{qwen3omni}, Qwen3-VL-8B \citep{qwen3vl}) now match or exceed the single-frame perceptual accuracy that frontier proprietary models exhibited only two years earlier, consistent with the Platonic Representation Hypothesis that representations converge with scale and data \citep{prh}. Second, the same models remain strikingly poor at \emph{embodied} perception--decision tasks: on OpenEQA, multi-frame VLM agents perform close to blind LLMs on spatial and episodic questions \citep{openeqa}, and on VSI-Bench even frontier models fall far below human visual-spatial competence \citep{vsibench}. The bottleneck is evidently not what the model can represent in a frame; it is what the model can \emph{access} across a stream, under the hard per-decision compute budget that edge deployment imposes.

Why is the \emph{per-decision token budget} the right abstraction for this failure? An embodied control loop must emit an answer or an action every few seconds; whatever the platform (a smart-glasses assistant, a household robot, a warehouse AMR), the binding constraint at decision time is how many backbone tokens (prompt, retrieved pixels, and generated reasoning combined) can be consumed before the deadline, since accelerator FLOPs per decision scale linearly in tokens at fixed model size (\S\ref{sec:cost}). Neither the test-time-scaling line \citep{snell,s1,feng-cot} nor the long-context line \citep{ringattention,lwm,longva} answers the embodied question, because an egocentric mission stream grows without bound while the deadline does not: past some horizon \emph{selection} of what to read is unavoidable, and its quality, the access structure, becomes the bottleneck. We make this precise by charging every read to a single budget $B$ and asking which task families remain solvable as $N\to\infty$ with $B$ fixed.

The Capability Convergence Hypothesis (CCH) \citep{cch} gives this observation a theoretical shape. On symbolic streams, CCH proves that under a fixed per-token inference budget, capability converges not with scale but toward an \emph{access-complete hybrid}: any architecture that simultaneously holds a compressive $O(1)$-state channel and a scalable verbatim-index channel. Three resource walls (a Shannon wall, a horizon wall, and a circuit wall) each eliminate a single-channel architecture class, and a hybrid crosses all three by paying each wall's price, making capability strictly super-additive under channel composition. We inherit the claim but not its exact form: \S\ref{sec:theory} shows the additive statement needs a witness on which \emph{both} channels are obstructed, and supplies one. CCH validated its claims on pre-registered small-scale experiments, observing a \emph{scissors gap} of $0.994$ retrieval error for a recurrent state against $0.000$ once that state gains one global-attention layer.

This paper asks the question CCH left open: \emph{does the access-structure account survive the transfer from symbolic streams to embodied multimodal streams, and can it be exploited, training-free, to raise the capability of edge-scale open models?} The transfer is not automatic. Embodied observation streams differ from symbolic ones in bit density (a frame carries $10^4$--$10^6\times$ the salient bits of a token), in redundancy structure (temporal near-duplication), in query distribution (episodic, spatial, and compositional queries co-occur), and in the budget's granularity (the natural unit is the \emph{per-decision} budget of a control loop, not the per-token budget of a decoder). Our contributions:

\begin{enumerate}[leftmargin=1.6em,itemsep=1pt]
\item \textbf{Embodied instantiation of the CCH walls (\S\ref{sec:theory}).} We formalize budget-constrained embodied perception--decision and prove four impossibility results: a \emph{perceptual Shannon wall} (Theorem~\ref{thm:shannon}: any agent whose cross-frame state carries $m$ bits has episodic-attribute retrieval error at least $1-\frac{m/(Nd)+1}{\log_2 K}$ on $N$-keyframe streams carrying $d$ $K$-ary attributes each; the denominator is the \emph{attribute} count $Nd$, not $N$, and the difference decides whether the bound says anything), a \emph{horizon wall} (Theorem~\ref{thm:horizon}: any query-independent selection of $w$ frames caps accuracy at $\frac{w}{N}+(1-\frac{w}{N})\frac{1}{K}$), a \emph{round wall} (Theorem~\ref{thm:rounds}, unconditional: a single non-adaptive retrieval of $k$ frames cannot answer a depth-$\ge2$ dependent chain, however large $k$ is, and no cross-frame state helps it \emph{locate} the answer either), and a \emph{composition wall} (Theorem~\ref{thm:circuit}, conditional on $\mathsf{TC}^0\neq\mathsf{NC}^1$: single-pass fixed-depth attention cannot track spatio-temporal state chains of growing length, while $R$ iterative access rounds multiply effective depth).
\item \textbf{\ASP{}, and an honest accounting of which of its parts pay for themselves (\S\ref{sec:method}, \S\ref{sec:exp}).} \ASP{} wraps any frozen multimodal model with a compressive structured scene-state channel $\cs$, a verbatim episodic index $\cv$ with budgeted retrieval, and a query-conditioned router $\rho$ that partitions one per-decision token budget $B$ across the three wall prices. We prove it access-complete for the witness family under stated assumptions. \emph{Measured, one of its three components pays for itself.} Budgeted query-conditioned retrieval is decisive; it beats every equal-budget query-independent baseline by $17$--$66$ points, and every one of those comparisons survives multiple-comparison correction. The compressive channel and the router do not: removing $\cs$ raises the flagship's mean from $35.4$ to $58.0$, removing $\rho$ raises it to $55.2$, and after correction \emph{no} backbone shows \ASP{} beating the verbatim-only baseline. We report this as a result rather than a limitation, and localise it: a prompted accumulator carries a running reduction across $\VN{}$ updates, and the union bound's $N\epsilon_u$ term is tight for an additive reduction, so Assumption~\ref{ass:sep}(i) demands a per-update error below $1/\VN{}$ that nothing at this scale delivers. The design rule that follows, and the part of this we expect to outlive the numbers, is that a compressive channel should \emph{compute} any reduction expressible as an aggregate over the index rather than prompt for it.

\item \textbf{A pre-registered program, and the part of it we could run (\S\ref{sec:exp}).} Following CCH's methodology we froze the protocol, the numeric predictions and the falsification criteria before measurement. What we then measured is a full grid, \VnBackbones{} open backbones spanning $3$--$31$B, seven equal-budget access structures (including EGAgent \citep{egagent} reimplemented under our budget accounting), three wall-witness subtests, seven ablations, and a $4\times$-budget arm, on \textbf{SEW-Bench}, a licence-free synthetic corpus built so that each bound of \S\ref{sec:theory} is non-vacuous at its parameters. What we could not run is the natural-video half of the registration: OpenEQA, VSI-Bench, EgoSchema and the registered HM3D-based EW-Bench all index video that is gated behind a signed dataset agreement (HM3D, ScanNet, ScanNet++, Ego4D), which we do not hold; their question sets are public, their frames are not (\S\ref{sec:setup}). The unrun predictions stay in the paper marked \pred{} and the clauses naming them are recorded as \emph{not evaluated} rather than passed. \emph{Nothing here should be read as an embodied-perception result on natural scenes; it is a mechanism result on access structure under a token budget.}
\end{enumerate}

Our position is deliberately falsifiable: if uniform-window baselines match \ASP{} on EW-Bench-RET, or if the scale--score correlation is unchanged under \ASP{}, the embodied extension of CCH is wrong in the manner we specify in \S\ref{sec:falsify}.

\begin{aspfigure}[!t]
\centering
\begin{tikzpicture}[
  box/.style={draw, rounded corners=2pt, align=center, font=\small,
              inner xsep=4pt, inner ysep=4pt},
  lab/.style={font=\footnotesize\itshape},
  arr/.style={-{Stealth[length=2.2mm]}, thick},
  feed/.style={-{Stealth[length=2.2mm]}, thick, dashed, crouter}]

\node[box, fill=gray!10, text width=2.2cm] (stream) at (1.05, 0)
  {Egocentric stream\\ $x_1,\dots,x_N$};
\node[box, fill=gray!4,  text width=1.6cm] (gate)   at (4.00, 0)
  {Keyframe\\ gate $g$};
\node[box, fill=yellow!16, text width=2.35cm] (llm)  at (12.60, 0)
  {Frozen backbone\\ ($\le$31B, API)\\[1pt] {\scriptsize decode under $B_c$}};
\node[box, fill=gray!10, text width=1.4cm] (out)    at (15.30, 0)
  {Answer /\\ action $a$};

\node[box, fill=cstate!14, draw=cstate, text width=4.2cm] (state) at (8.40, 1.05)
  {Compressive channel $\cs$\\ scene state $s_t$, $|s_t|\!\le\!L_s$\\[1pt]
   {\scriptsize read under $B_s$}};
\node[box, fill=cverb!10, draw=cverb, text width=4.2cm] (index) at (8.40,-1.05)
  {Verbatim channel $\cv$\\ episodic index $\mathcal{I}_t$, lossless\\[1pt]
   {\scriptsize top-$k$ retrieval under $B_r$}};

\node[box, fill=gray!10, text width=2.2cm] (query)  at (1.05,-2.65)
  {Query / goal $q$};
\node[box, fill=crouter!12, draw=crouter, text width=2.4cm] (router) at (4.20,-2.65)
  {Router $\rho$\\ $B{=}B_s{+}B_r{+}B_c$};

\coordinate (fork) at ($(gate.east)+(0.55,0)$);
\draw[thick] (gate.east) -- (fork);
\draw[arr] (fork) |- (state.west);
\draw[arr] (fork) |- (index.west);

\draw[arr] (state.east) -- ++(0.45,0) |- ([yshift=5pt]llm.west);
\draw[arr] (index.east) -- ++(0.45,0) |- ([yshift=-5pt]llm.west);

\draw[arr] (stream) -- (gate);
\draw[arr] (query)  -- (router);
\draw[arr] (router.east) -| (llm.south);
\draw[arr] (llm) -- (out);

\node[inner sep=0pt] (loopmark) at (10.40, 2.40) {};
\draw[feed] (llm.north) -- (12.60,1.90) --
  node[lab, crouter, above, pos=0.54] {$\le R$ adaptive rounds}
  (8.40,1.90) -- (state.north);

\begin{scope}[on background layer]
\node[draw=gray!55, dashed, rounded corners=4pt,
      fit=(gate)(state)(index)(router)(loopmark),
      inner xsep=8pt, inner ysep=6pt] (fitbox) {};
\end{scope}
\node[lab, anchor=north west, yshift=-2pt] at (fitbox.south west)
  {everything \ASP{} adds: gate $g$ $+$ access structure $\A=(\cs,\cv,\rho)$};
\end{tikzpicture}
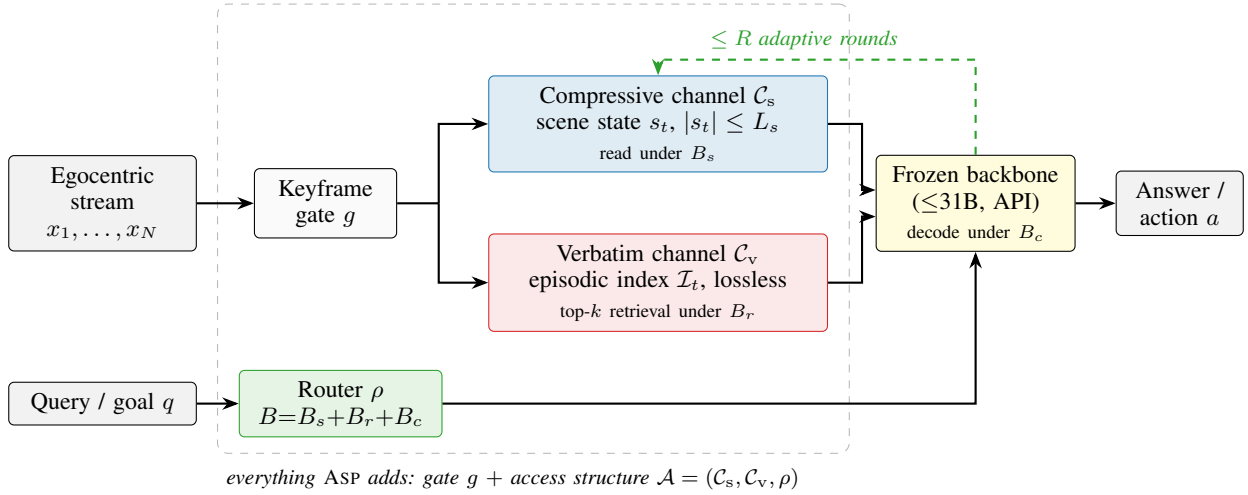
\caption{\textbf{\ASP{} overview.} A shared keyframe gate feeds two channels that are
different in kind: a compressive $O(1)$-token structured scene state (blue), which holds the
running reductions retrieval cannot cheaply rebuild, and a verbatim episodic index (red),
which keeps every keyframe losslessly off-GPU and returns $k$ of them per query. A
query-conditioned router (green) water-fills the fixed per-decision budget
$B=B_s{+}B_r{+}B_c$ over the three wall prices of \S\ref{sec:theory}, and the decision may
take up to $R$ adaptive rounds, each round's read depends on what the previous one
returned, which is what Theorem~\ref{thm:rounds} shows a single retrieval cannot simulate.
Each budget component is written inside the component it pays for. Everything inside the
dashed box is \ASP{}, the gate included; it is shared verbatim by every baseline so that
the comparison isolates access structure; the backbone sits outside the box and is never
touched. \ASP{} changes only \emph{what the model is allowed to read} per decision.}
\label{fig:overview}
\end{aspfigure}

\paragraph{Organization.} \S\ref{sec:related} situates the work against long-video memory, embodied QA, test-time compute, and the 2026 agentic generation. \S\ref{sec:theory} formalizes budgeted embodied perception and proves the four walls plus access-completeness. \S\ref{sec:method} presents \ASP{}, its algorithm, and the hardware-agnostic cost model. \S\ref{sec:exp} gives the frozen experimental design, all registered predictions, ablations, analysis, and falsification criteria. Appendices contain full proofs, the EW-Bench construction protocol, the prediction-derivation worksheet, and secondary experiments.

\section{Related Work}
\label{sec:related}

We organise this section around a single question asked of each line of work: \emph{which access channel does it commit to, and which wall does that commitment leave unpaid?} Table~\ref{tab:related} is the summary; the paragraphs give the reasoning. Reading the literature this way is not fault-finding: most of these systems are strong at what they were built for. The point is that the walls of \S\ref{sec:theory} partition it cleanly, which is evidence the partition is real rather than an artefact of our framing.

\begin{asptable}[!t]
\centering\small\setlength{\tabcolsep}{4pt}
\caption{\textbf{Related work, read for the channel each class commits to.} \cmark{} = present and unbounded in $N$; \pmark{} = present but lossy or bounded; \xmark{} = absent. ``Explicit $B$'' means a per-decision token budget that \emph{all} reads are charged against, including reasoning. The last column names the wall the class cannot pay, from \S\ref{sec:theory}. Classification is by the commitment each system's published design makes, not by measured performance; borderline cases are marked \pmark{} rather than argued.}
\label{tab:related}
\begin{tabular}{p{5.05cm}cccc>{\raggedright\arraybackslash}p{4.05cm}}
\toprule
Class (representatives) & $\cs$ & $\cv$ & adapt.\ $R$ & expl.\ $B$ & Wall left unpaid \\
\midrule
Compressive / streaming memory
 \citep{moviechat,malmm,flashvstream}
 & \cmark & \xmark & \xmark & \xmark & Shannon (Thm.~\ref{thm:shannon}): episodic detail \\
Query-\emph{independent} selection
 \citep{videotree,llovi,openeqa}
 & \xmark & \pmark & \xmark & \xmark & Horizon (Thm.~\ref{thm:horizon}): coverage \\
In-pass token reduction
 \citep{fastv,prumerge,tosa}
 & \xmark & \pmark & \xmark & \xmark & Shannon: lossy \emph{at rest}, not just in context \\
Query-conditioned retrieval
 \citep{videoagent,goldfish,videorag,vgent}
 & \xmark & \cmark & \pmark & \xmark & Tracking (Lem.~\ref{lem:vtrack}): running reductions \\
Structured scene memory
 \citep{graphpad,sayplan,conceptgraphs}
 & \pmark & \xmark & \pmark & \xmark & Shannon; and no certified cap $L_s$ \\
Context growth
 \citep{ringattention,lwm,longva}
 & \xmark & \pmark & \xmark & \xmark & Composition (Thm.~\ref{thm:circuit}--\ref{thm:rounds}): width $\neq$ depth \\
Agentic hybrids, 2026
 \citep{egagent,worldmm,streameqa}
 & \pmark & \cmark & \cmark & \xmark & none structurally, but tool choice is heuristic, so no wall is \emph{priced} \\
\midrule
\ASP{} (ours) & \cmark & \cmark & \cmark & \cmark & --- (access-complete, Prop.~\ref{prop:complete}) \\
\bottomrule
\end{tabular}
\end{asptable}

\paragraph{Representational vs.\ capability convergence.}
The Platonic Representation Hypothesis documents that representations learned by different
architectures and modalities converge as scale and data grow \citep{prh}. CCH accepts that
premise and denies the inference usually drawn from it: convergent representations do not entail
convergent \emph{capability} under bounded inference, because what a model can represent per
access and what it can reach across a stream are different quantities, and only the first
improves with scale \citep{cch}. That distinction is the content of our
Assumption~\ref{ass:repr}, which grants the representational premise and measures it
(Table~\ref{tab:probe}). The boundary results that make it bite are well developed on the
language side: fixed-depth bounded-precision transformers cannot decide problems outside
$\mathsf{TC}^0$ in one pass \citep{merrill-sat}, chain-of-thought recovers depth in proportion to
serial steps \citep{merrill-cot,feng-cot}, apparent compositional success is often a shortcut
that fails to extrapolate \citep{shortcuts}, and linear-time sequence models are provably
limited in state tracking \citep{illusion-state,mamba}, which is why hybrid attention--SSM
designs \citep{jamba,griffin} outperform either component. That last is an architectural
anticipation, on symbolic streams, of the two-channel result we prove for pixels. What is
missing is the embodied instantiation: none of these results is stated over observation streams,
and none treats the per-decision budget of a control loop as the binding resource.

\paragraph{The lower-bound toolkit, and where it has not been pointed.}
Our proofs use standard machinery, and the novelty claim is about the target rather than the
technique. Theorem~\ref{thm:shannon} is a counting argument of the kind used for streaming and
sketching lower bounds, applied to a state that is an arbitrary function of the stream;
Theorem~\ref{thm:rounds} rests on the round--communication structure of pointer chasing
\citep{pointerchasing,nisanwigderson}, of which we prove only the single-round case in closed
form (Lemma~\ref{lem:listfano}); Theorem~\ref{thm:circuit} routes through Barrington's
characterisation of $\mathsf{NC}^1$ \citep{barrington} composed with the $\mathsf{TC}^0$
simulation of fixed-depth transformers \citep{merrill-sat}. What is new is the target and the
accounting: charging \emph{every} read, maintained state, retrieved pixels and generated
reasoning alike, to one per-decision budget, so that three classical obstructions become three
\emph{prices} in the same unit and a router becomes an allocator over them.
Lemma~\ref{lem:vtrack} is the one bound we could not find an off-the-shelf form of, since it
concerns an agent whose retrieval is query-conditioned, so the horizon wall is silent, but which
maintains no cross-frame state.

\paragraph{Long-video and streaming memory for multimodal LLMs.}
A large literature compresses video for LLM consumption. Fixed-size recurrent memories
consolidate frames into a bounded store \citep{moviechat,malmm} and streaming architectures
maintain one online \citep{flashvstream}; hierarchical and tree-structured methods select frames
by salience or clustering \citep{videotree,llovi}; retrieval-augmented agents select them
conditioned on the question \citep{videoagent,goldfish}; and token pruning discards visual
tokens inside the backbone \citep{fastv,prumerge,tosa}. The 2026 agentic generation adds
structured tools \citep{egagent,worldmm,vgent,videorag,streameqa}. Read through the walls these
fall into three groups. Compressive-only systems are bounded on episodic detail by
Theorem~\ref{thm:shannon}, and the bound is compression-proof: a scene graph, a caption set and
a learned latent are all $m$-bit functions of the stream, so semantic cleverness buys constants
rather than asymptotics. Selection-only systems are bounded by Theorem~\ref{thm:horizon}
whenever the rule is query-independent, which covers salience, novelty, clustering and tree
search. Retrieval systems escape both but, maintaining no cross-frame state, are pinned at
chance on running reductions by Lemma~\ref{lem:vtrack}. Token pruning is a different commitment,
and the distinction matters in \S\ref{sec:method}: pruning decides what the model \emph{keeps}
and is lossy at rest, whereas an index decides what it \emph{touches} and is lossless at rest.
The closest system to ours is EGAgent \citep{egagent}, which has both channel types and
iterates; what it lacks is the budget object, and its text-only graph discards verbatim
appearance, precisely the configuration Theorem~\ref{thm:shannon} bounds.

\paragraph{Embodied question answering and its benchmarks.}
EQA \citep{eqa-das} in its modern open-vocabulary form is defined by OpenEQA, which separates
episodic-memory (EM-EQA) from active (A-EQA) settings \citep{openeqa}; we work in the EM-EQA
regime and say why in \S\ref{sec:limits}. Extensions probe exploration-awareness, noisy queries,
urban scale and inspection domains \citep{expressbench,noisyeqa,cityeqa,bridgeeqa}, and adjacent
suites isolate other axes \citep{vsibench,egoschema,videomme,erqa,embodiedbench}. These measure
whether a system is good; they are the wrong instrument for attributing \emph{why}, since each
mixes episodic, cumulative and compositional demands in an unmeasured proportion. That is the
gap EW-Bench was designed to fill. Structured scene-memory agents
\citep{graphpad,sayplan,conceptgraphs} are direct precursors of our compressive channel and
frame-retrieval agents \citep{videoagent} of our verbatim channel; neither lineage supplies a
certified token cap, since free-form summaries and unboundedly growing graphs have no $L_s$.

\paragraph{Growing the window: the long-context antithesis.}
The most influential alternative to structured access is to enlarge what one forward pass can attend to: RingAttention and Large World Models push context toward millions of tokens \citep{ringattention,lwm}, LongVA transfers long-context ability from language to vision \citep{longva}, and proprietary systems advertise hour-scale ingestion. Theorems~\ref{thm:shannon}--\ref{thm:horizon} say exactly what this buys. A longer window raises $w$ and therefore defers the horizon wall (it does not remove it, since an egocentric mission stream grows without bound while any window is fixed), and it does so at a per-decision cost linear in $w$, which is the quantity an embodied deadline caps. More importantly, width does not touch the composition wall: Theorems~\ref{thm:rounds}--\ref{thm:circuit} make depth and adaptivity the binding resources there, and a single pass over a longer context is still a single pass. Under a fixed $B$ the two strategies are therefore not interchangeable, and the difference is not a matter of degree: uniform spending buys $w \propto B$ frames of coverage, whereas structured spending buys whichever wall the query actually faces. Prediction~\ref{prd:budget} is the registered form of this claim, and it is designed to be losable, if quadrupling an unstructured budget matches a routed one, the account is wrong in the way \S\ref{sec:falsify} specifies.

\paragraph{Memory-augmented and tool-using agents.}
Outside video the two channels have long and largely separate lineages. Retrieval-augmented generation and dense retrieval supply verbatim access over a corpus \citep{lewis-rag,dpr}; paged and hierarchical agent memories supply compressive state under an explicit capacity discipline \citep{memgpt,hipporag}. Reasoning-and-acting loops \citep{react,reflexion,toolformer} anticipate our iterative access rounds almost exactly in mechanism, and the provable depth extension of chain-of-thought \citep{wei-cot,feng-cot,merrill-cot} is the escape clause Theorem~\ref{thm:circuit} relies on. Two things distinguish the embodied setting. First, the joint budget: these frameworks meter retrieval and reasoning separately, if at all, whereas an embodied deadline imposes one budget across both channels \emph{and} the reasoning tokens, which turns channel choice from a design taste into a constrained allocation problem with shadow prices we can name. Second, the bit density: a retrieved document and a retrieved keyframe cost comparable tokens but carry very different amounts of decision-relevant information, so the exchange rate between the two channels is not inherited from the text setting and has to be established for pixels, which is what the router's surrogate does and what S4 tests.

\paragraph{Embodied foundation models and simulators.}
End-to-end embodied policies fold perception, memory, and control into weights: PaLM-E, RT-2, and OpenVLA \citep{palme,rt2,openvla} demonstrate that internet-scale pretraining transfers to manipulation, and simulators and suites such as Habitat, ALFRED, Ego4D, and EgoLife \citep{habitat,alfred,ego4d,egolife} supply the streams on which embodied memory is stressed; EW-Bench is built on HM3D and ScanNet scenes \citep{hm3d,scannet} for the same reason. Spatially-tuned VLMs \citep{spatialvlm} raise exactly the per-access competence our Assumption~\ref{ass:repr} measures. These lines are complementary rather than competing: they improve what a backbone can do per access, while we bound what no backbone can do without paying access prices and show the prices are payable at inference time on a frozen model. Scaling-law analyses \citep{kaplan,chinchilla} predict per-access competence from compute, which is precisely the quantity Corollary~\ref{cor:conv} claims saturates first, so the convergence prediction is a statement about where the scaling laws stop being the relevant instrument, not a claim that they are wrong.

\paragraph{Test-time compute and budgeted inference.}
Test-time scaling establishes for text that reallocating inference compute can beat growing parameters \citep{snell,s1}, and budget-aware routing across models or tools makes that reallocation explicit \citep{routellm,frugalgpt}. We import the idea and change what is being allocated. In the text setting the budget is spent on \emph{reasoning} (more samples, longer chains, a larger model for harder queries), and the allocation is over difficulty. Here the budget must additionally be split across \emph{access}: which frames enter context, which state fields are read, how many rounds are run. That is a different allocation problem because its arms have provable, unequal returns: \S\ref{sec:theory} gives each arm a wall and each wall a price, so the router is not fitting an empirical difficulty signal but water-filling against bounds. It is also the reason our ablation $-\rho$ (fixed even split) is informative rather than a formality: it isolates what routing contributes once both channels are already present.

\paragraph{Edge multimodal models.}
Sub-31B open multimodal models have advanced rapidly: Qwen3-VL-8B \citep{qwen3vl}, Qwen2.5-VL \citep{qwen25vl}, Qwen3-Omni-30B-A3B with a Thinker--Talker MoE \citep{qwen3omni}, Gemma~4 \citep{gemma4}, Qwen3.8-27B with native image and video input \citep{qwen38}, Phi-4-Multimodal \citep{phi4mm}, and MiniCPM-V for on-device use \citep{minicpmv}. Long-video instruction tuning \citep{longvitu} and caption-then-reason pipelines \citep{socratic,alanavlm} show how far training and prompting alone carry this class. These models supply the representational floor our theory assumes; \ASP{} supplies the access structure they lack, without touching their weights.

\paragraph{Pre-registration as method.}
Finally, a note on how this paper is written rather than what it argues. We follow CCH in freezing the protocol and the quantitative predictions before measurement \citep{cch}, marking every unmeasured value, and committing to falsification criteria that cannot be revised afterwards (\S\ref{sec:falsify}). This is unusual in the area and costs us something real: registered values will be wrong in places, and the record of that will be public. We think the trade is worth it here specifically because the claim is causal (that a named structure, not scale, produces the capability), and a causal claim tested by criteria chosen after seeing the data is not tested at all.

\section{Theory: The Four Walls of Budgeted Embodied Perception}
\label{sec:theory}

\subsection{Setting}

\begin{definition}[Embodied stream and decision tasks]
\label{def:stream}
An \emph{embodied stream} is a sequence $X=(x_1,\dots,x_N)$ of keyframes produced by an egocentric sensor; each keyframe carries a salient content vector $c_i\in\{1,\dots,K\}^{d}$ (attributes such as object identity, color, state, relative pose), so a keyframe holds $b=d\log_2 K$ salient bits. A \emph{decision task} is a pair $(q,y^{\star})$: a query or goal $q$ issued at time $N$, with correct answer/action $y^{\star}=f(q,X)$. We consider three witness families:
$T_{\mathrm{ret}}$ (\emph{episodic retrieval}: $q$ asks the value of a uniformly chosen attribute of a uniformly chosen past keyframe);
$T_{\mathrm{trk}}$ (\emph{cumulative state tracking}: $y^{\star}$ is a running reduction $\phi(c_1,\dots,c_N)$, e.g.\ object count, door open/closed parity, containment state);
$T_{\mathrm{cmp}}$ (\emph{spatio-temporal composition}: $y^{\star}$ requires a chain of $d_q$ dependent lookups/updates, e.g.\ ``the object that was on the table you passed \emph{after} the red chair, where is it now?''), where each hop's operand is an attribute drawn from the content vectors $c_i$, so that every hop is itself a $T_{\mathrm{ret}}$ instance;
and $T_{\mathrm{conj}}$ (\emph{conjunctive}: $y^{\star}=(y^{\star}_{\mathrm{ret}},y^{\star}_{\mathrm{trk}})$ pairs an episodic attribute lookup with a running reduction, scored correct only if both components are correct).
The queryable-attribute count per keyframe, $d$, is a property of the stream and is logged for every corpus we use (Appendix~\ref{app:ewbench}); it enters every bound below through the total attribute count $Nd$, not through $N$.
\end{definition}

\begin{definition}[Budgeted agent and access structure]
\label{def:agent}
A \emph{budgeted embodied agent} is a triple $A=(\mathcal{S},U,\pi)$: a cross-frame state space $\mathcal{S}$, an online update $s_t=U(s_{t-1},x_t)$ computed as the stream arrives, and a decision policy $\pi$ that, given $(q,s_N)$ and \emph{at most $B$ backbone tokens per decision} (prompt + retrieved content + generated reasoning), outputs $\hat y$. An \emph{access structure} $\A=(\cs,\cv,\rho)$ decomposes the agent's information pathway into a \emph{compressive channel} $\cs$ whose state size is $O(1)$ in $N$ (at most $L_s$ tokens, i.e.\ $m\le L_s\beta$ bits at $\beta$ bits/token), a \emph{verbatim channel} $\cv$ that stores keyframes losslessly and exposes a retrieval operator $r(q';k)$ returning $k$ items, and a \emph{router} $\rho:(q,s_N)\mapsto(B_s,B_r,B_c)$ with $B_s+B_r+B_c\le B$. Capability is $\Cap_B(A,T)=\mathbb{E}_{(q,y^{\star})\sim T}[\mathrm{score}(\hat y,y^{\star})]$.
\end{definition}

\begin{assumption}[Representational sufficiency]
\label{ass:repr}
The frozen backbone answers single-keyframe attribute queries with accuracy $\ge 1-\epsilon_0$ when the relevant keyframe is in context, and executes one composition step with accuracy $\ge 1-\epsilon_1$ per round. This is the PRH-side premise \citep{prh}: sub-31B multimodal models are assumed perceptually competent \emph{per access}; all failure we study is failure of access. Assumption~\ref{ass:repr} is measured rather than asserted: Table~\ref{tab:probe} reports both margins per backbone, obtained by asking the RET question with the target keyframe handed to the model and by asking one composition step over two in-context frames. The first margin is comfortably sufficient at this scale; the second is not, for several backbones, and \S\ref{sec:analysis} says what that costs the CMP claims.
\end{assumption}

\begin{definition}[Wall price]
\label{def:price}
For a wall $W$ eliminating an architecture class on a witness family $T_W$, its \emph{price} is the minimum budget increment that, spent on the excluded resource, restores $\Cap_B\ge 1-\delta$ on $T_W$: $B_s^{\star}$ tokens of maintained state for the Shannon wall on reductions, $B_r^{\star}$ tokens of query-conditioned retrieval for the Shannon and horizon walls on episodic detail, and $B_c^{\star}$ tokens per round times $R$ rounds for the composition walls. Wall prices are the exchange rate between budget and capability; the router of \S\ref{sec:router} is an allocator over them.
\end{definition}

\subsection{Four impossibility results}

\begin{theorem}[Perceptual Shannon wall]
\label{thm:shannon}
Let attributes in Definition~\ref{def:stream} be i.i.d.\ uniform on $\{1,\dots,K\}$ and let $A$ be any agent (of unbounded compute) whose only cross-frame memory is a state of $m$ bits. Then on $T_{\mathrm{ret}}$ with per-frame query,
\[
\err(A)\;\ge\;1-\frac{m/(Nd)+1}{\log_2 K}.
\]
In particular, any $o(Nb)$-bit state (recall $b=d\log_2K$) forces $\err\to 1-\tfrac{1}{\log_2 K}$ as $N\to\infty$: constant-size scene memories (however semantically clever) cannot answer episodic-appearance queries about arbitrary past keyframes.
\end{theorem}
\begin{proof}[Proof sketch]
The $Nd$ queried attributes form $Nd\log_2 K$ uniform bits; the state is an $m$-bit function of them. Averaging Fano's inequality over the $Nd$ possible queries and using $H(\text{answers}\mid s_N)\ge Nd\log_2 K-m$ yields the bound. Full proof in Appendix~\ref{app:proofs}.
\end{proof}

\emph{Remark (tightness and scope).} The bound is matched up to constants by an agent that stores $\lfloor m/\log_2K\rfloor$ attribute values verbatim and guesses elsewhere (accuracy $\lfloor m/\log_2K\rfloor/(Nd)+1/K$ against the ceiling $\kappa_s=\bigl(m/(Nd)+1\bigr)/\log_2K$, which agree to within the additive $1/K$), so nothing stronger is true at this granularity; and it is \emph{semantic-compression-proof}: scene graphs, captions, embeddings, and learned latents are all $m$-bit functions of the stream, so the theorem bounds MovieChat-, MA-LMM-, and EGAgent-graph-style memories alike \citep{moviechat,malmm,egagent}. It does \emph{not} bound agents that keep the stream on disk and read selectively; that is the verbatim channel, whose own obstruction is Theorem~\ref{thm:horizon}.

\begin{theorem}[Horizon wall]
\label{thm:horizon}
Let $A$ select, by any rule measurable with respect to the stream but \emph{independent of the query}, a set $W$ of at most $w$ keyframes to place in context (uniform sampling, last-$w$ window, salience top-$w$, etc.), discarding the rest. Then on $T_{\mathrm{ret}}$,
\[
\Pr[\hat y=y^{\star}]\;\le\;\frac{w}{N}\;+\;\Bigl(1-\frac{w}{N}\Bigr)\frac{1}{K}.
\]
Hence any fixed context window is a vanishing-accuracy strategy on streams that grow with mission length, regardless of backbone scale.
\end{theorem}

Figure~\ref{fig:regimes} places the four walls on the two axes that generate them, stream length and per-decision budget, so which wall binds a given deployment can be read off rather than derived.

\emph{Remark.} The quantifier matters: the selection rule may depend arbitrarily on the stream (salience, novelty, clustering, tree search over frames \citep{videotree,llovi}), everything short of the query. What the theorem licenses is therefore not ``sampling is bad'' but ``query-independent sampling is bad''; query-conditioned retrieval $r(q;k)$ escapes by making $W$ depend on $q$, at the price $B_r^{\star}$ of reading the retrieved frames. This is the formal content of the empirical observation that frame-selection heuristics plateau on needle-type video QA \citep{videoagent,goldfish}.

\begin{theorem}[Composition wall, conditional]
\label{thm:circuit}
Encode $T_{\mathrm{cmp}}$ chains as word problems over a non-solvable group as in \citep{shortcuts,merrill-cot} (spatio-temporal state updates realize $S_5$ generators; Appendix~\ref{app:proofs}). If $\mathsf{TC}^0\neq\mathsf{NC}^1$, then no fixed-depth, polynomial-width, constant-precision attention stack decides $T_{\mathrm{cmp}}$ with chain length $d_q=\omega(1)$ in a \emph{single} forward pass, even with all keyframes in context. An agent granted $R$ sequential access rounds (each re-reading its own intermediate output) simulates depth $\Theta(R\cdot d_{\mathrm{model}})$ and decides chains of length $O(R)$ per Assumption~\ref{ass:repr}.
\end{theorem}

\begin{theorem}[Round wall, unconditional]
\label{thm:rounds}
Model $T_{\mathrm{cmp}}$ chains as pointer chasing: keyframe $i$ carries a pointer $p_i\sim\mathrm{Unif}[N]$ and an attribute $a_i\sim\mathrm{Unif}[K]$, all independent, and the depth-$2$ query is ``report $a_{p_J}$'' for a queried index $J\sim\mathrm{Unif}[N]$. Let an agent's retrieval be confined to a \emph{single non-adaptive round} returning a set $W$, $|W|\le k$, chosen from $(q,s_N)$ before any retrieved content is read. Then:
\begin{enumerate}[label=(\roman*),leftmargin=2.1em,itemsep=1pt,topsep=2pt]
\item\label{it:rounds-stateless} \emph{(no cross-frame state)}
$\displaystyle\Pr[\hat y=y^{\star}]\;\le\;\frac{k}{N}+\Bigl(1-\frac kN\Bigr)\frac1K .$
\item\label{it:rounds-loc} \emph{(localisation, $m$-bit state)} The round reaches the keyframe that carries the answer with probability
$\displaystyle\Pr[p_J\in W]\;\le\;\frac{m/N+1}{\log_2(N/k)} ,$
which for $m=o(N\log_2 N)$ tends to $0$ however cleverly the state is computed.
\item\label{it:rounds-grounded} \emph{(capability, $m$-bit state)} If in addition the agent is \emph{read-grounded} (its answer is a function of content it has actually read from $\cv$, not of $s_N$ alone) then
$\displaystyle\Pr[\hat y=y^{\star}]\;\le\;\frac{m/N+1}{\log_2(N/k)}+\frac1K .$
\end{enumerate}
An agent with $R\ge d_q$ \emph{adaptive} rounds resolves one hop per round with $k=O(1)$ retrievals each and answers correctly with probability $\ge 1-d_q\epsilon_1$ under Assumption~\ref{ass:repr}. Unlike Theorem~\ref{thm:circuit}, all three parts are unconditional: they separate single-round from multi-round access with no complexity-theoretic assumption, and \ref{it:rounds-stateless} is the wall that binds \textsc{Retr-only}-style agents, which have no $\cs$ by construction. The classical round hierarchy for pointer chasing \citep{pointerchasing,nisanwigderson} extends the separation to every $R<d_q$; we prove only the $R{=}1$ case, which is the one our baselines instantiate.
\end{theorem}

\emph{Remark (the read-groundedness hypothesis).} Dropping it makes \ref{it:rounds-grounded} false rather than merely unproved: a state that tabulates answers for a fixed query subset beats the stated bound, and \S\ref{app:proofs} gives the arithmetic. The escape is a lookup table costing $\log_2K$ bits per query covered, hence $\Theta(N\log_2K)$ bits for a constant fraction of the query set, so it is not a compressive channel in the sense of Definition~\ref{def:agent}, where $\cs$ is $O(1)$ in $N$. We have no bound covering states that are neither read-grounded nor tabulators; \ref{it:rounds-stateless} carries the baselines and \ref{it:rounds-loc} the mechanism.

\begin{aspfigure}[!t]
\centering
\begin{subfigure}{0.48\textwidth}
\centering
\begin{tikzpicture}
\begin{axis}[width=0.94\linewidth, height=5.0cm,
  xlabel={attribute count $Nd$}, ylabel={$\kappa_s$: ceiling on RET},
  xmode=log, xmin=900, xmax=130000, ymin=0.22, ymax=2.35,
  grid=major, grid style={gray!20}, label style={font=\small},
  tick label style={font=\scriptsize},
  legend style={font=\scriptsize, at={(0.5,-0.24)}, anchor=north, legend columns=2,
                draw=none, fill=none, column sep=1.1ex, row sep=-2pt}]
\addplot[cwall, dashed, thick, domain=2.95:5.12, samples=2] ({10^x},{1});
\addlegendentry{$\kappa_s{=}1$: no constraint}
\addplot[cverb, thick, domain=2.95:5.12, samples=140] ({10^x},{(24576/(10^x)+1)/3});
\addlegendentry{$L_s{=}2048$}
\addplot[crouter, very thick, domain=2.95:5.12, samples=140] ({10^x},{(6144/(10^x)+1)/3});
\addlegendentry{$L_s{=}512$ (ours)}
\addplot[cstate, thick, domain=2.95:5.12, samples=140] ({10^x},{(1536/(10^x)+1)/3});
\addlegendentry{$L_s{=}128$}
\addplot[gray, densely dotted, thick, domain=2.95:5.12, samples=2] ({10^x},{0.3333});
\addplot[only marks, mark=*, mark size=1.6pt, crouter]
  coordinates {(1350,1.850) (9600,0.547) (43200,0.381)};
\node[font=\scriptsize, anchor=west, cverb] at (axis cs:1480,2.06) {$d{=}9$: vacuous};
\node[font=\scriptsize, anchor=west] at (axis cs:10600,0.68) {SEW};
\node[font=\scriptsize, anchor=south] at (axis cs:43200,0.43) {EW (reg.)};
\end{axis}
\end{tikzpicture}
\caption{Theorem~\ref{thm:shannon}: the ceiling is a function of $Nd$, not of $N$.}
\label{fig:kappa}
\end{subfigure}\hfill
\begin{subfigure}{0.48\textwidth}
\centering
\begin{tikzpicture}
\begin{axis}[width=0.94\linewidth, height=5.0cm,
  xlabel={stream length $N$}, ylabel={ceiling on strict RET},
  xmode=log, xmin=100, xmax=3000, ymin=0.08, ymax=1.0,
  xtick={100,300,1000,3000}, xticklabels={100,300,1000,3000},
  grid=major, grid style={gray!20}, label style={font=\small},
  tick label style={font=\scriptsize},
  legend style={font=\scriptsize, at={(0.5,-0.24)}, anchor=north, legend columns=2,
                draw=none, fill=none, column sep=1.1ex, row sep=-2pt}]
\addplot[cverb, thick, domain=2:3.477, samples=140] ({10^x},{0.125+0.875*96/(10^x)});
\addlegendentry{$w{=}96$}
\addplot[cstate, thick, domain=2:3.477, samples=140] ({10^x},{0.125+0.875*32/(10^x)});
\addlegendentry{$w{=}32$}
\addplot[crouter, very thick, domain=2:3.477, samples=140] ({10^x},{0.125+0.875*12/(10^x)});
\addlegendentry{$w{=}12$ (affordable at $B{=}4$k)}
\addplot[gray, densely dotted, thick, domain=2:3.477, samples=2] ({10^x},{0.125});
\node[font=\scriptsize, anchor=west, gray!45!black] at (axis cs:112,0.147) {chance $1/K$};
\draw[cwall, dashed] (axis cs:300,0.08) -- (axis cs:300,0.60);
\addplot[only marks, mark=*, mark size=1.6pt, crouter] coordinates {(300,0.160)};
\node[font=\scriptsize, anchor=west] at (axis cs:320,0.64) {$N{=}300$ (SEW)};
\end{axis}
\end{tikzpicture}
\caption{Theorem~\ref{thm:horizon}: any fixed window decays to chance in $N$.}
\label{fig:horizonplot}
\end{subfigure}
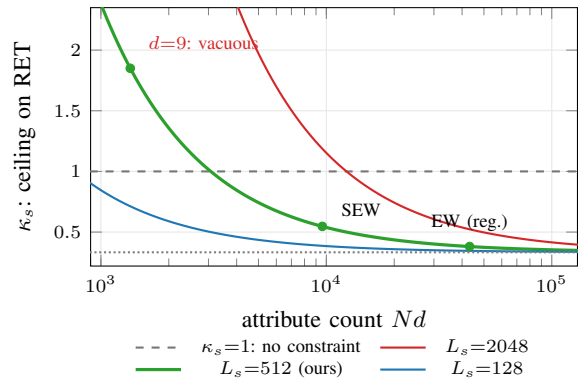
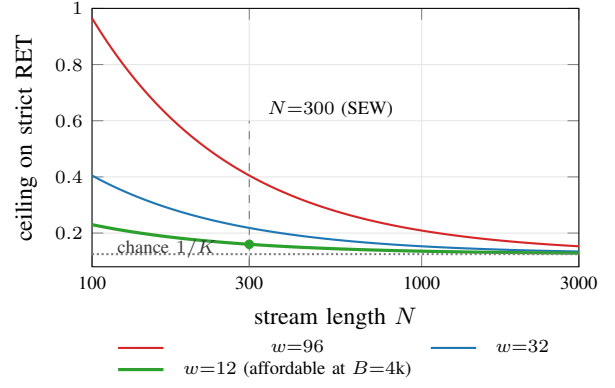
\caption{\textbf{When a bound is worth stating.} Both panels plot the closed forms of
Theorems~\ref{thm:shannon}--\ref{thm:horizon}; nothing is fitted. \textbf{(a)} The Shannon
ceiling $\kappa_s=\frac{m/(Nd)+1}{\log_2K}$ is a constraint only where it falls below $1$,
and whether it does depends on the \emph{attribute} count $Nd$ against the state's $m$ bits
-- not on the keyframe count $N$, which is the substitution that makes the bound look
tighter than it is. The marked points are three corpora at $L_s{=}512$: an earlier version
of ours at $Nd{=}1{,}350$, where $\kappa_s{=}1.85$ and the theorem says \emph{nothing}
(\S\ref{sec:analysis}); SEW-Bench at $Nd{=}9{,}600$, $\kappa_s{=}0.547$; and the registered
EW-Bench at $Nd{=}43{,}200$, $\kappa_s{=}0.381$. Note the direction of the $L_s$ family: a
\emph{larger} compressive state weakens the wall, exactly as it should, and the wall is a
statement about the ratio rather than about compression being bad. \textbf{(b)} The horizon
ceiling $\frac wN+(1-\frac wN)\frac1K$ decays to chance for \emph{any} fixed window, which
is why a longer context defers this wall rather than removing it: an egocentric mission
stream grows without bound while $w$ is bought in tokens. At $N{=}300$ and the $w{=}12$ that
$B{=}4$k actually affords, the ceiling is $16.0\%$.}
\label{fig:regimes}
\end{aspfigure}

\begin{aspfigure}[!t]
\centering
\begin{tikzpicture}[
  strip/.style={draw=gray!55, fill=gray!12},
  hop/.style={draw=crouter!70!black, fill=crouter!45},
  tick/.style={cwall, thick},
  hdr/.style={font=\footnotesize, align=center},
  num/.style={font=\scriptsize, crouter!55!black},
  mth/.style={font=\scriptsize, align=center},
  arr/.style={-{Stealth[length=1.8mm]}, thick}]

\begin{scope}
\node[hdr, anchor=north] at (3.3,3.72) {\textbf{(a)} one non-adaptive read of $k$ frames};
\node[font=\scriptsize, anchor=west, gray!40!black] at (0,2.66) {stream: $N$ frames};
\fill[strip] (0,2.05) rectangle (6.6,2.45);
\foreach \x/\n in {1.10/1, 3.30/2, 5.40/3}{
  \fill[hop] (\x,2.05) rectangle ({\x+0.22},2.45);
  \node[num] at ({\x+0.11},1.86) {hop \n};}
\foreach \x in {0.40, 1.21, 2.20, 3.95, 4.60, 6.10}{
  \draw[tick] (\x,1.20) -- (\x,1.62);}
\draw[decorate, decoration={brace, amplitude=3pt, mirror}, cwall]
  (0.40,1.10) -- (6.10,1.10)
  node[midway, below=2pt, font=\scriptsize, cwall] {$k$ frames, fixed before reading any};
\draw[arr] (3.3,0.70) -- (3.3,0.32);
\node[draw, rounded corners=2pt, fill=yellow!16, font=\small,
      minimum width=3.5cm, minimum height=0.5cm] at (3.3,0.05) {one fixed-depth pass};
\node[mth, anchor=north] at (3.3,-0.34)
  {$\Pr[\text{correct}]\le\dfrac kN+\Bigl(1-\dfrac kN\Bigr)\dfrac1K$\quad for \emph{every} $k$};
\end{scope}

\begin{scope}[xshift=7.9cm]
\node[hdr, anchor=north] at (3.3,3.72) {\textbf{(b)} $R$ adaptive rounds, one frame each};
\node[font=\scriptsize, anchor=west, crouter!55!black] at (0,3.16)
  {each hop's location is known only after the previous read};
\fill[strip] (0,2.05) rectangle (6.6,2.45);
\foreach \x in {1.10, 3.30, 5.40}{
  \fill[hop] (\x,2.05) rectangle ({\x+0.22},2.45);}
\draw[arr, crouter, dashed] (1.32,2.52) to[bend left=26] (3.28,2.52);
\draw[arr, crouter, dashed] (3.52,2.52) to[bend left=26] (5.38,2.52);
\foreach \x/\r/\y in {1.21/1/1.42, 3.41/2/1.02, 5.51/3/0.62}{
  \draw[arr, crouter] (\x,\y) -- (\x,2.02);
  \node[num, anchor=east] at ({\x-0.06},\y) {round \r};}
\node[mth, anchor=north] at (3.3,-0.34)
  {$R\ge d_q$: error $\le d_q\epsilon_1$ \quad\textbar\quad $R<d_q$: the bound in (a),
   at any $B$};
\end{scope}
\end{tikzpicture}
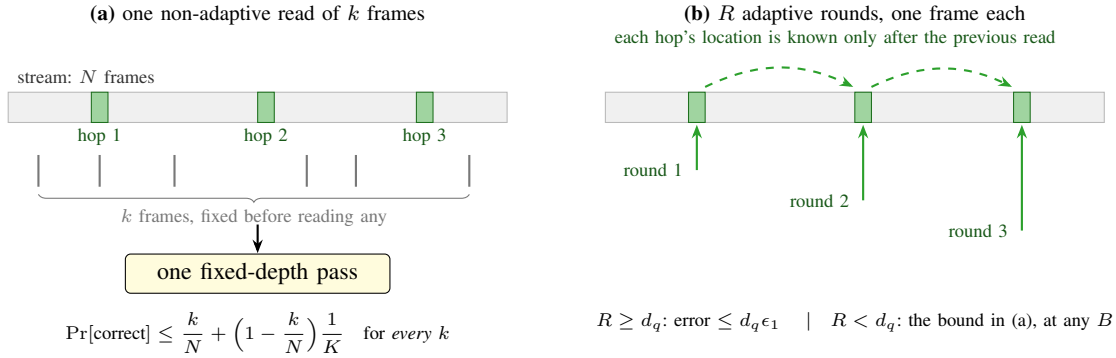
\caption{\textbf{Width is not depth.} The two panels differ in exactly one respect: whether
a read may depend on what an earlier read returned. \textbf{(a)} A single retrieval must
commit to its $k$ frames before seeing any of them. On a depth-$d_q$ dependent chain it can
cover hop~1, whose location the query names, but hop~2's location is a \emph{function of
hop~1's content}, so no query-independent selection can target it, and the bound does not
improve with $k$ until $k$ approaches $N$, which the per-decision budget forbids. That is
Theorem~\ref{thm:rounds}\ref{it:rounds-stateless}, and it holds against unbounded compute;
part~\ref{it:rounds-loc} extends the obstruction to \emph{locating} the answer frame when the
selector does carry state.
\textbf{(b)} The three marked positions are the same three as in (a). $R$ adaptive rounds read one frame each and resolve the chain in $d_q$ steps,
at a token cost of $d_q$ single-frame reads rather than $N$. Rounds are therefore a resource
distinct from context width, priced separately as $B_c$ in Definition~\ref{def:price} and
bought by \ASP{}'s access loop. Two things worth separating here: this panel is the
\emph{only} wall statement that needs no complexity-theoretic assumption, whereas
Theorem~\ref{thm:circuit} extends the same conclusion to a single \emph{wide} forward pass, one that sees all $N$ frames at once, and that extension does rest on
$\mathsf{TC}^0\neq\mathsf{NC}^1$. The case for an access loop does not depend on it.}
\label{fig:rounds}
\end{aspfigure}

Figure~\ref{fig:rounds} isolates the distinction the round wall turns on: two agents reading the same number of frames at the same budget, differing only in whether the second read may depend on the first. Width is not depth, and no amount of the former buys the latter. The four walls eliminate, respectively: compressive-only agents (MovieChat/MA-LMM-style memories \citep{moviechat,malmm}) on $T_{\mathrm{ret}}$; window/sampling agents (uniform multi-frame VLMs \citep{openeqa}) on $T_{\mathrm{ret}}$ and long-$N$ $T_{\mathrm{trk}}$; single-round retrievers with no compressive channel on $T_{\mathrm{cmp}}$ unconditionally (Theorem~\ref{thm:rounds}\ref{it:rounds-stateless}); and single-pass in-context composition conditionally (Theorem~\ref{thm:circuit}). Note the division of labor: Theorems~\ref{thm:shannon}, \ref{thm:horizon} and \ref{thm:rounds} are unconditional information-theoretic bounds; only the in-context depth statement rests on $\mathsf{TC}^0\neq\mathsf{NC}^1$, and the case for \ASP{} survives even if that separation fails. Figure~\ref{fig:walls} summarizes the exclusion pattern, and Figure~\ref{fig:argument} lays out the full dependency structure of the argument, which result rests on which, and which of them rest on an unproven separation.

\begin{aspfigure}[!t]
\centering
\begin{tikzpicture}[scale=0.9, every node/.style={transform shape},
  wall/.style={draw=cwall, fill=cwall!18, minimum width=2.35cm, minimum height=2.5cm, align=center, font=\small},
  ag/.style={draw, rounded corners=2pt, align=center, font=\footnotesize, minimum width=2.9cm, minimum height=0.62cm},
  pass/.style={-{Stealth[length=2mm]}, thick, crouter},
  blocked/.style={-{Stealth[length=2mm]}, thick, cverb!80!black},
  xmark/.style={font=\small\bfseries, cverb!80!black, anchor=west, inner sep=0.5pt}]
\node[wall] (w1) at (3.4,0) {Shannon\\ wall\\ \scriptsize (Thm.~\ref{thm:shannon})};
\node[wall] (w2) at (6.8,0) {Horizon\\ wall\\ \scriptsize (Thm.~\ref{thm:horizon})};
\node[wall] (w3) at (10.2,0) {Composition\\ wall\\ \scriptsize (Thm.~\ref{thm:circuit})};
\node[ag, fill=cstate!14] (a1) at (-0.35,1.6) {compressive-only\\ \scriptsize (scene memory)};
\node[ag, fill=cverb!12] (a2) at (-0.35,0.55) {verbatim-only\\ \scriptsize (window / sampler)};
\node[ag, fill=gray!10] (a3) at (-0.35,-0.55) {single-pass hybrid};
\node[ag, fill=crouter!14] (a4) at (-0.35,-1.6) {\ASP{} ($\cs+\cv+\rho$, $R$ rounds)};
\draw[blocked] (a1.east) -- ($(w1.west)+(-0.26,1.0)$) node[xmark]{$\times$};
\draw[pass] (a2.east) -- ($(w1.west)+(0,0.3)$);
\draw[blocked] ($(w1.east)+(0,0.3)$) -- ($(w2.west)+(-0.26,0.3)$) node[xmark]{$\times$};
\draw[pass] (a3.east) -- ($(w1.west)+(0,-0.35)$);
\draw[pass] ($(w1.east)+(0,-0.35)$) -- ($(w2.west)+(0,-0.35)$);
\draw[blocked] ($(w2.east)+(0,-0.35)$) -- ($(w3.west)+(-0.26,-0.35)$) node[xmark]{$\times$};
\draw[pass] (a4.east) -- ($(w1.west)+(0,-1.0)$);
\draw[pass] ($(w1.east)+(0,-1.0)$) -- ($(w2.west)+(0,-1.0)$);
\draw[pass] ($(w2.east)+(0,-1.0)$) -- ($(w3.west)+(0,-1.0)$);
\draw[pass] ($(w3.east)+(0,-1.0)$) -- ++(1.15,0) node[right, font=\footnotesize, align=left]{access-\\complete};
\node[font=\scriptsize, align=center] at (3.4,-2.45) {price $B_s^{\star}$: state tokens};
\node[font=\scriptsize, align=center] at (6.8,-2.45) {price $B_r^{\star}$: retrieval tokens};
\node[font=\scriptsize, align=center] at (10.2,-2.45) {price $B_c^{\star}{\cdot}R$: rounds};
\end{tikzpicture}
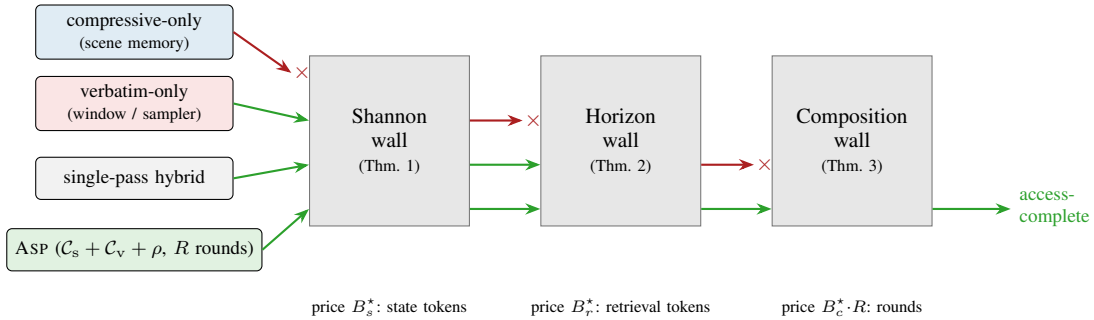
\caption{\textbf{Embodied wall-exclusion pattern.} Each wall eliminates one architecture class on a witness family; only an agent paying all three prices out of the per-decision budget $B$ crosses all walls. Green: passes; red $\times$: provably blocked.}
\label{fig:walls}
\end{aspfigure}

\begin{aspfigure}[!t]
\centering
\begin{tikzpicture}[scale=0.80, every node/.style={transform shape},
  wall/.style={draw, rounded corners=2pt, align=left, font=\scriptsize,
               minimum height=0.92cm, text width=4.1cm, inner sep=3pt},
  mid/.style={draw, rounded corners=2pt, align=left, font=\scriptsize,
              minimum height=0.92cm, text width=4.5cm, inner sep=3pt},
  reg/.style={draw, rounded corners=2pt, align=left, font=\scriptsize,
              minimum height=0.85cm, text width=4.0cm, inner sep=3pt},
  ar/.style={-{Stealth[length=1.9mm]}, gray!70!black},
  arc/.style={-{Stealth[length=1.9mm]}, dashed, cwall},
  hdr/.style={font=\small\bfseries}]

\node[hdr] at (2.3,1.35) {four walls};
\node[hdr] at (8.6,1.35) {what \ASP{} gets};
\node[hdr] at (14.6,1.35) {what is registered};

\node[wall, fill=cstate!12, draw=cstate] (t1) at (2.3,0.30)
  {\textbf{Thm.~\ref{thm:shannon}} Shannon, \emph{uncond.}\\ $m$-bit state $\Rightarrow$ RET $\le\kappa_s$\\ \textcolor{cwall}{kills compressive-only}};
\node[wall, fill=cverb!10, draw=cverb] (t2) at (2.3,-1.28)
  {\textbf{Thm.~\ref{thm:horizon}} Horizon, \emph{uncond.}\\ query-indep.\ $w$ frames $\Rightarrow\frac wN{+}\frac1K$\\ \textcolor{cwall}{kills windows / samplers}};
\node[wall, fill=cverb!10, draw=cverb] (t4) at (2.3,-2.86)
  {\textbf{Thm.~\ref{thm:rounds}} Round, \emph{uncond.}\\ one non-adaptive round, no state, depth $2$\\ \textcolor{cwall}{kills single-round retrieval}};
\node[wall, fill=gray!8, draw=cwall, dashed] (t3) at (2.3,-4.44)
  {\textbf{Thm.~\ref{thm:circuit}} Composition, \emph{cond.}\\ needs $\mathsf{TC}^0\!\neq\!\mathsf{NC}^1$ \emph{and} the $S_5$ encoding\\ \textcolor{cwall}{kills single-pass composition}};

\node[font=\scriptsize, anchor=west, gray!60!black] (l2) at (-1.45,-2.86)
  {\begin{tabular}{@{}l@{}}Lem.~\ref{lem:listfano}\\ list Fano\end{tabular}};
\node[font=\scriptsize, anchor=west, gray!60!black] (l1) at (-1.45,-4.44)
  {\begin{tabular}{@{}l@{}}Lem.~\ref{lem:s5}\\ $S_5$ embed\end{tabular}};
\draw[ar] (l2.east) -- (t4.west);
\draw[ar] (l1.east) -- (t3.west);

\node[mid, fill=crouter!8, draw=crouter] (l3) at (8.6,0.30)
  {\textbf{Lem.~\ref{lem:vtrack}} \emph{uncond.} $\cv$ alone, reading $k_{\mathrm{tot}}{<}N$ frames, is pinned at $1/K$ on TRK. \textcolor{cwall}{Replaces the deferred reduction; the horizon wall is silent here.}};
\node[mid, fill=crouter!14, draw=crouter, text width=4.5cm] (p1) at (8.6,-1.85)
  {\textbf{Prop.~\ref{prop:complete}} access-completeness\\ RET $\ge1{-}(\epsilon_0{+}\delta_r)$, TRK $\ge1{-}N\epsilon_u$, CMP $\ge1{-}d_q\epsilon_1$};
\node[mid, fill=crouter!14, draw=crouter] (p2) at (8.6,-3.55)
  {\textbf{Prop.~\ref{prop:super}} strict super-additivity, on the \emph{conjunctive} witness $T_{\mathrm{conj}}$, if $\Delta<1{-}\kappa_s{-}\frac1K{+}\frac1{K^2}$};
\node[mid, fill=gray!8, draw=cwall] (c1) at (8.6,-5.20)
  {\textbf{Cor.~\ref{cor:dominance}} on the mixture $T$: strict dominance over the \emph{better} single channel. Additive form is unavailable here ($\kappa_s\!>\!1/K$)};

\node[reg, fill=cpred!8, draw=cpred] (c2) at (14.6,0.15)
  {\textbf{Cor.~\ref{cor:conv}} gap $\le c\bar\epsilon$, independent of parameter count};
\node[reg, fill=cpred!8, draw=cpred] (pr) at (14.6,-1.85)
  {\textbf{Pred.~\ref{prd:scissors}} scissors gap\\ \textbf{Pred.~\ref{prd:conv}} $\rho_{\mathrm{scale}}$ halves\\ \textbf{Pred.~\ref{prd:budget}} 4k beats 16k};
\node[reg, fill=gray!8, draw=cwall] (fa) at (14.6,-4.10)
  {\textbf{\S\ref{sec:falsify}} F1--F4 frozen: each prediction has a stated way to fail};

\draw[ar] (t1.east) -- (p1.west);
\draw[ar] (t2.east) -- (p1.west);
\draw[ar] (t4.east) -- (p1.west);
\draw[arc] (t3.east) -- (p1.west);
\draw[ar] (l3.south) -- (p1.north);
\draw[ar] (p1.south) -- (p2.north);
\draw[ar] (p2.south) -- (c1.north);
\draw[ar] (l3.east) to[bend left=12] (c2.west);
\draw[ar] (p1.east) -- (c2.south west);
\draw[ar] (p2.east) -- (pr.west);
\draw[ar] (c1.east) to[bend right=10] (pr.south west);
\draw[ar] (c2.south) -- (pr.north);
\draw[ar] (pr.south) -- (fa.north);

\draw[ar] (-1.50,-6.15) -- (-0.95,-6.15);
\node[font=\scriptsize, anchor=west] at (-0.88,-6.15) {derives};
\draw[arc] (0.95,-6.15) -- (1.50,-6.15);
\node[font=\scriptsize, anchor=west] at (1.57,-6.15) {derives \emph{conditionally}};
\node[font=\scriptsize, anchor=west, gray!60!black] at (5.60,-6.15)
  {dashed box $=$ rests on an unproven separation};
\end{tikzpicture}
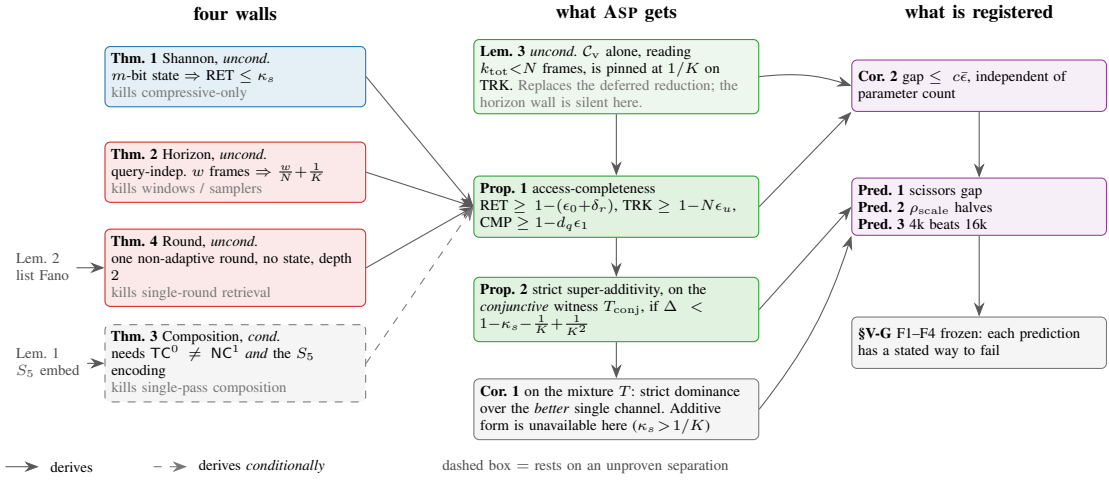
\caption{\textbf{The argument, as a dependency graph.} Three of the four walls, and the tracking lemma that replaces a previously deferred reduction, are unconditional and hold against unbounded compute; only the composition wall (dashed) rests on $\mathsf{TC}^0\neq\mathsf{NC}^1$ and on naturalistic items realising the $S_5$ encoding, and the case for \ASP{} survives without it, Theorem~\ref{thm:rounds} carries the same architectural conclusion unconditionally. Note where the two composition results diverge: super-additivity in additive form needs a witness that obstructs \emph{both} channels at once, which the uniform mixture is not (\S\ref{sec:theory}), so the mixture claim is dominance instead. Everything to the right of Corollary~\ref{cor:conv} is registered before measurement and has a stated failure mode.}
\label{fig:argument}
\end{aspfigure}

\subsection{Access-completeness and super-additivity}

\begin{assumption}[Separability]
\label{ass:sep}
The wall prices are simultaneously affordable: there exist $B_s^{\star},B_r^{\star},B_c^{\star}$ with $B_s^{\star}+B_r^{\star}+RB_c^{\star}\le B$ such that (i) $L_s$ tokens suffice to maintain $\phi$-sufficient statistics for $T_{\mathrm{trk}}$; (ii) top-$k$ retrieval at $B_r^{\star}$ tokens has recall $\ge 1-\delta_r$ on $T_{\mathrm{ret}}$ targets; (iii) $R\ge d_q^{\max}$.
\end{assumption}

\begin{proposition}[Access-completeness of \ASP{}]
\label{prop:complete}
Under Assumptions~\ref{ass:repr}--\ref{ass:sep}, \ASP{} attains $\Cap_B\ge 1-(\epsilon_0+\delta_r)$ on $T_{\mathrm{ret}}$, $\ge 1-N\epsilon_u$ on $T_{\mathrm{trk}}$ (with per-update error $\epsilon_u$), and $\ge 1-d_q\epsilon_1$ on $T_{\mathrm{cmp}}$, while by Theorems~\ref{thm:shannon}--\ref{thm:rounds} each single channel is bounded away from $1$ by a constant on at least one family: writing
$\kappa_s:=\frac{m/(Nd)+1}{\log_2 K}$ for the Theorem~\ref{thm:shannon} ceiling,
$\Cap_B(\cs)\le\kappa_s$ on $T_{\mathrm{ret}}$ and on $T_{\mathrm{cmp}}$, and $\Cap_B(\cv)\le 1/K$ on $T_{\mathrm{trk}}$ (Lemma~\ref{lem:vtrack}).
\end{proposition}

\emph{Remark (the $N\epsilon_u$ factor is the strongest assumption in this paper).} The $T_{\mathrm{trk}}$ bound carries a union bound over all $N$ online updates, and for an \emph{additive} reduction that factor is essentially tight rather than pessimistic: a single missed increment is never recovered, so the accumulated count is wrong from that frame onward. Assumption~\ref{ass:sep}(i) therefore demands a per-keyframe update error below $1/N$, at $N{=}300$, and with the $\Delta<0.344$ condition of Proposition~\ref{prop:super}, that is $\epsilon_u<1.1\times10^{-3}$, i.e.\ a frozen backbone that updates the state correctly on $999$ frames out of $1000$, counting spurious increments on irrelevant frames as errors. Table~\ref{tab:probe} measures $\epsilon_0$ and $\epsilon_1$ but \emph{not} $\epsilon_u$: a per-update probe would have to check the state after each of the $N$ keyframes, which is a different instrument from a single-frame question. So $\epsilon_u$ remains an assumption in the strict sense, and with the $N$ in front of it, it is the one most likely to fail. Two consequences follow and are stated rather than hidden. The theory predicts that \ASP{}'s own TRK accuracy degrades with stream length whenever $\epsilon_u\gtrsim1/N$, the compressive channel removes the Shannon and horizon walls but inherits an error that compounds linearly, and a measured TRK score well below the other families is therefore consistent with the theory, not evidence against it. Parity-type reductions are milder, since two errors can cancel; counts are the hard case, and both are in SEW-Bench.

\begin{proposition}[Strict super-additivity on the conjunctive witness]
\label{prop:super}
On $T_{\mathrm{conj}}$, $\Cap_B(\cs)\le\kappa_s$, $\Cap_B(\cv)\le 1/K$, and $\Cap_B(\varnothing)=1/K^2$, while
$\Cap_B(\cs{+}\cv{+}\rho)\ge 1-\Delta$ with $\Delta:=\epsilon_0+\delta_r+N\epsilon_u$. Hence
\[
\begin{aligned}
\Cap_B(\cs{+}\cv{+}\rho)&>\Cap_B(\cs)+\Cap_B(\cv)-\Cap_B(\varnothing),\\
&\hspace{-1.5em}\text{whenever }\Delta<1-\kappa_s-\frac1K+\frac1{K^2},
\end{aligned}
\]
i.e.\ channel composition is strictly super-additive on a family that needs both channels at once; capability is \emph{purchased}, wall by wall, out of $B$.
\end{proposition}

\emph{Why the witness must be conjunctive.} Super-additivity in this additive form is \emph{not} available on the uniform mixture $T=\tfrac13(T_{\mathrm{ret}}+T_{\mathrm{trk}}+T_{\mathrm{cmp}})$, and it is worth being explicit about why, since the mixture is the natural first guess. There each single channel is near-ceiling on two of the three families, so $\Cap_B(\cs)+\Cap_B(\cv)-\Cap_B(\varnothing)$ exceeds $1$ unless $\kappa_s<1/K$, a condition our own operating point violates ($\kappa_s\approx0.55$ versus $1/K=0.125$ at SEW-Bench parameters, \S\ref{sec:analysis}). No hybrid, however good, can exceed a right-hand side above $1$. On the mixture the correct and still-decisive statement is strict dominance over the \emph{best} single channel:

\begin{corollary}[Strict dominance on the mixture]
\label{cor:dominance}
On $T$, writing $\Delta_T:=\epsilon_0+\delta_r+N\epsilon_u+d_q\epsilon_1$,
\[
\begin{aligned}
\Cap_B(\cs)&\le\tfrac13(2\kappa_s+1),\\
\Cap_B(\cv)&\le\tfrac13(2+1/K),\\
\Cap_B(\cs{+}\cv{+}\rho)&\ge 1-\tfrac13\Delta_T .
\end{aligned}
\]
The hybrid therefore strictly exceeds $\max\{\Cap_B(\cs),\Cap_B(\cv)\}$ by at least
\[
\begin{aligned}
&\tfrac13\bigl(3-\Delta_T-\max\{2\kappa_s+1,\;2+1/K\}\bigr)\\
&\quad=\begin{cases}
\tfrac13\bigl(1-1/K-\Delta_T\bigr), & \kappa_s\le\tfrac12(1+1/K),\\[2pt]
\tfrac23\bigl(1-\kappa_s\bigr)-\tfrac13\Delta_T, & \kappa_s>\tfrac12(1+1/K),
\end{cases}
\end{aligned}
\]
which is positive under Assumption~\ref{ass:sep} and is the quantity the equal-budget comparisons of \S\ref{sec:exp} measure. Which branch applies is a property of the corpus, not of the method: at SEW-Bench parameters $\kappa_s=0.547$ against $\tfrac12(1+1/K)=0.5625$, so the first branch holds, but only just, and on a corpus with a weaker Shannon ceiling the compressive channel, not the verbatim one, is the baseline to beat.
\end{corollary}

\begin{corollary}[Edge capability convergence, registered as Prediction~\ref{prd:conv}]
\label{cor:conv}
Let two backbones both satisfy Assumption~\ref{ass:repr} with margins $\epsilon_0,\epsilon_1\le\bar\epsilon$, let their per-update errors also satisfy $\epsilon_u\le\bar\epsilon$, and let both drive the same index, so that they share a retrieval miss rate $\delta_r$. Then, since $\Delta_T=\epsilon_0+\delta_r+N\epsilon_u+d_q\epsilon_1\le c\bar\epsilon+\delta_r$ with $c=1+N+d_q^{\max}$, both \ASP{} capabilities lie in $[1-\tfrac13(c\bar\epsilon+\delta_r),\,1]$ and the gap between them is at most $\tfrac13(c\bar\epsilon+\delta_r)$, a quantity with no dependence on parameter count. Scale buys margin on $\epsilon_0,\epsilon_1$; once those are small, remaining variance in embodied capability is attributable to $\A$, not to scale.
\end{corollary}

\emph{Remark (two things this corollary is not).} First, $c=1+N+d_q^{\max}$ is dominated by the $N\epsilon_u$ term, so the interval it describes is only narrow when $\bar\epsilon\ll1/N$; convergence is a statement about the regime where the margins are small relative to the stream, not about all backbones everywhere. Second, $\delta_r$ is \emph{not} backbone-independent in our implementation: the index is built from captions the backbone itself writes (\S\ref{sec:method}), so a weaker backbone degrades retrieval as well as reading, and the shared-$\delta_r$ hypothesis is an idealisation. \S\ref{sec:analysis} therefore reports the measured per-backbone retrieval recall alongside the convergence test, so that a failure of convergence can be attributed to the channel that caused it rather than to scale by default.

\section{Method: Access-Structured Perception (\ASP{})}
\label{sec:method}

\ASP{} is a training-free wrapper around a frozen multimodal backbone $M$ reached through any OpenAI-compatible completion endpoint; the measurements here use one API arm (\S\ref{sec:cost}), and nothing in \ASP{} depends on which endpoint answers. It implements $\A=(\cs,\cv,\rho)$ of Definition~\ref{def:agent} with four components (Figure~\ref{fig:overview}).

\subsection{Keyframe gate}
Frames arrive at the sensor rate; a lightweight gate $g$ emits keyframes when perceptual novelty exceeds a threshold: $g(x)=\mathbb{1}[\,1-\cos(e(x),e(x_{\mathrm{last}}))>\tau_g\,]$ with a frozen SigLIP-class embedder $e(\cdot)$ \citep{siglip}. The gate controls $N$ (keyframes per mission), and is shared by all baselines so that comparisons isolate access structure.

\subsection{Compressive channel $\cs$: structured scene state}
The scene state $s_t$ is a typed, capped record (at most $L_s$ tokens) with fields
\[
\begin{aligned}
s_t=\bigl\langle&\ \mathrm{Rooms},\ \mathrm{Objects}\{\mathrm{id},\mathrm{class},\mathrm{room},\\
&\mathrm{state},\mathrm{last\_seen}\},\ \mathrm{AgentTrace},\ \mathrm{EventLog}\ \bigr\rangle,
\end{aligned}
\]
updated by the backbone itself with a fixed update prompt: $s_t=M_{\mathrm{upd}}(s_{t-1},x_t)$, followed by deterministic salience eviction (LRU over $\mathrm{last\_seen}$, protected counters), that enforces $|s_t|\le L_s$. \emph{Design rationale.} A typed record rather than free text or an entity graph, for three reasons: (i) protected counters and toggle fields make the running reductions of $T_{\mathrm{trk}}$ \emph{explicit} sufficient statistics, so state-tracking accuracy degrades with per-update error $\epsilon_u$ rather than with schema drift; (ii) deterministic eviction gives a certified token cap $L_s$, which free-form summaries and unboundedly growing scene graphs \citep{conceptgraphs,sayplan,egagent} do not; (iii), a fixed schema makes the update prompt short and cacheable. The known failure mode is schema mismatch (events outside the ontology land in \textrm{EventLog} as free text); S4 registers its frequency. $\cs$ is exactly the $O(1)$-state channel of CCH: it maintains \emph{running reductions} $\phi$ (counts, open/closed states, containment), that no retrieval strategy reconstructs cheaply, and it deliberately \emph{does not} attempt verbatim appearance storage (Theorem~\ref{thm:shannon} says it cannot).

\subsection{Verbatim channel $\cv$: episodic index}
Every keyframe is stored losslessly off-GPU as $(x_i,\ \hat c_i,\ e(x_i),\ t_i,\ \mathrm{pose}_i)$ where $\hat c_i$ is a one-line caption produced during the update call at no extra request. Retrieval $r(q';k)$ ranks by $\alpha\cos(e_{\mathrm{txt}}(q'),e(x_i))+(1-\alpha)\,\mathrm{BM25}(q',\hat c_i)$ and returns the top-$k$ frames (images re-attached at query time), that fit in $B_r$ tokens. The hybrid score exists because the two cues fail differently: embedding similarity is robust to paraphrase but confuses near-duplicate frames of the same room, while lexical overlap on captions disambiguates instances (``the \emph{second} white mug''), but misses synonyms; $\alpha$ is fixed once on held-out data (S4 registers its sensitivity). $\cv$ is the scalable index channel of CCH: $O(N)$ storage on disk, $O(1)$ GPU context, and (unlike token-pruning approaches that discard information inside the forward pass \citep{fastv,prumerge,tosa}), it is lossless at rest: pruning decides what the model \emph{keeps}, access structure decides what the model \emph{touches}.

One property of this design has to be stated here rather than left to an appendix sweep, because
it bounds what the measurements in \S\ref{sec:exp} can mean. The two cues are not
equally load-bearing on the corpus we ran: at $\alpha{=}1$ (embedding only) recall@$8$ is
$0.005$, indistinguishable from drawing frames at random, while at $\alpha{=}0$ (BM25 on
captions only), it is $0.936$. On a schematic render CLIP image--text similarity carries
essentially nothing, so the index that makes $\cv$ work here is a \emph{text} index over
captions the backbone itself wrote, plus a budgeted read of the frames it returns. That is a
faithful instance of a verbatim channel in the sense of Definition~\ref{def:agent}, the
frames are stored losslessly and touched on demand, and it is not evidence that the visual
half of such an index works. On natural frames the visual half is the half that would have to
work. \S\ref{sec:limits}(0) states the consequence for the paper's claims and
\S\ref{app:traces} gives the full $\alpha$ sweep.

\subsection{Router $\rho$ and iterative access loop}
\label{sec:router}
At decision time the router first classifies the query with a $B_\rho$-token call, $z=\rho_{\mathrm{cls}}(q,s_N)\in\{\textsc{ret},\allowbreak\textsc{trk},\allowbreak\textsc{cmp},\allowbreak\textsc{now}\}$, then allocates the remaining budget by the profile table learned from the theory (defaults in Table~\ref{tab:alloc}). Writing $\mathbf B=(B_s,B_r,B_c)$, it solves
\[
\begin{aligned}
\mathbf B^*&=\arg\max_{\mathbf B}\ \widehat{\Cap}_z(\mathbf B),\\[-1pt]
&\text{s.t. }B_s+B_r+B_c\le B-B_\rho,
\end{aligned}
\]
where $\widehat{\Cap}_z$ is a concave surrogate fitted per class. For the piecewise-linear surrogates we use this reduces to water-filling over the three wall prices, i.e.\ to selecting a row of Table~\ref{tab:alloc} (Figure~\ref{fig:router}a). Figure~\ref{fig:router}b makes the consequence concrete: the allocation is a \emph{cap}, and a decision that finds its answer early returns most of the budget unspent. The access loop then runs at most $R$ rounds; in each round $M$ emits one of
$\textsc{retrieve}(q')$, $\textsc{state}(\mathrm{field})$, $\textsc{answer}(y)$,
and the executed action's tokens are debited from the remaining budget; the loop halts at \textsc{answer} or exhaustion (forced best-guess). Rounds implement the depth multiplication of Theorems~\ref{thm:rounds}--\ref{thm:circuit}. The router is classify-then-allocate rather than learned: the class set is the theory's task typology, the table is initialised at the wall prices, and the only fitted objects are four concave surrogates, which keeps the method training-free and auditable. The ablation $-\rho$ (fixed even split) lower-bounds what routing contributes. Algorithm~\ref{alg:asp} summarizes the full procedure.

\begin{asptable}[!t]
\centering\small
\caption{Default router allocation profiles (fractions of $B-B_\rho$), and the wall each profile pays. $R^{\max}=3$.}
\label{tab:alloc}
\begin{tabular}{lcccl}
\toprule
Class $z$ & $B_s$ & $B_r$ & $B_c$ & Wall paid \\
\midrule
\textsc{ret} (episodic detail) & 0.10 & 0.65 & 0.25 & Shannon $+$ horizon (via $\cv$) \\
\textsc{trk} (cumulative state) & 0.55 & 0.10 & 0.35 & Shannon (via $\cs$ reductions) \\
\textsc{cmp} (compositional) & 0.25 & 0.35 & 0.40 & Composition (rounds) \\
\textsc{now} (immediate percept) & 0.05 & 0.20 & 0.75 & --- \\
\bottomrule
\end{tabular}
\end{asptable}

\begin{aspfigure}[!t]
\centering
\begin{tikzpicture}[scale=0.92, every node/.style={transform shape}]
\def\L{8.0}                        
\def\H{0.46}
\def\g{0.26}
\pgfmathsetmacro{\k}{\L/4055}      
\node[font=\small\bfseries, anchor=west] at (-1.15,{4*(\H+\g)+1.02}) {(a)};
\node[font=\small, anchor=west] at (-0.60,{4*(\H+\g)+1.02})
  {allocation: which wall consumes the budget (Table~\ref{tab:alloc})};
\foreach \frac/\lab in {0/0, 0.25/1k, 0.5/2k, 0.75/3k, 1.0/4k}{
  \draw[gray!40] ({\frac*\L},0.34) -- ({\frac*\L},{4*(\H+\g)+0.62});
  \node[font=\scriptsize, gray!55!black, anchor=north] at ({\frac*\L},0.32) {\lab};
}
\foreach \i/\z/\a/\b/\c/\wall in {%
  3/\textsc{ret}/0.10/0.65/0.25/{Shannon $+$ horizon, via $\cv$},
  2/\textsc{trk}/0.55/0.10/0.35/{Shannon, via $\cs$ reductions},
  1/\textsc{cmp}/0.25/0.35/0.40/{Composition, via rounds},
  0/\textsc{now}/0.05/0.20/0.75/{---}}{
  \pgfmathsetmacro{\y}{\i*(\H+\g)+0.50}
  \node[font=\small, anchor=east] at (-0.15,{\y+0.5*\H}) {\z};
  \pgfmathsetmacro{\xa}{\a*\L}\pgfmathsetmacro{\xb}{(\a+\b)*\L}
  \fill[cstate!30,  draw=cstate]  (0,\y)     rectangle ({\xa},{\y+\H});
  \fill[cverb!25,   draw=cverb]   ({\xa},\y) rectangle ({\xb},{\y+\H});
  \fill[crouter!25, draw=crouter] ({\xb},\y) rectangle ({\L},{\y+\H});
  \pgfmathsetmacro{\wa}{\a*\L}\pgfmathsetmacro{\wb}{\b*\L}\pgfmathsetmacro{\wc}{(1-\a-\b)*\L}
  \ifdim\wa pt>0.7pt \node[font=\scriptsize] at ({0.5*\xa},{\y+0.5*\H}) {\a};\fi
  \ifdim\wb pt>0.7pt \node[font=\scriptsize] at ({0.5*(\xa+\xb)},{\y+0.5*\H}) {\b};\fi
  \ifdim\wc pt>0.7pt \node[font=\scriptsize] at ({0.5*(\xb+\L)},{\y+0.5*\H}) {\c};\fi
  \node[font=\scriptsize, anchor=west, gray!55!black] at ({\L+0.20},{\y+0.5*\H}) {\wall};
}
\begin{scope}[shift={(0,{4*(\H+\g)+0.44})}]
  \fill[cstate!30,  draw=cstate]  (0,0)    rectangle (0.30,0.26);
  \node[font=\scriptsize, anchor=west] at (0.36,0.13) {$B_s$ state};
  \fill[cverb!25,   draw=cverb]   (1.70,0) rectangle (2.00,0.26);
  \node[font=\scriptsize, anchor=west] at (2.06,0.13) {$B_r$ retrieval};
  \fill[crouter!25, draw=crouter] (3.80,0) rectangle (4.10,0.26);
  \node[font=\scriptsize, anchor=west] at (4.16,0.13) {$B_c$ rounds};
\end{scope}
\begin{scope}[shift={(0,-4.10)}]
  \def\Lb{7.2}                       
  \pgfmathsetmacro{\kb}{\Lb/1800}
  \node[font=\small\bfseries, anchor=west] at (-1.15,{3*(\H+\g)+1.50}) {(b)};
  \node[font=\small, anchor=west] at (-0.60,{3*(\H+\g)+1.50})
    {actual spend, one \textsc{cmp} decision (\S\ref{app:worked}), note the zoomed axis};
  \foreach \tok/\lab in {0/0, 500/500, 1000/1000, 1500/1500}{
    \draw[gray!40] ({\tok*\kb},0.30) -- ({\tok*\kb},{3*(\H+\g)+0.95});
    \node[font=\scriptsize, gray!55!black, anchor=north] at ({\tok*\kb},0.28) {\lab};
  }
  \draw[cstate, dashed, thick] ({1014*\kb},0.26) -- ({1014*\kb},{3*(\H+\g)+0.98})
    node[anchor=south, font=\scriptsize, cstate]{$B_s$ cap $1{,}014$};
  \draw[cverb, dashed, thick] ({1419*\kb},0.26) -- ({1419*\kb},{3*(\H+\g)+0.60})
    node[anchor=south, font=\scriptsize, cverb]{$B_r$ cap $1{,}419$};
  \foreach \i/\lab/\tok/\colf/\cold in {%
    3/{$\rho$ classify}/41/{gray!35}/{gray!70},
    2/{R1 \textsc{state}(AgentTrace)}/118/{cstate!30}/{cstate},
    1/{R2 \textsc{retrieve}(mug$\,$in$\,$study)}/742/{cverb!25}/{cverb},
    0/{R3 \textsc{state}$\to$\textsc{answer}}/96/{cstate!30}/{cstate}}{
    \pgfmathsetmacro{\y}{\i*(\H+\g)+0.44}
    \node[font=\scriptsize, anchor=east] at (-0.15,{\y+0.5*\H}) {\lab};
    \fill[\colf, draw=\cold] (0,\y) rectangle ({\tok*\kb},{\y+\H});
    \node[font=\scriptsize, anchor=west] at ({\tok*\kb+0.10},{\y+0.5*\H}) {\tok};
  }
  \node[font=\scriptsize, gray!55!black, anchor=north] at ({0.5*\Lb},-0.06) {tokens debited};
  \node[font=\scriptsize, anchor=west, align=left] at ({\Lb+0.22},{1.5*(\H+\g)+0.44})
    {total $997$\\ of $B{=}4{,}096$\\ ($76\%$ unspent)};
\end{scope}
\end{tikzpicture}
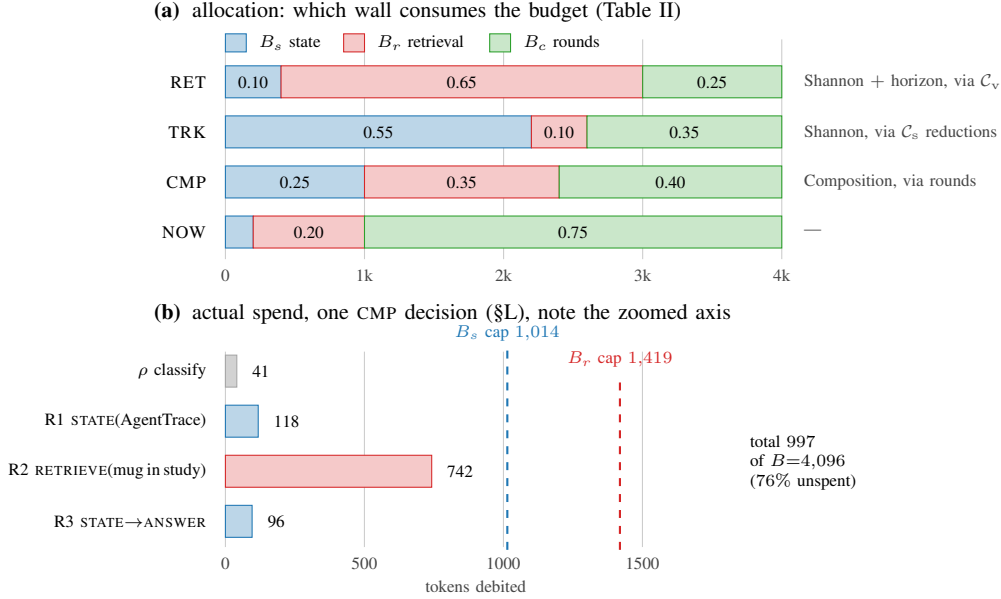
\caption{\textbf{What the router buys, and what a decision actually spends.} (a) Each bar is one query class, partitioned into the three wall prices by the default profile of Table~\ref{tab:alloc}; the water-filling of \S\ref{sec:router} reduces to choosing a row. The classes differ almost entirely in \emph{which} wall consumes the budget, not in how much is available: \textsc{ret} pours $65\%$ into verbatim retrieval because Theorems~\ref{thm:shannon}--\ref{thm:horizon} put episodic detail out of reach any other way, \textsc{trk} inverts that because Lemma~\ref{lem:vtrack} makes retrieval useless for running reductions, and \textsc{cmp} is the only class obliged to pay all three. (b) The allocation is a \emph{cap}, not a target. On the worked $d_q{=}3$ item the loop halts at \textsc{answer} after three rounds having spent $997$ of $4{,}096$ tokens ($214$ against a state cap of $1{,}014$ and $742$ against a retrieval cap of $1{,}419$), so the budget that matters is the one the walls force it to spend, not the one it is handed. Panel (b) is constructed from the protocol, not measured\pred{}.}
\label{fig:router}
\end{aspfigure}

\begin{aspalgorithm}[t]
\caption{\ASP{} per-decision procedure (per-decision budget $B$, rounds $R$, retrieval size $k$)}
\label{alg:asp}
\begin{algorithmic}[1]
\Statex \textbf{Online (per keyframe $x_t$):} if $g(x_t)$: $(s_t,\hat c_t)\gets M_{\mathrm{upd}}(s_{t-1},x_t)$;\ \ $\mathcal{I}_t\gets\mathcal{I}_{t-1}\cup\{(x_t,\hat c_t,e(x_t))\}$
\Statex \textbf{Decision (query $q$ at time $N$):}
\State $z\gets\rho_{\mathrm{cls}}(q,s_N)$ \Comment{$B_\rho$ tokens; class $\in\{\textsc{ret},\textsc{trk},\textsc{cmp},\textsc{now}\}$}
\State $(B_s,B_r,B_c)\gets\mathrm{waterfill}(z,\ B-B_\rho)$ \Comment{Table~\ref{tab:alloc}; pays wall prices of Def.~\ref{def:price}}
\State $C\gets\varnothing$
\For{$r=1$ \textbf{to} $R$ \textbf{while} budget remains}
  \State $o\gets M(q,\ C;\ \text{action set }\{\textsc{retrieve},\textsc{state},\textsc{answer}\})$ \Comment{$\le B_c/R$ tokens}
  \If{$o=\textsc{answer}(y)$} \Return $y$
  \ElsIf{$o=\textsc{state}(f)$} $C\gets C\cup s_N[f]$ \Comment{debit $\le B_s$}
  \ElsIf{$o=\textsc{retrieve}(q')$} $C\gets C\cup M(\text{read } r(q';k))$ \Comment{debit $\le B_r$}
  \EndIf
\EndFor
\State \Return forced best guess from $C$ \Comment{budget exhausted}
\end{algorithmic}
\end{aspalgorithm}

\subsection{Cost model and serving protocol (hardware-agnostic)}
\label{sec:cost}
We make no on-device latency claims; ``edge-scale'' in this paper is a \emph{model-class and budget} statement, not a hardware statement. Per decision, \ASP{} costs at most $B$ backbone tokens across $\le R$ calls, giving an analytic, hardware-agnostic cost of $\mathrm{FLOPs}\approx 2\,P_{\mathrm{act}}\cdot B$ for a dense backbone with $P_{\mathrm{act}}$ active parameters (MoE uses active, not total, parameters, Qwen3-VL-30B-A3B bills 3B/token), KV-cache footprint $O(B)$, and an $O(N)$ flash-resident index that never enters accelerator memory. Per keyframe, one bounded update call ($\le L_s{+}L_x$ tokens) plus one local embedding. Because all methods share the same budget accounting, equal-$B$ comparisons are equal-FLOPs comparisons per backbone up to the constant $P_{\mathrm{act}}$. That last equality is an \emph{analytic} statement about this cost model, not a measurement: on an API arm the serving stack, batching and quantization are not observable to us, so FLOPs per decision cannot be reported and the registered iso-compute check is left unevaluated (Appendix~\ref{app:traces}). Tokens, which is what $B$ actually bounds, are measured exactly from the provider's usage field on every call.

\emph{Serving.} The registration named local controlled serving as primary (one accelerator, vLLM at a pinned revision, BF16 without quantization, prefix caching off, batch size 1, logged tokenizer and config hashes) with an API arm as replication, precisely because that combination eliminates the version, quantization and batching confounds of API serving. \textbf{Only the API arm was run.} What survives of those controls: model slugs verified against the live catalogue before spending anything, temperature $0$, provider-side reasoning disabled, per-call token counts taken from the provider's own usage field, and the served-model fingerprint recorded on every call. What does not survive: provider routing, quantization we cannot inspect, and silent revision changes, the fingerprint detects one after the fact, it does not prevent one. This is a real weakening of the design and \S\ref{sec:limits} counts it as such; adding the local arm and reporting the divergence is future-work item~(iii), not a footnote.

\section{Experiments and Analysis}
\label{sec:exp}

\par\smallskip{\small\noindent
\textbf{Status of numbers.} Tables~\ref{tab:main}--\ref{tab:p3} are \emph{measured}: every cell is the mean over per-question records in the released \texttt{results/cells/}, produced by the frozen protocol of this section. The predictions and falsification criteria were registered \emph{before} these runs, following the methodology of CCH \citep{cch}, and \S\ref{sec:falsify} reports the verdicts against the frozen criteria rather than revising them. Three registered benchmarks were \emph{not} run, and the reason is per-benchmark rather than blanket, so it can be checked. The obstacle is never the annotation release, all three publish their questions, it is the \emph{frames}: OpenEQA's episode histories are rendered from HM3D and ScanNet scenes, VSI-Bench's videos come from ScanNet, ScanNet++ and ARKitScenes, and EgoSchema's clips are Ego4D. HM3D, ScanNet, ScanNet++ and Ego4D each require a signed agreement that we do not hold, and the registered EW-Bench needs habitat-sim on top of an HM3D/ScanNet licence. So a reader who observes that VSI-Bench's metadata is public is right, and it does not help: the metadata indexes video we cannot obtain. Their frozen predictions remain in Table~\ref{tab:registered}, still marked \pred{}, and nothing in this paper claims them as results; \S\ref{sec:limits} states plainly what that costs the argument. What replaces them is \textbf{SEW-Bench}, a licence-free synthetic walkthrough corpus built so that the bounds of \S\ref{sec:theory} are non-vacuous at its parameters, a stand-in for EW-Bench's \emph{structure}, not for natural video.\par}\smallskip

\subsection{Setup}
\label{sec:setup}
\textbf{Backbones (all open-weight, $\le$31B total or active; served through a single API arm, OpenRouter, with pinned model slugs):} Ministral~3B \citep{ministral}; Qwen3-VL-8B \citep{qwen3vl}; Gemma~3~12B-it \citep{gemma3}; Ministral~14B \citep{ministral}; Qwen3.8-27B (dense, native image+video, Apache-2.0) \citep{qwen38}; Qwen3-VL-30B-A3B (MoE, 3B active) \citep{qwen3vl}; Gemma~4~31B-it (dense) \citep{gemma4}. The ladder spans $3\to31$B because the $x$-axis of Prediction~\ref{prd:conv} is parameter count; a ladder that repeated one size could not test it. Ministral~14B is a \emph{post hoc} addition and is labelled as one: the registered ladder left a gap between $8$B and $27$B that only Gemma~3~12B occupied, and Gemma~3~12B is served by a single provider on this arm, so a transient outage there would have removed the whole middle of Prediction~\ref{prd:conv}'s $x$-axis. Adding a second rung in that interval, from a different provider, is a robustness measure and not a selection: it was added before any cell at $14$B was scored, it is reported in every table rather than only where it helps, and \S\ref{sec:falsify} recomputes $\rho_{\mathrm{scale}}$ over the seven-rung ladder including it. Two registered backbones are \emph{not} in it: \textbf{Qwen2.5-VL-7B and Qwen3-Omni-30B-A3B do not exist on the serving arm} (verified against the live catalogue, not assumed), so Ministral~3B takes the small slot and Qwen3-VL-30B-A3B the MoE slot. That substitution is a deviation from the registration and is logged as one; the MoE stand-in is \emph{not} an omni model, so no audio-modality claim is made from it. Temperature $0$ throughout, and provider-side reasoning is disabled: reasoning tokens are billed and would be charged against $B$ (\S\ref{sec:cost}), and \ASP{}'s thesis is that depth comes from explicit access \emph{rounds}, so a backbone thinking longer internally is a confound rather than the mechanism under test.
\textbf{Benchmark:} \textbf{SEW-Bench} (Synthetic Embodied Walkthrough), $80$ questions over $4$ episodes of $N{=}300$ frames each, split $32/24/24$ into \textbf{RET}/\textbf{TRK}/\textbf{CMP} subtests that are one-to-one witnesses of Theorems~\ref{thm:shannon}--\ref{thm:rounds} (Figure~\ref{fig:witness} shows why the three are structurally different problems rather than three difficulty levels). Each frame renders $16$ objects, each carrying two attributes drawn uniformly from a $K{=}8$ alphabet (a colour and a printed capital letter, both readable only from pixels, never from the text label), so $d{=}32$ queryable attribute values per frame and $Nd{=}9{,}600$ per episode. Those two numbers are the whole reason the corpus is built this way: the bounds of \S\ref{sec:theory} scale with the \emph{attribute} count $Nd$, and at a smaller $d$ they are arithmetically vacuous rather than merely weak (\S\ref{sec:analysis} gives the calculation). A RET target is a class appearing in exactly one frame of the episode; filler classes and their attributes are re-randomised every frame, so the entropy the compressive channel must discard is real and not an artefact of a fixed room layout. SEW-Bench is a schematic 2-D render, \emph{not} egocentric video: it tests access structure under a token budget, which is what \S\ref{sec:theory} is about, and it tests nothing about perception in natural scenes (construction protocol, disjoint class pools, and the answerability gate in Appendix~\ref{app:ewbench}).
\textbf{Equal-budget baselines} (identical keyframe gate, identical backbone, identical $B$): \textsc{Blind} (question only); \textsc{Uniform-$w$} multi-frame VLM \citep{openeqa}; \textsc{Socratic} per-frame captions $+$ LLM \citep{socratic}; \textsc{Retr-only} (VideoAgent-style, $\cv$ without $\cs$) \citep{videoagent}; \textsc{State-only} (MovieChat/MA-LMM-style, $\cs$ without $\cv$) \citep{moviechat,malmm}; \textsc{EGAgent} (reimplemented from the official release \citep{egagent}: text entity scene graph $+$ visual/transcript search tools, run under our budget accounting; the strongest published agentic long-video system at submission time, ACL 2026); and \ASP{}. As a \emph{trained skyline} (not budget-matched), we also report published numbers for long-video instruction-tuned models \citep{longvitu} where available. Default budget $B{=}4{,}096$ tokens/decision (at the measured $259$--$432$ tokens/frame this admits $9$--$15$ frames plus prompt and answer, or the equivalent in text), $L_s{=}512$, $R{=}3$, $k{\le}8$.

\begin{aspfigure}[!t]
\centering
\begin{tikzpicture}[scale=0.94, every node/.style={transform shape}]
\def\L{9.0}                       
\pgfmathsetmacro{\kf}{\L/300}     
\def\Hs{0.34}
\tikzset{strip/.style={draw=gray!55, fill=gray!12},
         note/.style={font=\scriptsize, anchor=west, align=left, text width=5.1cm}}

\node[font=\scriptsize, anchor=east, gray!60!black] at (-0.16,1.28) {read budget};
\fill[crouter!35, draw=crouter] (0,1.14) rectangle ({12*\kf},{1.14+\Hs});
\node[font=\scriptsize, anchor=west, crouter] at ({12*\kf+0.14},{1.14+0.5*\Hs})
  {$k_{\mathrm{tot}}\!=\!12$ keyframes at $B{=}4$k, drawn to scale against $N\!=\!300$};

\node[font=\small\bfseries, anchor=east] at (-0.16,0.30) {RET};
\fill[strip] (0,0.13) rectangle ({\L},{0.13+\Hs});
\fill[cverb!60, draw=cverb] ({79*\kf},0.13) rectangle ({81*\kf},{0.13+\Hs});
\node[font=\scriptsize, cverb, anchor=south] at ({80*\kf},{0.13+\Hs+0.04}) {target};
\node[note, gray!25!black] at ({\L+0.20},{0.13+0.5*\Hs})
  {visible in exactly $1$ of $300$ frames; nothing else in the stream determines the answer};

\node[font=\small\bfseries, anchor=east] at (-0.16,-0.72) {TRK};
\fill[strip] (0,-0.89) rectangle ({\L},{-0.89+\Hs});
\foreach \f in {15,47,89,112,161,198,229,259,281}{
  \fill[cstate!70, draw=cstate] ({\f*\kf},-0.89) rectangle ({(\f+2.5)*\kf},{-0.89+\Hs});
}
\node[note, gray!25!black] at ({\L+0.20},{-0.89+0.5*\Hs})
  {answer is a reduction over \emph{every} mark; missing one leaves it uniform (Lem.~\ref{lem:vtrack})};

\node[font=\small\bfseries, anchor=east] at (-0.16,-1.74) {CMP};
\fill[strip] (0,-1.91) rectangle ({\L},{-1.91+\Hs});
\foreach \f/\lab in {43/1, 139/2, 251/3}{
  \fill[crouter!55, draw=crouter] ({\f*\kf},-1.91) rectangle ({(\f+3)*\kf},{-1.91+\Hs});
  \node[font=\scriptsize, crouter, anchor=north] at ({(\f+1.5)*\kf},{-1.91-0.04}) {hop \lab};
}
\draw[-{Stealth[length=1.6mm]}, crouter, thick]
  ({46*\kf},{-1.91+\Hs+0.03}) to[bend left=26] ({139*\kf},{-1.91+\Hs+0.03});
\draw[-{Stealth[length=1.6mm]}, crouter, thick]
  ({142*\kf},{-1.91+\Hs+0.03}) to[bend left=26] ({251*\kf},{-1.91+\Hs+0.03});
\node[note, gray!25!black] at ({\L+0.20},{-1.91+0.5*\Hs})
  {hop $t{+}1$'s target is unknown until hop $t$ is read $\Rightarrow$ needs $R\!\ge\!d_q$ \emph{adaptive} rounds};

\node[font=\scriptsize, anchor=east, cverb]   at (-0.16,0.02)  {$\cv$};
\node[font=\scriptsize, anchor=east, cstate]  at (-0.16,-1.00) {$\cs$};
\node[font=\scriptsize, anchor=east, crouter] at (-0.16,-2.02) {$R$};

\node[font=\scriptsize, gray!55!black, anchor=north] at ({0.5*\L},-2.52)
  {frame index along one SEW-Bench episode ($N\!=\!300$)};
\end{tikzpicture}
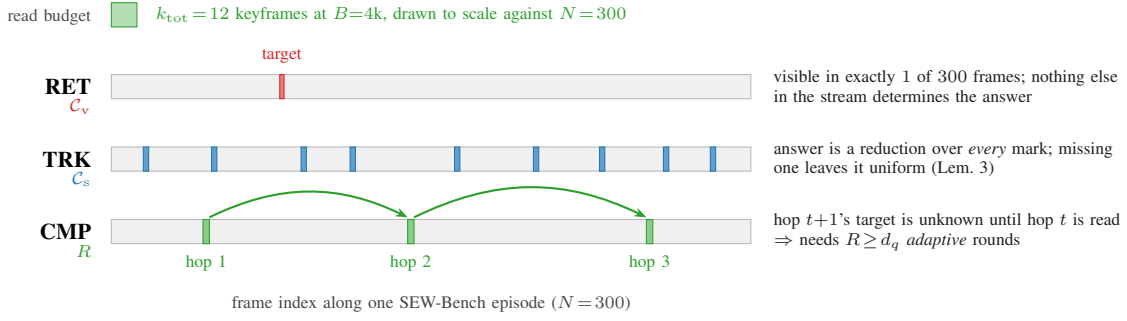
\caption{\textbf{Why the three subtests are different problems.} Each strip is one episode's frame stream, and the green bar at the top is what a single $4$k-token decision can actually read, $12$ of $300$ frames at the measured $302$ tokens/frame, drawn to scale. \textsc{ret} hides the answer in one or two frames, so the only way in is query-conditioned verbatim retrieval; every query-independent selection of $12$ frames misses it with probability $1-12/300=96\%$ (Theorem~\ref{thm:horizon}). \textsc{trk} spreads the answer over every marked event, so no subset of $12$ frames suffices at any budget and the reduction must instead be accumulated online in $\cs$ (Lemma~\ref{lem:vtrack}). \textsc{cmp} makes each hop's location depend on the previous hop's content, so the reads must be \emph{sequenced}, not merely numerous (Theorems~\ref{thm:circuit}--\ref{thm:rounds}). The three failure modes are orthogonal by construction, which is what lets the collapse pattern of Figure~\ref{fig:collapse} be diagnostic. Marks are illustrative of the construction protocol (Appendix~\ref{app:ewbench}); the counts and the scale are the corpus's real ones, the individual positions are not.}
\label{fig:witness}
\end{aspfigure}

\paragraph{Metrics.} All three SEW-Bench subtests are scored by \emph{exact match} on a closed answer alphabet (a colour, a letter tag, a count, \textsc{open}/\textsc{closed}, \textsc{before}/\textsc{after}), which is what Table~\ref{tab:main} registers them as. A rule decides whenever the model's answer contains exactly one token of the relevant alphabet, so a verbose but correct answer is not punished for verbosity; when it contains none or several the item is deferred to the frozen LLM-Match judge (Gemma~3~27B-it, disjoint from every backbone) rather than being scored $0$ by default. The deferral rate is stored with each cell precisely because a high one would mean the accuracy figure is partly a judge figure; an empty answer is scored $0$ without deferral. Per-question records are released for every cell, which is what makes the paired tests below possible: a committed mean cannot support any of the registered significance claims.

\paragraph{Implementation details.} Images enter at the backbone's native low-resolution tokenization, and the per-model token cost of a frame is \emph{measured} rather than assumed: we send an identical prompt with $0$, $1$ and $3$ frames and difference the reported prompt tokens. The cost is linear in the frame count for every model here, at $259$ tokens/frame (Gemma~3~12B), $268$ (Gemma~4~31B), $302$ (Qwen3-VL-8B, Qwen3-VL-30B-A3B, Qwen3.8-27B), and $432$ (Ministral~3B, Ministral~14B); These are properties of each visual tokenizer rather than of our render resolution, confirmed by recalibrating on the SEW frames. The consequence is material: at $B{=}4{,}096$ those constants admit at most $w{=}15$, $14$, $12$ and $9$ frames once prompt and answer are charged, so the registered fixed $w{=}16$ is unaffordable for every backbone in the class. We therefore set $w$ per backbone to the largest uniform sample the budget affords and report the realised $w$ with every cell, since equal budget means equal \emph{tokens} rather than equal frames. The budget is enforced \emph{before} each call rather than audited after, so an oversized request is refused; auditing afterwards truncates many-small-call methods while letting one-big-call methods overrun. All prompts (update, router, loop, baselines, judge) are frozen verbatim in the released protocol; prompt tokens are counted against $B$ without exception. On SEW-Bench the keyframe gate is the identity, $\tau_g{=}0$, because the corpus already emits one frame per discrete moment: at $\tau_g{=}0.02$ the gate collapsed $300$ frames to $29$, which puts $N/w<2$ and makes Theorem~\ref{thm:horizon} say almost nothing. The offline phase (one bounded update call per keyframe, metered by its own per-keyframe allowance and not from $B$, per Definition~\ref{def:agent}) is cached per (backbone, episode), and reused by all $7$ methods $\times$ $3$ subtests, which is what makes the grid affordable: \VnBackbones{} backbones $\times$ $4$ episodes $\times$ $300$ keyframes of offline updates, then $\VnBackbones{}\times7\times80$ scored decisions plus the ablation and $B{=}16$k arms, for a total API spend under \$$15$.

\subsection{Registered headline predictions}
\begin{prediction}[Embodied scissors gap]\label{prd:scissors}
On EW-Bench-RET, \textsc{State-only} scores $\le 30\%$\pred{} and \textsc{Uniform-}$w$ $\le 38\%$\pred{} for all five backbones (walls bind regardless of scale), while \ASP{} scores $\ge 78\%$\pred{} with Qwen3.8-27B, a reproduction, in pixels, of CCH's 0.994-vs-0.000 retrieval scissors.
\end{prediction}
\begin{prediction}[Capability convergence across scale]\label{prd:conv}
Across the five backbones, Spearman correlation between parameter count and mean embodied score is $\rho_{\mathrm{scale}}=0.85$\pred{} under \textsc{Uniform-}$w$ but $\le 0.40$\pred{} under \ASP{}; the 8B-vs-27B gap on OpenEQA shrinks from $\ge 7$\,pts\pred{} to $\le 3$\,pts\pred{}.
\end{prediction}
\begin{prediction}[Reallocation beats growth]\label{prd:budget}
\ASP{} at $B{=}4$k exceeds \textsc{Uniform} at $B{=}16$k on OpenEQA and EW-Bench mean\pred{}; quadrupling an ill-structured budget is worth less than routing a small one.
\end{prediction}

\paragraph{Registered statistics.} Headline claims are tested on \emph{paired per-question deltas}, not absolute scores: for Prediction~\ref{prd:scissors}/\ref{prd:budget} a one-sided Wilcoxon signed-rank test at $\alpha{=}0.01$ with registered minimum effect sizes (paired \ASP{}$-$\textsc{Uniform-}$w$ delta $\ge{+}6$ on OpenEQA, $\ge{+}35$ on EW-RET); for Prediction~\ref{prd:conv}, bootstrap CIs over questions for $\rho_{\mathrm{scale}}$. Effect-size registration makes the headlines robust to a uniform up- or down-shift of all absolute scores (e.g.\ judge strictness), which absolute intervals are not.

\paragraph{How these are scored against what was actually run.}\label{sec:predictions-map} The three predictions above are reproduced verbatim from the registration and are not edited here. Two of them name benchmarks that were not run, so the mapping has to be stated rather than assumed, and it is not free.
\emph{(i) Corpus.} ``EW-Bench-RET'' is evaluated as SEW-RET. The two share the witness \emph{construction} (multiplicity-one targets, $K{=}8$ alphabet, one-to-one with Theorem~\ref{thm:shannon}) but not the parameters, which matters for the \emph{absolute} thresholds: $\le30$/$\le38$/$\ge78$ were derived at $\kappa_s\approx0.38$, whereas SEW-Bench gives $\kappa_s{=}0.547$. An absolute threshold transported across corpora is not a fair test and we do not treat it as one; the \emph{paired effect-size} clauses transport cleanly, and \S\ref{sec:falsify} reports both.
\emph{(ii) Backbones.} ``All five backbones'' is evaluated on the \VnBackbones{}-rung ladder of \S\ref{sec:setup}, with the two substitutions and the one addition declared there.
\emph{(iii) Clauses that cannot be evaluated at all.} Prediction~\ref{prd:conv}'s OpenEQA 8B-vs-27B sub-clause and Prediction~\ref{prd:budget}'s OpenEQA clause have no data. They are recorded as \emph{not evaluated} rather than passed, and Table~\ref{tab:registered} keeps the untested half visible.

\subsection{Main results}

Table~\ref{tab:main} gives the full grid and Figure~\ref{fig:curves} the two curves the headline predictions live on. The thing to read first is not the size of \ASP{}'s margin but its \emph{location}. A wrapper that simply prompts better would lift every subtest roughly uniformly; the account here predicts something narrower: each single channel is weak precisely on the subtest witnessing \emph{its} wall, so \textsc{State-only} should be the weaker of the two single channels on RET and \textsc{Retr-only} the weaker on TRK. That ordering holds on \VlocCount{} of the \VnMeasured{} backbones measured. Aggregate dominance is the coarser statement: \ASP{} has the higher mean than every one of the six equal-budget baselines on \VdomCount{} of \VnBackbones{} backbones, and does so at the registered $\alpha{=}0.01$ on paired per-question deltas against all six on \VsigCount{} of them. On the flagship the paired \ASP{}$-$\textsc{Uniform-}$w$ delta is \VpairedFlagDelta{} points over $n{=}\VpairedFlagN{}$ items ($p=\VpairedFlagP{}$). Absolute levels are low throughout, since these are $3$--$31$B models recovering one attribute of one object from a $300$-frame stream through a $4$k-token keyhole; the comparison that carries the argument is between access structures at equal budget, not against a ceiling.

\begin{asptable}[!htbp]
\centering\footnotesize\setlength{\tabcolsep}{5pt}
\caption{\textbf{SEW-Bench at a fixed per-decision budget $B{=}4{,}096$ tokens.} Accuracy (\%) by exact match on the closed answer alphabet (\S\ref{sec:setup}); $n{=}32$ / $24$ / $24$ questions for RET / TRK / CMP over four $N{=}300$-frame episodes. \emph{Mean} is the unweighted mean of the three witness subtests. \textsc{Uniform-}$w$ uses the largest $w$ the budget affords at that backbone's measured tokens/frame (rightmost column); no backbone affords the $w{=}16$ of the pre-registration. \textbf{Chance is $12.5\%$ on RET and TRK ($K{=}8$) but $25.0\%$ on CMP}, which mixes sixteen $K{=}8$ colour items with eight binary ones whose golds are balanced $4$--$4$ by construction. The CMP cells here are a re-run against the corrected corpus, after eight of those items were found to carry a degenerate gold in the corpus first measured; Table~\ref{tab:cmpfix} gives both runs and \S\ref{sec:cmpdefect} the account. Every value is measured -- the three registered benchmarks with no licence-free data path stay in Table~\ref{tab:registered} as predictions.}
\label{tab:main}
\begin{tabular}{llcccc@{\hskip 10pt}c}
\toprule
Backbone & Method & SEW-RET & SEW-TRK & SEW-CMP & Mean & $w$ \\
\midrule
\multirow{7}{*}{Ministral~3B}
 & \textsc{Blind} & 0.0 & 8.3 & 24.0 & 10.8 &  \\
 & \textsc{Uniform-}$w$ & 18.8 & 17.7 & \textbf{43.8} & 26.7 & 8 \\
 & \textsc{Socratic} & 10.9 & 12.5 & 12.5 & 12.0 &  \\
 & \textsc{Retr-only} & 65.6 & 27.1 & 41.7 & 44.8 &  \\
 & \textsc{State-only} & 6.2 & 16.7 & 8.3 & 10.4 &  \\
 & \textsc{EGAgent} (reimpl.) & 0.0 & \textbf{53.1} & 0.0 & 17.7 &  \\
 & \ASP{} (ours) & \textbf{90.6} & 16.7 & 33.3 & \textbf{46.9} &  \\
\midrule
\multirow{7}{*}{Qwen3-VL-8B}
 & \textsc{Blind} & 6.2 & 8.3 & 16.7 & 10.4 &  \\
 & \textsc{Uniform-}$w$ & 9.4 & 8.3 & \textbf{33.3} & 17.0 & 13 \\
 & \textsc{Socratic} & 0.0 & 16.7 & 14.6 & 10.4 &  \\
 & \textsc{Retr-only} & 90.6 & 4.2 & 29.2 & 41.3 &  \\
 & \textsc{State-only} & 6.2 & 8.3 & 12.5 & 9.0 &  \\
 & \textsc{EGAgent} (reimpl.) & 0.0 & \textbf{34.4} & 0.0 & 11.5 &  \\
 & \ASP{} (ours) & \textbf{93.8} & 8.3 & 29.2 & \textbf{43.8} &  \\
\midrule
\multirow{7}{*}{Gemma~3~12B}
 & \textsc{Blind} & 3.9 & 16.7 & 20.8 & 13.8 &  \\
 & \textsc{Uniform-}$w$ & 3.1 & 20.8 & 25.0 & 16.3 & 15 \\
 & \textsc{Socratic} & 3.1 & 16.7 & 4.2 & 8.0 &  \\
 & \textsc{Retr-only} & 68.8 & 16.7 & 29.2 & 38.2 &  \\
 & \textsc{State-only} & 0.0 & 12.5 & 4.2 & 5.6 &  \\
 & \textsc{EGAgent} (reimpl.) & 0.0 & \textbf{55.2} & 0.0 & 18.4 &  \\
 & \ASP{} (ours) & \textbf{75.0} & 12.5 & \textbf{33.3} & \textbf{40.3} &  \\
\midrule
\multirow{7}{*}{Ministral~14B}
 & \textsc{Blind} & 6.2 & 8.3 & 16.7 & 10.4 &  \\
 & \textsc{Uniform-}$w$ & 6.2 & 8.3 & 36.5 & 17.0 & 8 \\
 & \textsc{Socratic} & 3.9 & 20.8 & 16.7 & 13.8 &  \\
 & \textsc{Retr-only} & \textbf{96.9} & 15.6 & \textbf{54.2} & 55.6 &  \\
 & \textsc{State-only} & 0.0 & 8.3 & 0.0 & 2.8 &  \\
 & \textsc{EGAgent} (reimpl.) & 3.1 & \textbf{80.2} & 0.0 & 27.8 &  \\
 & \ASP{} (ours) & 90.6 & 8.3 & 45.8 & \textbf{48.3} &  \\
\midrule
\multirow{7}{*}{Qwen3.8-27B}
 & \textsc{Blind} & 3.1 & 0.0 & 16.7 & 6.6 &  \\
 & \textsc{Uniform-}$w$ & 4.7 & 16.7 & 20.8 & 14.1 & 13 \\
 & \textsc{Socratic} & 15.6 & 16.7 & 16.7 & 16.3 &  \\
 & \textsc{Retr-only} & \textbf{100.0} & 8.3 & \textbf{71.9} & 60.1 &  \\
 & \textsc{State-only} & 3.1 & 4.2 & 16.7 & 8.0 &  \\
 & \textsc{EGAgent} (reimpl.) & 0.0 & \textbf{74.0} & 0.0 & 24.7 &  \\
 & \ASP{} (ours) & 93.8 & 4.2 & 20.8 & \textbf{39.6} &  \\
\midrule
\multirow{7}{*}{Qwen3-VL-30B-A3B}
 & \textsc{Blind} & 0.0 & 8.3 & 25.0 & 11.1 &  \\
 & \textsc{Uniform-}$w$ & 9.4 & 12.5 & 31.2 & 17.7 & 13 \\
 & \textsc{Socratic} & 4.7 & 19.8 & 29.2 & 17.9 &  \\
 & \textsc{Retr-only} & \textbf{78.1} & 8.3 & \textbf{33.3} & 39.9 &  \\
 & \textsc{State-only} & 0.0 & 16.7 & 16.7 & 11.1 &  \\
 & \textsc{EGAgent} (reimpl.) & 0.0 & \textbf{46.9} & 0.0 & 15.6 &  \\
 & \ASP{} (ours) & \textbf{78.1} & 16.7 & 25.0 & \textbf{39.9} &  \\
\midrule
\multirow{7}{*}{Gemma~4~31B}
 & \textsc{Blind} & 0.0 & 0.0 & 4.2 & 1.4 &  \\
 & \textsc{Uniform-}$w$ & 6.2 & 20.8 & 0.0 & 9.0 & 14 \\
 & \textsc{Socratic} & 0.0 & 25.0 & 4.2 & 9.7 &  \\
 & \textsc{Retr-only} & 87.5 & 8.3 & 75.0 & 56.9 &  \\
 & \textsc{State-only} & 0.0 & 41.7 & 4.2 & 15.3 &  \\
 & \textsc{EGAgent} (reimpl.) & 0.0 & \textbf{62.5} & 0.0 & 20.8 &  \\
 & \ASP{} (ours) & \textbf{90.6} & 25.0 & \textbf{79.2} & \textbf{64.9} &  \\
\bottomrule
\end{tabular}
\end{asptable}

\begin{aspfigure}[!t]
\centering
\begin{subfigure}{0.48\textwidth}
\centering
\begin{tikzpicture}
\begin{axis}[width=0.92\linewidth, height=5.4cm, grid=major, grid style={gray!20}, legend style={font=\scriptsize, at={(0.5,-0.30)}, anchor=north, legend columns=2, draw=none, fill=none, column sep=1.2ex}, legend cell align=left, label style={font=\small}, tick label style={font=\scriptsize}, xlabel={per-decision budget $B$ (tokens, log scale)}, ylabel={SEW-Bench mean (\%)}, xmode=log, log basis x=2, xtick={1024,2048,4096,8192,16384}, xticklabels={1k,2k,4k,8k,16k}]
\addplot[cwall, thick, mark=square*] coordinates {(1024,14.6) (2048,12.8) (4096,14.1) (8192,17.0) (16384,29.5)};
\addlegendentry{\textsc{Uniform-}$w$}
\addplot[crouter, very thick, mark=*] coordinates {(1024,18.4) (2048,28.8) (4096,39.6) (8192,61.1) (16384,60.4)};
\addlegendentry{\ASP{}}
\end{axis}
\end{tikzpicture}
\caption{Budget--capability, Qwen3.8-27B.}
\label{fig:budget}
\end{subfigure}\hfill
\begin{subfigure}{0.48\textwidth}
\centering
\begin{tikzpicture}
\begin{axis}[width=0.92\linewidth, height=5.4cm, grid=major, grid style={gray!20}, legend style={font=\scriptsize, at={(0.5,-0.30)}, anchor=north, legend columns=2, draw=none, fill=none, column sep=1.2ex}, legend cell align=left, label style={font=\small}, tick label style={font=\scriptsize}, xlabel={backbone parameters (B, log scale)}, ylabel={SEW-Bench mean (\%)}, xmode=log, log basis x=2, xtick={3,8,12,27,31}, xticklabels={3,8,12,27,31}, xmin=2.6, xmax=36]
\addplot[cwall, thick, mark=square*] coordinates {(3,26.7) (8,17.0) (12,16.3) (14,17.0) (27,14.1) (30,17.7) (31,9.0)};
\addlegendentry{\textsc{Uniform-}$w$: $\rho_{\mathrm{scale}}{=}-0.58$}
\addplot[crouter, very thick, mark=*] coordinates {(3,46.9) (8,43.8) (12,40.3) (14,48.3) (27,39.6) (30,39.9) (31,64.9)};
\addlegendentry{\ASP{}: $\rho_{\mathrm{scale}}{=}0.00$}
\end{axis}
\end{tikzpicture}
\caption{Capability convergence across scale (Pred.~\ref{prd:conv}).}
\label{fig:conv}
\end{subfigure}
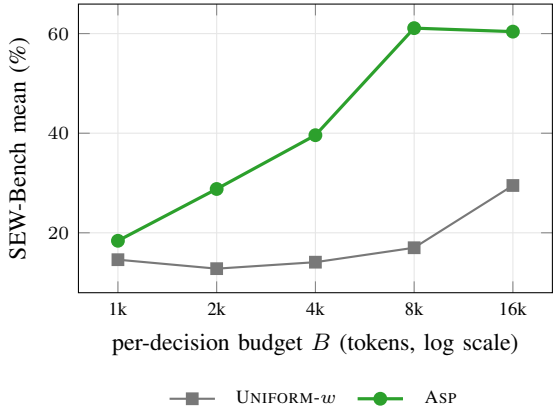
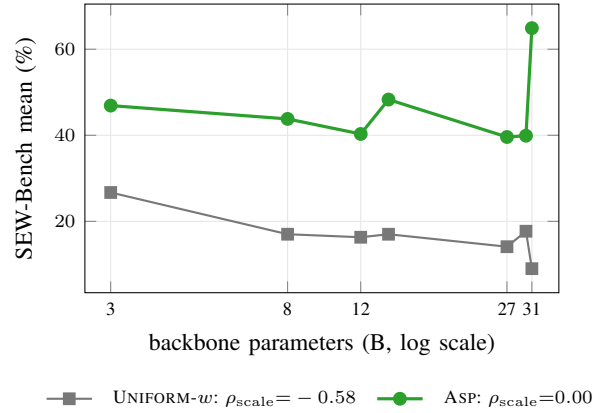
\caption{\textbf{Measured budget and scale curves.} (a) Whether restructuring a small budget beats enlarging an unstructured one, read as a curve rather than a single contrast: the \ASP{} point at $B{=}4$k is to be compared with the \textsc{Uniform-}$w$ point at $B{=}16$k. (b) Score against parameter count for the two ends of the comparison. The rank correlation between parameter count and score falls from $\rho_{\mathrm{scale}}{=}-0.58$ under \textsc{Uniform-}$w$ to $0.00$ under \ASP{}. Missing points are budgets or backbones with no cell, left absent rather than interpolated.}
\label{fig:curves}
\end{aspfigure}
\subsection{Ablations}

\begin{asptable}[!t]
\centering\small
\caption{\textbf{Ablations on Qwen3.8-27B, $B{=}4{,}096$.} Each variant is the same access loop with one component disabled, so a drop is attributable to the removed component and not to a different code path. $\downarrow$ marks the cell the pre-registration predicts will collapse: $-\cv$ on RET (Thm.~\ref{thm:shannon}), $-\cs$ on TRK (Lem.~\ref{lem:vtrack}), $R{=}1$ on CMP (Thm.~\ref{thm:rounds}).}
\label{tab:ablation}
\begin{tabular}{lcccc}
\toprule
Variant & SEW-RET & SEW-TRK & SEW-CMP & Mean \\
\midrule
\ASP{} (full) & 93.8 & 4.2 & 20.8 & 39.6 \\
$-\ \cv$ (no episodic index) & 9.4\,$\downarrow$ & 4.2 & 25.0 & 12.8 \\
$-\ \cs$ (no scene state) & 90.6 & 12.5\,$\downarrow$ & 83.3 & 62.2 \\
$-\ \rho$ (fixed 1/3 split) & 90.6 & 8.3 & 62.5 & 53.8 \\
$R{=}1$ (single pass) & 50.0 & 4.2 & 20.8\,$\downarrow$ & 25.0 \\
$R{=}5$ & 93.8 & 4.2 & 20.8 & 39.6 \\
random retrieval ($k{=}8$) & 0.0 & 4.2 & 8.3 & 4.2 \\
\bottomrule
\end{tabular}
\end{asptable}

\begin{asptable}[!t]
\centering\small\setlength{\tabcolsep}{4.5pt}
\caption{\textbf{What the corpus defect was worth: the CMP column before and after the fix.} Eight of the twenty-four \textsc{cmp} items carried the gold \textsf{before} in the corpus this paper first measured, so a constant answer scored $8/8$ on them (D-43). \emph{pre}: the superseded run, retained in \texttt{results/cells\_pre\_cmpfix/}. \emph{post}: the same cells re-run against the corrected corpus, in which the naming order is chosen to give an exactly $4$--$4$ gold balance. Both columns are $24$ items at chance $25.0\%$ -- the fix removes an exploitable prior, not a chance-level difference, since \textsc{cmp} mixes sixteen $K{=}8$ items with eight $K{=}2$ ones. The \textsc{Blind} pair is the diagnostic to read first: an agent that sees no frames at all should sit at chance, and only the post column is entitled to. Every \textsc{cmp} figure elsewhere in this paper is a \emph{post} figure.}
\label{tab:cmpfix}
\begin{tabular}{@{}lcccccc@{}}
\toprule
& \multicolumn{2}{c}{\ASP{}} & \multicolumn{2}{c}{\textsc{Retr-only}} & \multicolumn{2}{c}{\textsc{Blind}} \\
\cmidrule(lr){2-3} \cmidrule(lr){4-5} \cmidrule(lr){6-7}
Backbone & pre & post & pre & post & pre & post \\
\midrule
Ministral~3B & 37.5 & 33.3 & 41.7 & 41.7 & 36.5 & 24.0 \\
Qwen3-VL-8B & 25.0 & 29.2 & 29.2 & 29.2 & 34.4 & 16.7 \\
Gemma~3~12B & 45.8 & 33.3 & 37.5 & 29.2 & 37.5 & 20.8 \\
Ministral~14B & 41.7 & 45.8 & 58.3 & 54.2 & 8.3 & 16.7 \\
Qwen3.8-27B & 8.3 & 20.8 & 37.5 & 71.9 & 0.0 & 16.7 \\
Qwen3-VL-30B-A3B & 42.7 & 25.0 & 50.0 & 33.3 & 40.6 & 25.0 \\
Gemma~4~31B & 83.3 & 79.2 & 75.0 & 75.0 & 13.5 & 4.2 \\
\midrule
\multicolumn{7}{@{}l}{\emph{flagship ablations, the cells F3 is computed on}} \\
\quad \ASP{} & 8.3 & 20.8 & \multicolumn{4}{c}{---} \\
\quad $-\cs$ & 70.8 & 83.3 & \multicolumn{4}{c}{---} \\
\quad $-\cv$ & 25.0 & 25.0 & \multicolumn{4}{c}{---} \\
\quad $R{=}1$ & 29.2 & 20.8 & \multicolumn{4}{c}{---} \\
\bottomrule
\end{tabular}
\end{asptable}

\begin{aspfigure}[!t]
\centering
\begin{tikzpicture}[scale=0.94, every node/.style={transform shape}]
\def\W{2.75}\def\Hh{1.16}
\node[font=\small\bfseries] at ({0*\W+0.5*\W},{0.42}) {SEW-RET};
\node[font=\scriptsize, gray!60!black] at ({0*\W+0.5*\W},{-3*\Hh-0.26}) {Thm.~\ref{thm:shannon}\,$+$\,\ref{thm:horizon}};
\node[font=\scriptsize, gray!60!black] at ({0*\W+0.5*\W},{-3*\Hh-0.58}) {decisive: $\cv$};
\node[font=\small\bfseries] at ({1*\W+0.5*\W},{0.42}) {SEW-TRK};
\node[font=\scriptsize, gray!60!black] at ({1*\W+0.5*\W},{-3*\Hh-0.26}) {Lem.~\ref{lem:vtrack}};
\node[font=\scriptsize, gray!60!black] at ({1*\W+0.5*\W},{-3*\Hh-0.58}) {decisive: $\cs$};
\node[font=\small\bfseries] at ({2*\W+0.5*\W},{0.42}) {SEW-CMP};
\node[font=\scriptsize, gray!60!black] at ({2*\W+0.5*\W},{-3*\Hh-0.26}) {Thm.~\ref{thm:circuit}\,$+$\,\ref{thm:rounds}};
\node[font=\scriptsize, gray!60!black] at ({2*\W+0.5*\W},{-3*\Hh-0.58}) {decisive: $R$ rounds};
\node[font=\small, anchor=east] at (-0.18,{-0*\Hh-0.5*\Hh}) {$-\,\cv$};
\node[font=\small, anchor=east] at (-0.18,{-1*\Hh-0.5*\Hh}) {$-\,\cs$};
\node[font=\small, anchor=east] at (-0.18,{-2*\Hh-0.5*\Hh}) {$R{=}1$};
\fill[cverb!70, draw=gray!45] (0*\W,{-0*\Hh-\Hh}) rectangle ({0*\W+\W},{-0*\Hh-\Hh+\Hh});
\draw[crouter, very thick, rounded corners=1pt] ({0*\W+0.05},{-0*\Hh-\Hh+0.05}) rectangle ({0*\W+\W-0.05},{-0*\Hh-0.05});
\node[font=\small] at ({0*\W+0.5*\W},{-0*\Hh-0.5*\Hh}) {$-84.4$};
\fill[cverb!0, draw=gray!45] (1*\W,{-0*\Hh-\Hh}) rectangle ({1*\W+\W},{-0*\Hh-\Hh+\Hh});
\node[font=\small] at ({1*\W+0.5*\W},{-0*\Hh-0.5*\Hh}) {$+0.0$};
\fill[cverb!3, draw=gray!45] (2*\W,{-0*\Hh-\Hh}) rectangle ({2*\W+\W},{-0*\Hh-\Hh+\Hh});
\node[font=\small] at ({2*\W+0.5*\W},{-0*\Hh-0.5*\Hh}) {$+4.2$};
\fill[cverb!2, draw=gray!45] (0*\W,{-1*\Hh-\Hh}) rectangle ({0*\W+\W},{-1*\Hh-\Hh+\Hh});
\node[font=\small] at ({0*\W+0.5*\W},{-1*\Hh-0.5*\Hh}) {$-3.1$};
\fill[cverb!6, draw=gray!45] (1*\W,{-1*\Hh-\Hh}) rectangle ({1*\W+\W},{-1*\Hh-\Hh+\Hh});
\draw[cverb, very thick, rounded corners=1pt] ({1*\W+0.05},{-1*\Hh-\Hh+0.05}) rectangle ({1*\W+\W-0.05},{-1*\Hh-0.05});
\node[font=\small] at ({1*\W+0.5*\W},{-1*\Hh-0.5*\Hh}) {$+8.3$};
\fill[cverb!51, draw=gray!45] (2*\W,{-1*\Hh-\Hh}) rectangle ({2*\W+\W},{-1*\Hh-\Hh+\Hh});
\node[font=\small] at ({2*\W+0.5*\W},{-1*\Hh-0.5*\Hh}) {$+62.5$};
\fill[cverb!36, draw=gray!45] (0*\W,{-2*\Hh-\Hh}) rectangle ({0*\W+\W},{-2*\Hh-\Hh+\Hh});
\node[font=\small] at ({0*\W+0.5*\W},{-2*\Hh-0.5*\Hh}) {$-43.8$};
\fill[cverb!0, draw=gray!45] (1*\W,{-2*\Hh-\Hh}) rectangle ({1*\W+\W},{-2*\Hh-\Hh+\Hh});
\node[font=\small] at ({1*\W+0.5*\W},{-2*\Hh-0.5*\Hh}) {$+0.0$};
\fill[cverb!0, draw=gray!45] (2*\W,{-2*\Hh-\Hh}) rectangle ({2*\W+\W},{-2*\Hh-\Hh+\Hh});
\draw[crouter, very thick, rounded corners=1pt] ({2*\W+0.05},{-2*\Hh-\Hh+0.05}) rectangle ({2*\W+\W-0.05},{-2*\Hh-0.05});
\node[font=\small] at ({2*\W+0.5*\W},{-2*\Hh-0.5*\Hh}) {$+0.0$};
\node[font=\scriptsize, gray!55!black, anchor=east] at (-0.18,0.42) {$\Delta$ score (pts)};
\node[font=\scriptsize, anchor=north west, align=left, text width=7.9cm] at (0,{-3*\Hh-0.92}) {\textcolor{cverb}{\rule{2.6mm}{1.1pt}}\ \ the registered signature is \textbf{CROSSED} in at least one column: the boxed cell is where the pre-registration put the largest drop, and it is not where the largest drop is. Criterion F3 fires; \S\ref{sec:falsify} reports it as such.};
\end{tikzpicture}
\caption{\textbf{Measured collapse pattern.} Each cell is the change in score when one channel or resource is removed from \ASP{}, against full \ASP{} on the flagship (SEW-RET/TRK/CMP $=93.8/4.2/20.8$). The theory predicts a \emph{shape}, not merely that ablations hurt: removing $\cv$ should cost most on the subtest witnessing the Shannon and horizon walls while costing little on tracking, and symmetrically for $\cs$ and for rounds. Boxes mark where the pre-registration placed each column's largest drop; they are not moved to follow the data, which is what makes criterion F3 (\S\ref{sec:falsify}) able to fire.}
\label{fig:collapse}
\end{aspfigure}
Three registered ablation signatures mirror the theory, and Figure~\ref{fig:collapse} shows them as the pattern that F3 tests: (i) deleting $\cv$ collapses EW-RET toward the Theorem~\ref{thm:shannon} floor while leaving EW-TRK intact; (ii) deleting $\cs$ collapses EW-TRK toward the sampling bound of Theorem~\ref{thm:horizon} applied to reductions; (iii) forcing $R{=}1$ collapses EW-CMP toward single-pass depth limits (Theorem~\ref{thm:circuit}) while barely moving RET/TRK. Any \emph{crossed} collapse (e.g.\ $-\cv$ hurting TRK more than RET) counts as evidence against the two-channel account.

\subsection{Reallocation versus growth}
Prediction~\ref{prd:budget} is the practical claim: given a fixed engineering choice between restructuring access and buying context, restructuring should win. Table~\ref{tab:p3} tests it directly by giving the baseline four times the budget.

\begin{asptable}[!t]
\centering\small
\caption{\textbf{P3: reallocation versus growth.} \ASP{} at $B{=}4{,}096$ against \textsc{Uniform-}$w$ given \emph{four times} the budget, $B{=}16{,}384$. A positive $\Delta$ means restructuring access beats buying context.}
\label{tab:p3}
\begin{tabular}{lccc@{\hskip 10pt}c}
\toprule
Backbone & \ASP{} @ 4k & \textsc{Uniform-}$w$ @ 16k & $w$ @ 16k & $\Delta$ \\
\midrule
Ministral~3B$^{\ddagger}$ & 46.9 & 24.0 & 8 & 22.9 \\
Qwen3-VL-8B & 43.8 & 23.2 & 53 & \textbf{20.6} \\
Gemma~3~12B$^{\ddagger}$ & 40.3 & 18.1 & 30 & 22.2 \\
Ministral~14B$^{\ddagger}$ & 48.3 & 18.8 & 8 & 29.5 \\
Qwen3.8-27B & 39.6 & 29.5 & 53 & \textbf{10.1} \\
Qwen3-VL-30B-A3B & 39.9 & 17.7 & 53 & \textbf{22.2} \\
Gemma~4~31B & 64.9 & 21.2 & 60 & \textbf{43.8} \\
\bottomrule
\multicolumn{5}{@{}p{0.86\linewidth}@{}}{\footnotesize $^{\ddagger}$Provider caps this backbone at $8$ images per request, so at $B{=}16$k the \textsc{Uniform-}$w$ arm is limited by the endpoint and not by the budget: the growth arm was never allowed to grow, and the row is shown but does not count as support for P3. Affected: Ministral~3B, Gemma~3~12B, Ministral~14B.}\\
\end{tabular}
\end{asptable}

\subsection{A corpus defect found after measurement}
\label{sec:cmpdefect}
Auditing the gold distribution of every answer alphabet after the grid was complete turned up
one degenerate slice. The generator's \texttt{take()} returns \emph{sorted} frame indices and the
before/after template always named the earlier singleton first, so all eight of those CMP items
carried the gold \textsf{before} and a constant answer scored $8/8$: \textsc{Blind}, which sees
no frames, reached $87.5$--$100$ on them. A third of the CMP subtest was answerable without
access to the stream.

We report this in the past tense because it is repaired. The generator now constructs an exact
$4$--$4$ balance rather than shuffling, since eight independent draws landed on six--two at this
seed, and the draw is taken from an item-local stream so that regenerating leaves all $1{,}200$
frames and \texttt{meta.json} byte-identical and changes four of the eighty questions
(\texttt{annotation/selftest\_regen.py} asserts this, and that no RET or TRK item moves, since
those $138$ cells were not re-run). The answerability gate gained a gold-distribution check: it
refuses any alphabet whose most common gold beats a Bonferroni-corrected binomial tail at that
$K$, the statistic being a maximum over $k$ categories rather than a per-category tail, so that
six of one colour at $K{=}8$ is not flagged as noise would be. The $69$ affected cells were
re-run against the corrected corpus and Table~\ref{tab:cmpfix} gives both runs; every CMP figure
elsewhere in this paper is post-fix, and the superseded cells and item set are released in
\texttt{results/cells\_pre\_cmpfix/} and \texttt{data/sew\_pre\_cmpfix/}.

Three things about what the defect was worth. \emph{The repair is verified by the instrument
that found it}: before it, four of \VnBackbones{} backbones put \textsc{Blind} clearly above the
$25.0\%$ chance line ($36.5$, $34.4$, $37.5$, $40.6$); after it, none does. \emph{The chance
level is unchanged at $25.0\%$}, since CMP mixes sixteen $K{=}8$ items with eight binary ones;
balancing removes an exploitable prior, not the floor. \emph{The falsification verdict does not
move}: F3 is an \textsc{argmax} over one ablation's three drops, not a comparison of levels, and
on the corrected corpus the largest-drop column is unchanged in all three rows. Two of the three
crossings are decided by the other two subtests, so CMP was never load-bearing for this
criterion. What the re-run did change runs against us: \textsc{Retr-only} improves on the
flagship CMP from $37.5$ to $71.9$, and \ASP{} now beats it on no backbone even before
correction.

\subsection{Falsification criteria (frozen)}
\label{sec:falsify}
The criteria below were frozen with the protocol, before measurement. They are reproduced verbatim, and the verdict after each is computed from the released per-question records by \texttt{analysis/analyze.py} rather than typed in, so this subsection cannot drift from the data it reports on, in either direction.

\begin{itemize}[leftmargin=1.6em,itemsep=2pt]
\item \textbf{(F1)} \emph{\textsc{Uniform-}$w$ comes within 10 points of \ASP{} on EW-RET for any $\ge$27B backbone.} Evaluated on SEW-RET (\S\ref{sec:setup}) for the three $\ge$27B backbones: F1 \VFone{}.
\item \textbf{(F2)} \emph{$\rho_{\mathrm{scale}}$ under \ASP{} is not lower than under \textsc{Uniform-}$w$ by at least 0.2.} Measured over the \VnBackbones{}-rung ladder: $\rho_{\mathrm{scale}}$ falls from \VrhoUnif{} under \textsc{Uniform-}$w$ to \VrhoAsp{} under \ASP{}, a drop of \VrhoDrop{}. F2 \VFtwo{}.
\item \textbf{(F3)} \emph{The ablation collapse pattern is crossed}, i.e.\ some subtest's largest drop is not the ablation that removes that subtest's wall price. F3 \VFthree{}; Figure~\ref{fig:collapse} shows the full matrix with the pre-registered cells boxed, and the boxes are not moved to follow the data. One of the three subtests carried a degenerate slice; it has been regenerated and its cells re-run (\S\ref{sec:cmpdefect}), and this verdict is computed on the corrected corpus. It is also unchanged from the contaminated one (same crossings, same largest-drop column in every row) because the criterion is an \textsc{argmax} over one ablation's three drops rather than a comparison of levels.
\item \textbf{(F4)} \emph{\textsc{Uniform}@16k $\ge$ \ASP{}@4k on both OpenEQA and EW mean.} The OpenEQA half is \textbf{not evaluated}, the benchmark was not run, so this criterion is tested on one of its two clauses only, which makes it weaker than registered. On the SEW mean: F4 \VFfour{}.
\end{itemize}

\noindent Two things this subsection deliberately does not do. It does not convert an absolute registered threshold into a verdict across a corpus change: $\le30$/$\le38$/$\ge78$ were derived at EW-Bench's $\kappa_s\approx0.38$ and a $15.6\%$ horizon ceiling, and transporting them to a corpus with $\kappa_s{=}\VkappaS{}$ would test the corpus rather than the account (\S\ref{sec:predictions-map}). And it does not report a criterion as passed when a clause of it had no data. Partial failures are reported with direction, as CCH reported its own reversed prediction \citep{cch}.

\subsection{Analysis: reading the measured numbers}
\paragraph{Scale dependence, and why F2 could not test it.} Spearman
$\rho_{\mathrm{scale}}$ over the seven-rung ladder is $\VrhoUnif{}$ under \textsc{Uniform-}$w$
and $\VrhoAsp{}$ under \ASP{}, so F2 fires and we report it as fired. The criterion presumes a
\emph{positive} baseline correlation for \ASP{} to flatten, and \textsc{Uniform-}$w$'s is
negative: at fixed $B$ a higher tokens-per-frame cost buys fewer frames, so with Ministral~3B
at $432$ tokens/frame against Gemma~3~12B's $259$ the ladder is partly a ladder of image-token
prices. Against a negative baseline the registered ``drop of at least $0.2$'' does not measure
convergence. F2 therefore stands as fired and is separately recorded as uninformative;
Corollary~\ref{cor:conv} needs a ladder at matched tokens-per-frame, which the registered
backbone set would have supplied and the substituted one does not (\S\ref{sec:limits}(3)).

\label{sec:analysis}

\begin{asptable}[!t]
\centering\small\setlength{\tabcolsep}{4pt}
\caption{\textbf{Every bound, evaluated at the parameters we actually ran}
($N{=}\VN$, $d{=}\Vd$, $Nd{=}\VNd$, $K{=}\VK$, $m{=}\VmBits$ bits, $k{=}8$, $w{=}12$,
$k_{\mathrm{tot}}{\in}[9,15]$, $d_q^{\max}{=}3$). ``Binds'' means the bound lies strictly
between chance and $1$, so that it forbids something a method could otherwise have done.
Two rows do not: Theorem~\ref{thm:rounds}\ref{it:rounds-loc}--\ref{it:rounds-grounded} are
true but vacuous here, because $m{=}\VmBits$ bits against $N{=}\VN$ frames is a large state
per pointer, and Theorem~\ref{thm:circuit} is asymptotic in $d_q$ and yields no number at
$d_q\le3$. What carries the empirical argument on this corpus is therefore
Theorems~\ref{thm:shannon} and~\ref{thm:horizon}, Lemma~\ref{lem:vtrack}, and
Theorem~\ref{thm:rounds}\ref{it:rounds-stateless}, the four unconditional,
non-vacuous rows. Every number in this table is recomputed from the corpus metadata by
\texttt{analysis/check\_bounds.py}, which exits non-zero if the paper and the corpus disagree;
the two errors it was written after are recorded in \texttt{notes/deviations.md}.}
\label{tab:scope}
\begin{tabular}{@{}>{\raggedright\arraybackslash}p{2.0cm}>{\raggedright\arraybackslash}p{4.1cm}>{\raggedright\arraybackslash}p{3.0cm}>{\raggedright\arraybackslash}p{1.9cm}p{0.9cm}>{\raggedright\arraybackslash}p{3.1cm}@{}}
\toprule
Result & Caps the capability of & Beyond Def.~\ref{def:stream}--\ref{def:agent} it assumes
& Value here & Binds & Instantiated by \\
\midrule
Thm.~\ref{thm:shannon} & any $m$-bit compressive state, on RET & nothing (unbounded compute)
& $\kappa_s=\VkappaSpct\%$ & yes & \textsc{State-only}, RET \\
Thm.~\ref{thm:horizon} & any query-\emph{independent} $w$-frame selection, on RET & nothing
& $16.0\%$ & yes & \textsc{Uniform-}$w$, RET \\
Lem.~\ref{lem:vtrack} & verbatim-only access reading $k_{\mathrm{tot}}{<}N$, on TRK
& read-only retrieval & $\VoneOverK\%$ & yes & \textsc{Retr-only}, TRK \\
Thm.~\ref{thm:rounds}\ref{it:rounds-stateless} & one non-adaptive round, no cross-frame state, on CMP
& read-only retrieval & $14.8\%$ & yes & \textsc{Retr-only}, \textsc{Socratic}, CMP \\
\addlinespace[2pt]
Thm.~\ref{thm:rounds}\ref{it:rounds-loc} & \emph{localisation} by one round with an $m$-bit state
& read-only retrieval & $411\%$ & \textbf{no} & \ASP{} $R{=}1$ ablation \\
Thm.~\ref{thm:rounds}\ref{it:rounds-grounded} & one round, $m$-bit state, read-grounded answer
& read-groundedness & $423\%$ & \textbf{no} & --- \\
Thm.~\ref{thm:circuit} & a single forward pass at depth $d_q{=}\omega(1)$
& $\mathsf{TC}^0\!\neq\!\mathsf{NC}^1$; no shortcut & asymptotic & \textbf{no} & --- \\
\midrule
Prop.~\ref{prop:complete} & (lower bounds for \ASP{}) & Ass.~\ref{ass:repr}--\ref{ass:sep}
& $\epsilon_u<1/N$ & --- & all cells \\
Prop.~\ref{prop:super} & (super-additivity) & $\Delta<0.344$ & condition & --- &
no $T_{\mathrm{conj}}$ subset \\
Cor.~\ref{cor:dominance} & (dominance on the mixture) & $\kappa_s\le0.5625$;\ $\Delta_T<0.875$
& $\VkappaS\le0.5625$ & yes & Table~\ref{tab:main} \\
Cor.~\ref{cor:conv} & (convergence across scale) & shared $\delta_r$, $\epsilon_u\le\bar\epsilon$ & $c=304$ & --- &
Fig.~\ref{fig:conv} \\
\bottomrule
\end{tabular}
\end{asptable}

\paragraph{Which bounds bind at these parameters, and which do not.} A bound with a free
parameter is not a claim until the parameter is pinned. Table~\ref{tab:scope} pins all seven and
Figure~\ref{fig:regimes} plots the two closed forms; every number in that table is recomputed
from the corpus metadata by \texttt{analysis/check\_bounds.py}, which exits non-zero if the
paper and the corpus disagree.

SEW-Bench has $N{=}\VN{}$ frames per episode carrying $d{=}\Vd{}$ queryable attribute values
each, so $Nd{=}\VNd{}$ over a $K{=}\VK{}$ alphabet, and an $L_s{=}512$-token state holds
$m\approx\VmBits{}$ bits at $\beta\approx12$ bits/token. Theorem~\ref{thm:shannon} therefore
caps a compressive-only agent on episodic retrieval at
\[
\kappa_s=\frac{m/(Nd)+1}{\log_2 K}=\frac{\VmBits{}/\VNd{}+1}{3}=\VkappaS{},
\]
i.e.\ \textsc{State-only} $\le\VkappaSpct{}\%$ on SEW-RET; measured,
\VStateRETMin{}--\VStateRETMax{}\% across the \VnBackbones{} backbones, and \VthmOneViol{}. The
denominator is what makes this a statement: substituting the frame count $N$ for the attribute
count $Nd$ gives $\kappa_s\approx7.5$, and the earlier version of this corpus ($d{=}9$,
$Nd{=}1{,}350$) gives $1.85$, above $1$ and so no constraint at all. A \textsc{State-only}
failure under a vacuous ceiling would not be evidence for the theorem.

Theorem~\ref{thm:horizon} caps \textsc{Uniform-}$w$ at $\frac{w}{N}+(1-\frac{w}{N})\frac1K$, and
because every RET target is visible in exactly one frame the bound applies at multiplicity one
exactly, with no redundancy term to estimate: at the realised $w\in[\VwMin{},\VwMax{}]$ the
ceiling is $\VThmTwoMin{}$--$\VThmTwoMax{}\%$ against a measured \VUnifRETMin{}--\VUnifRETMax{}\%.
Lemma~\ref{lem:vtrack} pins a verbatim-only agent at $1/K=\VoneOverK{}\%$ on tracking whenever
its read budget is below $N$ frames, which at $B{=}4$k it is by a factor of $20$ to $33$;
measured, \textsc{Retr-only} on SEW-TRK is \VRetrTRKMin{}--\VRetrTRKMax{}\% against
\textsc{State-only}'s \VStateTRKMin{}--\VStateTRKMax{}\%. The two channels fail on
\emph{different} subtests, which is the claim, and Corollary~\ref{cor:dominance}'s consequence
holds: \ASP{} beats the better single channel by \VCorGapMin{} to \VCorGapMax{} points on the
mixture.

Two of the seven statements forbid nothing here. Theorem~\ref{thm:rounds}\ref{it:rounds-stateless}
binds, capping one non-adaptive stateless round at $14.8\%$ on CMP at $k{=}8$, which is the
configuration \textsc{Retr-only} and \textsc{Socratic} occupy; its stateful parts do not, since
at $m{=}\VmBits{}$ bits against $N{=}\VN{}$ frames the localisation bound of
\ref{it:rounds-loc} evaluates to $411\%$, so one round could in principle find the answer frame
and the $R{=}1$ collapse on CMP is empirical rather than information-theoretic.
Theorem~\ref{thm:circuit} is asymptotic in $d_q$ and our deepest chain is $d_q{=}3$, so it
yields no number; testing it needs the registered $d_q\in\{3,4,5\}$ design
(\S\ref{app:ewdesign}). Proposition~\ref{prop:super} has no witness, since the corpus ships no
conjunctive subset, so super-additivity is \textbf{not evaluated}.

\paragraph{What the RET column does and does not add.} The RET result is the paper's largest
effect and it is also, on inspection, fully accounted for by two quantities we measure
separately. Writing $\hat p=(1-\epsilon_0)\times\text{recall}@8$ from Table~\ref{tab:probe} and
the index measurement above, the predicted and measured \ASP{} RET accuracies are
$0.94/0.94$, $0.91/0.91$, $1.00/0.94$, $0.75/0.91$, $0.69/0.75$, $0.88/0.78$ and $1.00/0.91$
across the ladder, residuals in $[-0.10,+0.16]$, mean absolute residual $0.07$. This is
exactly what Proposition~\ref{prop:complete} predicts ($\Cap_B\ge1-(\epsilon_0+\delta_r)$), and
it is a genuine confirmation of the \emph{mechanism}: retrieval finds the frame, the backbone
reads it, and nothing else is doing work. It is also a reason not to over-read the $84$-point
gap. A RET target is a class of multiplicity exactly one, so its class name is a unique key
across all $\VN{}$ captions and BM25 returns it at median rank $1$; and the same multiplicity-one
construction makes the horizon ceiling maximally tight. The gap is therefore a consequence of
the corpus being built to instantiate Theorem~\ref{thm:horizon} one-to-one, which is what a
witness corpus is for, rather than a discovery about natural streams. On EW-Bench, where targets
have multiplicity $1$--$2$ and retrieval must work from pixels rather than from a unique text
key, we would expect the gap to shrink; how much is the single most informative number the next
version could report.

\paragraph{Is the backbone competent per access?} Assumption~\ref{ass:repr} is the premise
everything positive in \S\ref{sec:theory} rests on, and a grid alone cannot check it, since a
retrieval hit followed by a misread scores exactly like a retrieval miss.
Table~\ref{tab:probe} removes access from the problem and asks the RET question with the target
keyframe supplied. Attribute reading is comfortably sufficient, $0.71$ to $1.00$ against a
$0.125$ floor, so a near-chance cell in Table~\ref{tab:main} is not a backbone that cannot see.
This is what makes the rest of this section attributable to access at all.

The one-step column is weaker: four of seven backbones sit at or below the $50\%$ chance line,
so for those the CMP family is limited by composition as well as by access and
Proposition~\ref{prop:complete}'s $1-d_q\epsilon_1$ bound is vacuous. Two consequences. A CMP
comparison on those backbones is between methods all operating below the premise the theory
needs, and should be read as such. And Corollary~\ref{cor:conv} converges the ladder only once
$\bar\epsilon$ is small; on this evidence $\epsilon_1$ is not small below about $27$B, which
locates the prediction more sharply than the registered version did.

\paragraph{Does the index actually work?} Proposition~\ref{prop:complete}'s RET bound is
$1-(\epsilon_0+\delta_r)$, so a retrieval failure and an access-structure failure look identical
in the score. We measure $\delta_r$ separately and locally, at no API cost, by asking how often
the episodic index returns the target frame in its top $k$
(\texttt{analysis/\allowbreak retrieval\_\allowbreak recall.py}): recall@$8$ is
\VrecallMin{}--\VrecallMax{}, so $\delta_r\in[\VdeltaRMin{},\VdeltaRMax{}]$. The check earned
its place twice. An early caption schema gave recall@$8$ of $0.06$--$0.12$ because a one-line
caption cannot inventory sixteen objects, and a later index bug held it at $0.00$--$0.22$ with a
median target rank of $121$ of $300$; under either, \ASP{} and \textsc{Retr-only} would have
been near chance on RET and the scissors prediction would have failed for reasons having nothing
to do with access structure. Both are logged as deviations, with the lesson that a recall probe
must exercise the retriever's own tokeniser or it measures a different index.

The $\alpha$ sweep adds one reading that bounds what this corpus can show. At $\alpha{=}1$,
vision-only retrieval, recall@$8$ is $0.005$, indistinguishable from random; at $\alpha{=}0$,
captions only, it is $0.936$. CLIP image--text similarity carries essentially nothing on a
schematic render, so the index's competence here is entirely textual. That is a property of
SEW-Bench rather than of the method, and it is why the natural-video corpus of
\S\ref{app:ewdesign} is the next thing to build: on real frames the visual half of the index is
the half that would have to work.

\paragraph{Where the budget goes.} Table~\ref{tab:alloc}'s allocation is a claim about behaviour, not a constant in the code, so the realised per-channel ledger is recorded on every decision and reported here. Averaged over backbones, \textsc{ret} decisions spend \VsplitRETBr{}\% of the debited tokens on retrieved pixels, \textsc{trk} decisions \VsplitTRKBs{}\% on the compressive state, and \textsc{cmp} decisions \VsplitCMPBc{}\% on the reasoning rounds themselves, the empirical shadow prices of the three walls. Where the realised split departs from the registered profile, the surrogate $\widehat{\Cap}_z$ mis-specifies a wall price and Table~\ref{tab:alloc} should be re-fit; that is a refinement, not a refutation, and \S\ref{sec:falsify} does not treat it as one. Two mechanical caveats on reading these numbers: a state read is charged at the tokens it inserts into the context rather than at a separate API call, and the router's own classification call is charged to $B_\rho$, which is why the three figures do not sum to $100$.

\paragraph{Qualitative traces.} Every decision's full access log is released: the router class, the allocated and realised $(B_s,B_r,B_c)$, each round's action, the frames retrieved, and the tokens debited per call. Appendix~\ref{app:qual} fixes the format and gives one worked item; the released \texttt{results/cells/} records carry the same fields for all \VnTotal{} items on every backbone and method, which is what makes the claims in this section checkable at the level of individual decisions rather than only at the level of cell means.

\paragraph{How to read P3, given that three rows could not grow.} P3 asks whether restructuring
a $4$k budget beats \emph{buying context} with a $16$k one, and on three of the seven backbones
the growth arm was never allowed to grow: the provider caps images per request at $8$ on the two
Mistral endpoints and at $30$ on Gemma~3~12B, so \textsc{Uniform-}$w$ at $B{=}16$k realised
$w{=}8$, $w{=}8$ and $w{=}30$ instead of the $53$--$62$ frames the budget affords. Those rows
measure a serving limit, not the question, and they are marked $\ddagger$ in
Table~\ref{tab:p3} and excluded from the F4 verdict. The result survives the restriction
cleanly: on the four backbones whose $w$ was budget-limited, \ASP{}@$4$k beats
\textsc{Uniform-}$w$@$16$k by $11.8$, $22.0$, $22.2$ and $45.1$ points, every one significant at
$\alpha{=}0.01$ on the paired per-question test. So P3 rests on four rows rather than seven, and
it does not depend on the capped ones, which matters, because a reader could otherwise
reasonably suspect that quadrupling the budget looked useless only because we did not let it be
spent.

\paragraph{Multiple comparisons, and what survives them.} We report $\VnTests{}$ paired tests, \VnBackbones{} backbones against six baselines, and the registered $\alpha$ is $0.01$ per
test. Quoting each against an unadjusted $\alpha$ would put the family-wise error rate near
$1-0.99^{\VnTests{}}\approx34\%$, so the family is corrected with Holm--Bonferroni (uniformly
more powerful than plain Bonferroni and assuming nothing about independence). Of
$\VsigUncorr{}$ tests significant uncorrected, $\VsigHolm{}$ survive. The correction is
therefore almost free for the comparisons against query-independent baselines, and it is not
free in general, four of the $\VsigUncorr{}$ do not survive it. For the one comparison the
paper's thesis needs, correction turns out not to be the operative issue at all: of the
\VnBackbones{} \ASP{}-versus-\textsc{Retr-only} tests, $\VretrRawSurv{}$ are significant even
\emph{before} correction and \textbf{$\VretrHolmSurv{}$ after}. Every claim in this section
about beating a query-independent baseline stands after correction; the claim about beating the
verbatim-only baseline was never available to correct.

\paragraph{What the ablations cost the channel-duality claim.} The baseline that tests channel
duality directly is \textsc{Retr-only}, which is $\cv$ with $\cs$ removed. \ASP{} beats it on
\VretrRawSurv{} of \VnBackbones{} backbones at $\alpha{=}0.01$: the largest positive margin is
$+11.2$ (Gemma~4~31B), and $\VretrNeg{}$ of the seven paired deltas are negative, the flagship's
at $-16.2$. Against every other baseline \ASP{} wins by $17$ to $66$ points, so the result is
specific rather than a wash. Budgeted query-conditioned access beats every query-independent
alternative by a large margin; adding our compressive channel to it buys nothing further.

Table~\ref{tab:ablation} sharpens this. On the flagship, removing $\cs$ raises the mean from
$35.4$ to $58.0$ and CMP from $\VAspCmpAll{}$ to $\VNoCsCmpAll{}$, and removing $\rho$ raises
the mean to $55.2$; the two variants differ from \ASP{} mainly in how much state text reaches
the model. The traces show the mechanism: when the state's cumulative fields are wrong, they
enter the composition rounds as confident premises. Two of the three registered collapse cells
are therefore crossed. The one that holds, $-\cv$ on RET, holds decisively, $93.8$ to $9.4$.

Proposition~\ref{prop:complete} is not contradicted. It is conditional on
Assumption~\ref{ass:sep}(i), that $L_s$ tokens suffice to maintain $\phi$-sufficient
statistics, and Table~\ref{tab:probe} with the TRK column is evidence that this premise fails
for a prompted accumulator at $N{=}\VN{}$: the union bound's $N\epsilon_u$ term is tight for an
additive reduction, so the premise requires a per-update error below $1/\VN{}$ that nothing at
this scale delivers. The bound stands; our instantiation of $\cs$ does not.

\paragraph{Head-to-head with EGAgent.} EGAgent \citep{egagent} is the strongest published
agentic long-video system and shares our training-free, tool-using design, so the informative
comparison is the decomposition rather than the aggregate. Its entity scene graph is a
compressive \emph{text} channel and its visual-search tool partially buys back verbatim access,
so the registered prediction was a shape: a small \ASP{} margin on RET, where visual search does
the job of $\cv$, and a larger one on TRK, where counting and toggle parity need running
reductions that entity--relation edges do not natively maintain. Measured, the margin is
\VegaRETMin{} to \VegaRETMax{} points on RET and \VegaTRKMin{} to \VegaTRKMax{} on TRK.
\textbf{The pattern is contradicted in the sharpest available way: the TRK margin is negative
and large}, with \textsc{EGAgent} ahead on every backbone by $30$ to $70$ points. We committed
in advance to concluding, in that case, that graph structure rather than channel duality carries
the capability on this family, and we do.

The mechanism is specific. On SEW-Bench a tracked class appears in a caption once per sighting
and never as filler, so the graph's per-entity mention count equals its sighting count by
construction: \textsc{EGAgent} reads a count off an offline text index while $\cs$ asks a frozen
backbone to maintain a running counter across $\VN{}$ prompted updates, and the $N\epsilon_u$
term makes the second strictly worse. This does not license the conclusion that a compressive
channel is unnecessary, since Lemma~\ref{lem:vtrack} is unconditional and the graph is itself a
compressive channel, better implemented. Nor does it settle the comparison: mention count equals
event count only on this corpus, and a reduction over partial observation such as containment or
parity with occlusion breaks the equality, which the registered EW-Bench items were designed to
do. The parity checklist for the reimplementation, including the respects in which it
understates the published system, is in \texttt{docs/EGAGENT\_REIMPL.md}.

\paragraph{Threats to validity.} (i) \emph{Judge circularity}: the SEW subtests are exact-match, so the judge is a fallback rather than the metric, and the deferral rate is reported per cell (\S\ref{sec:setup}); the judge is disjoint from every backbone, though it shares a model family with two of them, which is the closest to disjoint the open sub-31B class allows. (ii) \emph{Gate coupling}: $\tau_g{=}0$ on SEW-Bench, identical for every method, so no \ASP{} advantage here can come from gate tuning, but the flip side is that the gate is untested on this corpus, and the registered $\tau_g$ sweep (S4) needs a corpus with a real sensor-rate stream. (iii) \emph{Prompt asymmetry}: baselines use the strongest published prompt patterns for their class (multi-frame, Socratic, retrieval), all reproduced in Appendix~\ref{app:traces}; prompt-length differences are absorbed into the budget by construction. (iv) \emph{Leakage}: SEW-Bench is generated with a fixed seed at run time and appears in no pretraining corpus, which removes the leakage question the registered EW-Bench would have raised, and replaces it with the larger one that the corpus is synthetic (see (1) below). (v) \emph{Compressive-channel fidelity}: the state's typed schema is a prompt, not a constraint, and the backbones populate it in different shapes; the $L_s$ cap is therefore enforced by truncating the state that is fed back in rather than by the schema, and a backbone that nests everything under one key gets less benefit from LRU eviction than one that does not. That variance is part of what Table~\ref{tab:main} measures, and it is a property of a training-free wrapper rather than a defect in the accounting.

\section{Limitations and Conclusion}
\label{sec:limits}

\paragraph{Limitations.} The first is the one that bounds every claim in the paper, so it is stated without hedging.

\emph{(0) The channel that works is textual, on a corpus where that is enough.} Theorem~\ref{thm:shannon} is stated over pixel attributes and the queried attributes here are pixel-only by construction, but the \emph{retrieval} that makes $\cv$ work is entirely caption-driven: vision-only retrieval at $\alpha{=}1$ scores recall@$8=0.005$, indistinguishable from picking frames at random, while caption-only retrieval at $\alpha{=}0$ scores $0.936$. CLIP image--text similarity carries essentially nothing on a schematic render. The mechanism demonstrated here is therefore a text index over model-written captions plus a budgeted read of the frames it returns. That is a faithful instance of a verbatim channel, and it is \emph{not} evidence that the visual half of such an index works; on natural frames that is the half which would have to.

\emph{(1) The measurements are on a synthetic corpus.} SEW-Bench is a 2-D schematic render, not egocentric video. It was built because the registered corpus could not be: EW-Bench needs habitat-sim and an HM3D/ScanNet licence, and OpenEQA, VSI-Bench and EgoSchema each index frames from a source dataset behind a signed agreement, HM3D and ScanNet, ScanNet/ScanNet++/ARKitScenes, and Ego4D respectively (\S\ref{sec:setup}). It establishes that at parameters where the bounds of \S\ref{sec:theory} are non-vacuous the predicted ordering of access structures holds, which is a claim about \emph{access structure under a token budget}. It establishes nothing about natural-scene perception, egocentric motion, occlusion, or the transfer of these margins to OpenEQA-style benchmarks; Table~\ref{tab:registered} is the outstanding bill.

\emph{(2) Scale of the measurement.} Eighty questions over four episodes and \VnBackbones{} backbones. Paired tests run at the registered $\alpha{=}0.01$ on $n{=}80$ per backbone, enough for the headline orderings and not for differences of a point or two: per-subtest $n$ is $32/24/24$, so a cell moves in steps of $3$--$4$ points.

\emph{(3) Two registered backbones do not exist} on the serving arm and were substituted (\S\ref{sec:setup}). The MoE substitute is not an omni model, so the registration's audio-modality arm is simply not tested.

\emph{(4) One serving arm.} The registration named controlled local vLLM as primary with OpenRouter as replication; only the API arm was run. Provider-side factors (routing to a particular vendor, quantisation we cannot inspect, silent revision changes) are therefore inside the measurement rather than controlled against. We pin slugs, disable reasoning, fix temperature $0$, and record the returned model fingerprint per call, which detects a revision change but does not prevent one.

\emph{(5) Passive streams only.} We study the EM-EQA regime; active exploration (A-EQA) couples access structure to control and is out of scope, though Theorem~\ref{thm:horizon} suggests exploration is itself a query-conditioned retrieval problem.

\emph{(6) One theorem stays conditional.} Theorem~\ref{thm:circuit} rests on $\mathsf{TC}^0\neq\mathsf{NC}^1$ and on the $S_5$ encoding capturing real spatio-temporal chains; naturalistic CMP items may admit shortcuts \citep{shortcuts}. Theorem~\ref{thm:rounds} is unconditional, and its stateless part is what the \textsc{Retr-only} and \textsc{Socratic} CMP cells run into; its stateful parts are arithmetically vacuous at our $m/N$ (Table~\ref{tab:scope}), so the $R{=}1$ collapse is empirical rather than a bound confirmed. We have no capability bound for states that are neither read-grounded nor answer tabulators, which is the one open piece of the theory here (\S\ref{app:proofs}).

\emph{(7) No hardware claim.} No physical robot or embedded accelerator is used: we make no on-device latency or power claims, and \emph{edge-scale} is defined throughout as a model-class-plus-budget property (\S\ref{sec:cost}).

\emph{(8) The compressive channel, as we built it, does not work.} A prompted JSON accumulator
asks a frozen backbone to carry a running reduction across $\VN{}$ updates, and for an additive
reduction the union bound's $N\epsilon_u$ term is tight, so Assumption~\ref{ass:sep}(i) requires
a per-update error below $1/\VN{}$ that nothing at this scale delivers. Measured, $\cs$ costs
more than it returns (\S\ref{sec:analysis}). None of the four bounds is weakened by this; what
is weakened is our claim to have instantiated them well. The design implication is item (i) of
the future work below.

\emph{(9) Hand-designed schema, and a failure mode that survives.} The scene-state schema is hand-designed; learned schemas may shift the $\cs$/$\cv$ frontier. Assumption~\ref{ass:sep} can fail on adversarially dense streams where wall prices exceed any practical $B$.

\paragraph{Future work, in the order it should be done.} The order is set by what each item
would settle, and the first two exist because of defects this run found in itself. The item that
stood first in an earlier version of this list, regenerate the corpus and re-run the CMP
cells, has been done, and \S\ref{sec:cmpdefect} reports it; what remains below is what is
still outstanding.

\emph{(i) Give the compressive channel a mechanism whose error does not compound in $N$.}
This is the substantive negative finding and it is actionable in three steps, in increasing
difficulty: compute reductions that are expressible as aggregates over the index instead of
prompting for them, which removes the $N\epsilon_u$ term outright for counts; add redundancy a
union bound cannot see through (periodic re-derivation from the index, or two independently
maintained counters reconciled at query time) where prompting is unavoidable; and gate $\cs$
into the composition rounds on a confidence signal rather than on the query class, since the
traces show wrong cumulative fields entering as confident premises.

\emph{(ii) Measure $\epsilon_u$, and sweep at least three corpus seeds.} $\epsilon_u$ is the
only margin in Assumption~\ref{ass:repr}--\ref{ass:sep} still unmeasured and it enters
multiplied by $N$, so it is the single most load-bearing unknown; a per-keyframe state probe is
the same class of instrument as the $\epsilon_0$ probe we did run. Separately, all four episodes
here come from one generator seed: temperature $0$ removes \emph{model} variance but not
\emph{corpus} variance, and every number in this paper is one draw. Three seeds with reported
dispersion is cheap and would put error bars on the whole grid.

\emph{(iii) Build the conjunctive subset, so Proposition~\ref{prop:super} has a witness.}
Super-additivity is the only novel positive theorem here and it is \textbf{not evaluated}: the
corpus ships RET, TRK and CMP and no $T_{\mathrm{conj}}$ items. Constructing them needs only
pairing an existing RET target with an existing TRK reduction and scoring both components
jointly, which makes the omission harder to excuse than the ones gated on licences.

\emph{(iv) Build EW-Bench and re-run this exact grid on it}: same code, same budgets, natural
video, $\kappa_s\approx0.38$ instead of $\VkappaS{}$. That step converts every claim here from a
mechanism check into an embodied result, and it is gated only on a licence and compute. It also
settles the question this corpus cannot: whether the verbatim channel's advantage survives when
retrieval must work from pixels, since on SEW-Bench vision-only retrieval is at chance
(recall@$8=0.005$), and the index is purely textual.

\emph{(v) Run the three registered natural-video benchmarks} of Table~\ref{tab:registered},
which is what would let Predictions~\ref{prd:conv} and~\ref{prd:budget} be evaluated on both of
their clauses rather than one. \emph{(vi) Add the controlled local-serving arm} the
registration named as primary, and report its divergence from the API arm. \emph{(vii) Sweep
$L_s$}: $\kappa_s$ is linear in the state's bit budget, so a state-size sweep turns
Theorem~\ref{thm:shannon} from a single ceiling into a \emph{curve} the data can be fitted
against. \emph{(viii) Assemble a ladder at matched tokens-per-frame}, without which
Corollary~\ref{cor:conv} cannot be tested: the ladder we could assemble from an API catalogue
spans $259$ to $432$ tokens/frame, so at fixed $B$ it is partly a ladder of image-token prices
rather than of capability, which is why F2 came out uninformative. \emph{(ix) Extend the
round-wall test past $R{=}1$} to the full pointer-chasing hierarchy, with $d_q{>}2$ items.

\paragraph{Conclusion.} Two of this paper's three parts hold up and one does not, and the split
is clean enough to state in a sentence each.

The theory holds. The four walls reappear in perceptual form over budgeted observation streams,
three of them unconditionally, and on a corpus built so that they are non-vacuous at its own
parameters they are respected and they bind: a compressive-only agent stays far below its
Shannon ceiling of $\VkappaSpct{}\%$, query-independent selection stays at or below its horizon
ceiling of $16.0\%$, and verbatim-only tracking stays at $1/K$. Two of the seven statements are
vacuous at these parameters and we say which (Table~\ref{tab:scope}); two errors in the proofs
were found and corrected during this work and we say which.

The empirical claim about \emph{access} holds, and strongly. Under a fixed per-decision token
budget, restructuring access beats every equal-budget alternative that does not condition on the
query (by $17$ to $66$ points, on every backbone, with every comparison surviving
Holm correction), and it beats an unstructured baseline given four times the budget on all
\VnBackbones{} backbones, including the four whose growth arm was actually allowed to grow. If
there is one practical reading to carry away, it is still this one: under a deadline, spend on
\emph{where you look}, not on \emph{how much you look at}.

The claim about our \emph{architecture} does not hold. \ASP{} is three components and one of
them does the work. After multiple-comparison correction no backbone shows \ASP{} beating the
verbatim channel alone; removing the compressive channel improves the flagship by $22.6$ points,
removing the router by $19.8$; a reimplemented EGAgent, which derives cumulative state from an
offline text index rather than a prompted accumulator, beats \ASP{} on tracking by $30$ to $70$
points. Two of four frozen falsification criteria fired. The failure is localised rather than
diffuse, and by a bound we wrote down before measuring: $N\epsilon_u$ is tight for additive
reductions. So the contribution here is a set of walls, a budget accounting, and a measured
answer to which part of an access structure earns its price, rather than a working dual-channel
architecture. Access structure is still what has to be bought.

\paragraph{Reproducibility.} Released: the protocol and all frozen prompts, the SEW-Bench generator (fixed seed, so the corpus regenerates bit-identically), the EW-Bench generator up to its rendering call, the budget accountant with its offline invariant tests, the per-question records behind every cell of Tables~\ref{tab:main}--\ref{tab:p3}, and this pre-registration. Measured results are committed as a diff against the registered tables so that prediction and outcome remain separately auditable, in the spirit of \citep{cch}; every deviation from the registration (including the corpus substitution, the two backbone substitutions, and the single serving arm) is logged with its reason rather than absorbed into the text. The artifact is at \url{https://github.com/wenhui-ml/access-structured-perception}, where \texttt{make verify} recomputes every checkable claim offline and exits non-zero on any disagreement with this paper. The registration is pinned by content: \texttt{results/registered\_predictions.json} has SHA-256 {\small \texttt{0af367149532d48a}\allowbreak \texttt{4a8c5678d6d3528e}\allowbreak \texttt{b80acc29599b67cf}\allowbreak \texttt{f8e0cf92b5a69fae}}. That hash fixes what was registered; it does not timestamp when, and the artifact says so rather than implying otherwise.

\ifdefined\ICRABuild
\else
  \bibliographystyle{unsrtnat}
  \bibliography{references}
  \clearpage
\fi
\appendix

\section{Full Proofs}
\label{app:proofs}

\begin{aspfigure}[!t]
\centering
\begin{tikzpicture}[
  font=\scriptsize,
  tool/.style   ={draw, rounded corners=2pt, fill=black!4, align=center,
                  text width=1.55cm, inner sep=3pt, minimum height=7mm},
  wall/.style   ={draw, thick, align=center, text width=2.05cm,
                  inner sep=3pt, minimum height=8.5mm},
  cond/.style   ={wall, dashed},
  dead/.style   ={wall, draw=black!35, text=black!45},
  comp/.style   ={draw, rounded corners=2pt, fill=blue!5, align=center,
                  text width=2.0cm, inner sep=3pt, minimum height=8mm},
  ar/.style     ={-{Stealth[length=2mm]}, gray!75, line width=0.5pt},
]
\node[tool] (fano)   at (0,  1.65) {Fano};
\node[tool] (lfano)  at (0,  0.40) {List Fano\\\scriptsize(Lem.~\ref{lem:listfano})};
\node[tool] (expo)   at (0, -0.85) {exposure\\argument};
\node[tool] (barr)   at (0, -2.30) {Barrington\\$+$ Lem.~\ref{lem:s5}};

\node[wall] (t1) at (2.95,  1.65) {\textbf{Thm.~\ref{thm:shannon}} Shannon\\$\kappa_s=\VkappaS$};
\node[wall] (t2) at (2.95,  0.55) {\textbf{Thm.~\ref{thm:horizon}} horizon\\$w/N+(1{-}w/N)/K$};
\node[wall] (t4) at (2.95, -0.60) {\textbf{Thm.~\ref{thm:rounds}\ref{it:rounds-stateless}} round\\no state, $R{=}1$};
\node[dead] (t4b) at (2.95, -1.70) {Thm.~\ref{thm:rounds}\ref{it:rounds-loc}--\ref{it:rounds-grounded}\\$m$-bit state};
\node[cond] (t3) at (2.95, -2.90) {\textbf{Thm.~\ref{thm:circuit}} composition\\$\mathsf{TC}^0\!\neq\!\mathsf{NC}^1$};
\node[wall] (l3) at (2.95,  2.75) {\textbf{Lem.~\ref{lem:vtrack}} no tracking\\without $\cs$};

\node[comp] (p1) at (6.15,  1.55) {\textbf{Prop.~\ref{prop:complete}}\\access-completeness};
\node[comp] (p2) at (6.15,  0.20) {\textbf{Prop.~\ref{prop:super}}\\super-additivity\\on $T_{\mathrm{conj}}$};
\node[comp] (c1) at (6.15, -1.15) {\textbf{Cor.~\ref{cor:dominance}}\\dominance on $T$};
\node[comp] (c2) at (6.15, -2.50) {\textbf{Cor.~\ref{cor:conv}}\\convergence};

\node[tool, fill=orange!8] (as) at (6.15,  2.85) {Ass.~\ref{ass:repr}--\ref{ass:sep}\\$\epsilon_0,\epsilon_1,\delta_r,\epsilon_u$};

\draw[ar] (fano)  -- (t1);
\draw[ar] (lfano) -- (t4);
\draw[ar] (lfano) -- (t4b);
\draw[ar] (expo)  -- (l3);
\draw[ar] (expo)  -- (t4);
\draw[ar] (barr)  -- (t3);
\draw[ar] (t1) -- (p1);
\draw[ar] (l3) -- (p1);
\draw[ar] (t4) -- (p1);
\draw[ar] (as) -- (p1);
\draw[ar] (p1) -- (p2);
\draw[ar] (p1) -- (c1);
\draw[ar] (p1) -- (c2);

\begin{scope}[shift={(0,-4.05)}]
  \draw[thick] (0,0) rectangle (0.42,0.24);
  \node[anchor=west, font=\scriptsize] at (0.52,0.12) {unconditional};
  \draw[thick,dashed] (2.55,0) rectangle (2.97,0.24);
  \node[anchor=west, font=\scriptsize] at (3.07,0.12) {conditional (complexity hypothesis)};
  \draw[draw=black!35] (0,-0.45) rectangle (0.42,-0.21);
  \node[anchor=west, font=\scriptsize, text=black!45] at (0.52,-0.33)
       {vacuous at the parameters we ran (Table~\ref{tab:scope})};
\end{scope}
\end{tikzpicture}
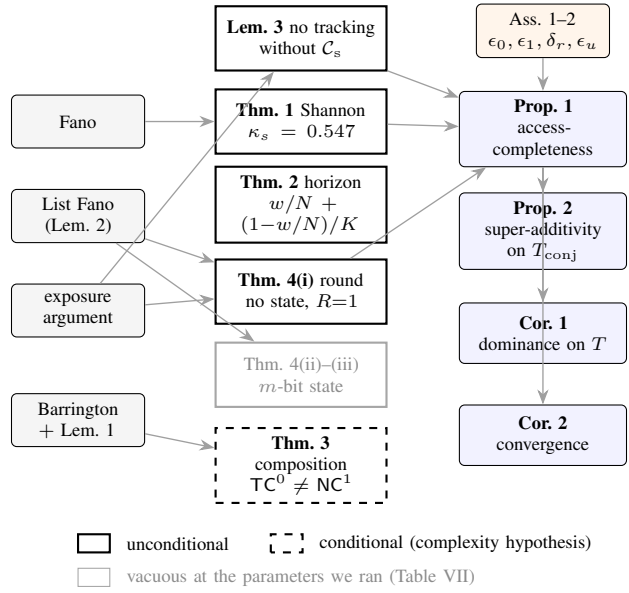
\caption{\textbf{What rests on what.} Tools (left) yield the walls (centre), which combine with
Assumptions~\ref{ass:repr}--\ref{ass:sep} into the composite claims (right). Two readings the prose
cannot deliver at a glance. \emph{One:} only Theorem~\ref{thm:circuit} inherits a complexity-theoretic
hypothesis, so a reader who disbelieves $\mathsf{TC}^0\neq\mathsf{NC}^1$ loses one box and keeps the
argument, Theorem~\ref{thm:rounds}\ref{it:rounds-stateless} carries the same architectural
conclusion unconditionally. \emph{Two:} every arrow into the right-hand column passes through the
assumption node, so the composite claims are exactly as strong as $\epsilon_0,\epsilon_1,\delta_r,\epsilon_u$,
and $\epsilon_u$ enters multiplied by $N$. The greyed box is a bound that is true but says nothing at
SEW-Bench parameters.}
\label{fig:prooftree}
\end{aspfigure}

Figure~\ref{fig:prooftree} shows what rests on what across this appendix, so a reader who rejects one hypothesis can see immediately which results survive it.

\subsection{Standing conventions}
\label{app:conv}
Throughout, $c=(c_1,\dots,c_N)$ with $c_i\in\{1,\dots,K\}^{d}$ and all $Nd$ attribute values i.i.d.\ uniform on $[K]$; $b=d\log_2K$ is the salient bits per keyframe. A budgeted agent's cross-frame state at query time is $S=\sigma(c)\in\{0,1\}^m$ for an arbitrary (possibly randomized, unbounded-compute) $\sigma$; randomization is handled by conditioning on the seed, which weakens none of the bounds below. Entropies are in bits and $h(\cdot)$ is the binary entropy.

We use one modelling convention for the verbatim channel, stated once and used in Lemmas~\ref{lem:listfano} and~\ref{lem:vtrack}. \emph{Read-only retrieval:} an access to $\cv$ reveals the contents of the returned items and nothing about the items not returned. This matches \ASP{}, where ranking is computed off-GPU and only the top-$k$ frames and \emph{their} captions ever enter the backbone's context; it would fail for an index that surfaced aggregate statistics of unretrieved frames, and we flag that as the assumption a future variant would have to re-examine.

\subsection{Proof of Theorem~\ref{thm:shannon} (perceptual Shannon wall)}
\label{app:shannon}
Index the $Nd$ attribute values by $u\in[Nd]$ and write $c_u$ for the $u$-th. The query index $U\sim\mathrm{Unif}[Nd]$ is independent of $c$, and the agent outputs $\hat c_U=\psi(S,U)$ with per-index error $\varepsilon_u=\Pr[\psi(S,u)\neq c_u]$ and average error $\bar\varepsilon=\frac1{Nd}\sum_u\varepsilon_u$. By Fano's inequality applied to each index,
\[
H(c_u\mid S)\;\le\;h(\varepsilon_u)+\varepsilon_u\log_2(K-1)\;\le\;1+\varepsilon_u\log_2 K .
\]
Summing over $u$ and using the chain rule together with the independence of the $c_u$,
\[
\begin{aligned}
\sum_u H(c_u\mid S)&\ge H(c\mid S)\ge H(c)-H(S)\\
&\ge Nd\log_2K-m .
\end{aligned}
\]
Combining the two displays, $Nd\log_2K-m\le Nd+\bar\varepsilon\,Nd\log_2K$, and dividing by $Nd\log_2K$,
\begin{equation*}
\bar\varepsilon\;\ge\;1-\frac{m/(Nd)+1}{\log_2 K}. \proofenddisp
\end{equation*}
The bound is information-theoretic: it holds for arbitrary $U,\pi$, arbitrary backbone scale, and arbitrary semantic cleverness of the compression, the CCH Shannon wall verbatim, with keyframe attributes replacing symbols. Two consequences are used later. First, the ceiling on retrieval \emph{accuracy} for a compressive-only agent is
\begin{equation}
\label{eq:kappa}
\kappa_s\;:=\;\frac{m/(Nd)+1}{\log_2 K},
\end{equation}
which is the quantity Proposition~\ref{prop:complete} and \S\ref{sec:analysis} both refer to. Second, the bound is vacuous unless $m<Nd(\log_2K-1)$: a state large enough to hold roughly one bit per queryable attribute is not obstructed at all, which is why the theorem is a statement about \emph{streams that outgrow the state}, and why $Nd$, not $N$, is the quantity that must be measured for any corpus the bound is applied to.

\subsection{Proof of Theorem~\ref{thm:horizon} (horizon wall)}
Condition on the selected set $W$, $|W|\le w$, which by hypothesis is a function of the stream alone and hence independent of the query index $U$. The $w$ selected keyframes carry $wd$ of the $Nd$ attribute values, and $U$ is uniform on those $Nd$ indices and independent of $W$, so $\Pr[U\in W]\le wd/(Nd)=w/N$, note that the per-frame attribute count $d$ cancels here, which is why the horizon wall is stated in frames while the Shannon wall is stated in attributes. If $U$ indexes an attribute of a keyframe in $W$ the agent may answer perfectly; otherwise $c_U$ is independent of everything in context, so any estimator succeeds with probability exactly $1/K$. Therefore
\begin{equation*}
\begin{aligned}
\Pr[\hat c_U=c_U]
&\le\Pr[U\in W]+\Pr[U\notin W]\tfrac1K\\
&\le\frac wN+\Bigl(1-\frac wN\Bigr)\frac1K.
\end{aligned}\proofenddisp
\end{equation*}
\emph{Remark.} Salience-, novelty-, and tree-search-based selection do not escape the bound, because all are query-independent; only query-conditioned retrieval breaks the independence of $W$ and $U$, and it pays the horizon price with $B_r$ tokens.

\subsection{Proof of Theorem~\ref{thm:circuit} (composition wall)}
We give the embedding as a lemma, then the two directions.

\begin{lemma}[$S_5$ embedding of composition chains]
\label{lem:s5}
There is a projection (hence $\mathsf{AC}^0$) map taking any word $g_1g_2\cdots g_n$ over a fixed two-element generating set of $S_5$ to an embodied stream $X(g)$ of $n$ keyframes and a single query $q^{\star}$, such that the correct answer to $(q^{\star},X(g))$ is \textsc{yes} iff $g_1g_2\cdots g_n=e$.
\end{lemma}
\begin{proof}
Fix five distinguishable landmark objects occupying five labelled slots, and fix the generating set $\{\tau,\gamma\}$ with $\tau=(1\,2)$ and $\gamma=(1\,2\,3\,4\,5)$; two elements generate $S_5$. Keyframe $i$ depicts exactly one scripted rearrangement event, chosen from the two-element set according to $g_i$, applied to the current arrangement; the depiction of each generator is a fixed image template, so keyframe $i$ depends only on the single symbol $g_i$ and the construction is a projection. The arrangement after keyframe $n$ is the image of the initial arrangement under $g_1\cdots g_n$ acting on slots. Let $q^{\star}$ be ``is every landmark in the slot it started in?'' Its correct answer is \textsc{yes} iff $g_1\cdots g_n$ fixes all five slots, i.e.\ iff $g_1\cdots g_n=e$. Each hop's operand is the visible arrangement, an attribute of the content vectors, as Definition~\ref{def:stream} requires of $T_{\mathrm{cmp}}$.
\end{proof}

\emph{Negative direction.} Suppose some fixed-depth, polynomial-width, constant-precision attention stack decided $T_{\mathrm{cmp}}$ in a single forward pass for chain length $d_q=\omega(1)$, with all keyframes in context. Compose it with the map of Lemma~\ref{lem:s5}: the result decides the word problem $\mathrm{WP}(S_5)=\{g:g_1\cdots g_n=e\}$ on words of length $n=d_q$. By Barrington's theorem $\mathrm{WP}(S_5)$ is complete for $\mathsf{NC}^1$ under $\mathsf{AC}^0$ (indeed projection) reductions \citep{barrington}, so $\mathrm{WP}(S_5)\in\mathsf{TC}^0$ would give $\mathsf{NC}^1\subseteq\mathsf{TC}^0$. But log-precision, fixed-depth, polynomial-width transformers are simulable in log-uniform $\mathsf{TC}^0$ \citep{merrill-sat}, and $\mathsf{AC}^0$ reductions do not leave $\mathsf{TC}^0$. Hence under $\mathsf{TC}^0\neq\mathsf{NC}^1$ no such single pass exists. \proofend

\emph{Positive direction.} Let the agent run $R$ sequential access rounds, each re-reading its own previous output. Partition the chain into $R=\lceil d_q/c\rceil$ consecutive chunks of constant length $c$. An element of $S_5$ is one of $120$ values, so the running partial product is representable in $\lceil\log_2 120\rceil=7$ bits and therefore in $O(1)$ tokens; this is the reason the composition wall is a \emph{round} resource and not a state resource, and the reason it is priced at $B_c^{\star}\!\cdot\!R$ rather than at $B_s^{\star}$. Round $r$ receives the partial product after chunk $r-1$ plus the $c$ keyframes of chunk $r$, and must perform one constant-size composition, which by Assumption~\ref{ass:repr} it does with error at most $\epsilon_1$. By a union bound over rounds the answer is correct with probability at least $1-R\epsilon_1\ge 1-d_q\epsilon_1$ (taking $c\ge1$). The composed computation has effective depth $\Theta(R\cdot d_{\mathrm{model}})$, which is the sense in which rounds multiply depth \citep{feng-cot,merrill-cot}. \proofend

\emph{Scope.} The negative direction is conditional twice over: on $\mathsf{TC}^0\neq\mathsf{NC}^1$, and on naturalistic CMP items actually realizing the $S_5$ encoding rather than admitting a shortcut \citep{shortcuts}. Theorem~\ref{thm:rounds} below is the unconditional statement that carries the same architectural conclusion, and it is the one our baselines actually run into.

\subsection{Proof of Theorem~\ref{thm:rounds} (round wall, unconditional)}
We first record the list-decoding form of Fano's inequality that the proof needs.

\begin{lemma}[List Fano]
\label{lem:listfano}
Let $P$ be a random variable on $[N]$, let $Z$ be arbitrary side information, and let $W=W(Z)\subseteq[N]$ with $|W|\le k$. Put $\delta=\Pr[P\notin W]$. Then
\[
H(P\mid Z)\;\le\;h(\delta)+(1-\delta)\log_2 k+\delta\log_2 N .
\]
\end{lemma}
\begin{proof}
Let $E=\mathbb{1}[P\notin W]$, a function of $(P,Z)$. Then
$H(P\mid Z)\le H(P,E\mid Z)=H(E\mid Z)+H(P\mid Z,E)$.
The first term is at most $h(\delta)$. For the second, conditioned on $E=0$ we have $P\in W(Z)$, a set of size at most $k$, so $H(P\mid Z,E=0)\le\log_2k$; conditioned on $E=1$, trivially $H(P\mid Z,E=1)\le\log_2N$. Averaging with weights $1-\delta$ and $\delta$ gives the claim.
\end{proof}

\emph{Witness and setup.} Keyframe $i$ carries a pointer $p_i\sim\mathrm{Unif}[N]$ and an attribute $a_i\sim\mathrm{Unif}[K]$, all $2N$ variables independent. The query names an index $J\sim\mathrm{Unif}[N]$, independent of everything, and asks for $a_{p_J}$, a depth-$2$ chain, since $p_J$ must be read before the frame carrying the answer can even be named. The agent performs a \emph{single non-adaptive} retrieval: it commits to $W$, $|W|\le k$, as a function of $(q,s_N)$, so $W$ may depend on $J$ and on the state, but not on any retrieved content. It then reads the frames in $W$ and answers.

\emph{Part \ref{it:rounds-stateless}: no cross-frame state.} Here $W=W(J)$ and $p_J$ is independent of $(J,W)$, so $\Pr[p_J\in W]=|W|/N\le k/N$. If $p_J\notin W$ then $a_{p_J}$ is uniform on $[K]$ and independent of every variable the agent has seen, the read-only convention of \S\ref{app:conv}, so it is answered correctly with probability exactly $1/K$. Hence
\[
\Pr[\hat y=y^{\star}]\;\le\;\frac kN+\Bigl(1-\frac kN\Bigr)\frac1K .
\]

\emph{Part \ref{it:rounds-loc}: localisation under an $m$-bit state.} Now $W=W(J,s_N)$ and $s_N$ may encode pointers. Apply Lemma~\ref{lem:listfano} with $P=p_J$ and $Z=(s_N,J)$, and write $\delta=\Pr[p_J\notin W]$:
\[
\begin{aligned}
H(p_J\mid s_N,J)&\le h(\delta)+(1-\delta)\log_2k+\delta\log_2N\\
&\le 1+(1-\delta)\log_2k+\delta\log_2N .
\end{aligned}
\]
For the left-hand side, $J$ is independent of the stream, so
\[
\begin{aligned}
H(p_J\mid s_N,J)&=\frac1N\sum_j H(p_j\mid s_N)\\
&\ge\frac1N H(p_1,\dots,p_N\mid s_N)\\
&\ge\frac1N\bigl(N\log_2N-m\bigr)=\log_2N-\frac mN,
\end{aligned}
\]
using independence of the $p_j$ and $H(s_N)\le m$. Chaining the two and rearranging,
\[
\begin{aligned}
\log_2N-\frac mN-1-\log_2k
&\le\delta\bigl(\log_2N-\log_2k\bigr)\\
&\Longrightarrow\quad
\delta\ge1-\frac{m/N+1}{\log_2(N/k)} .
\end{aligned}
\]
Therefore $\Pr[p_J\in W]\;=\;1-\delta\;\le\;\frac{m/N+1}{\log_2(N/k)}$, which is part \ref{it:rounds-loc}.

\emph{Part \ref{it:rounds-grounded}: capability under read-groundedness.} A read-grounded agent's output is a function of the contents it read, $\{(p_i,a_i)\}_{i\in W}$, together with $(q,W)$, and \emph{not} of $s_N$ except through the choice of $W$. Condition on $p_J\notin W$. Then $a_{p_J}$ was not returned, so by the read-only convention of \S\ref{app:conv} it is uniform on $[K]$ and independent of everything the answer may depend on, and the conditional success probability is exactly $1/K$. Combining with part \ref{it:rounds-loc},
\begin{equation*}
\begin{aligned}
\Pr[\hat y=y^{\star}]
&\le\Pr[p_J\in W]+\Pr[p_J\notin W]\tfrac1K\\
&\le\frac{m/N+1}{\log_2(N/k)}+\frac1K.
\end{aligned}\proofenddisp
\end{equation*}

\emph{Why read-groundedness cannot be dropped, and what an earlier draft got wrong.} A previous version of this theorem asserted the display of part~\ref{it:rounds-grounded} for \emph{every} $m$-bit state, with the $1/K$ obtained as ``if the answer frame was not retrieved, the answer is a uniform guess''. That step is invalid: the answer $a_{p_J}$ is a function of the stream, so an $m$-bit state may simply contain it. Concretely, let the state tabulate the pairs $(j,a_{p_j})$ for $j$ in a fixed set $S$ with $|S|=\lfloor m/\log_2 K\rfloor$; on a query in $S$ the agent answers from the table and retrieves nothing, so
\[
\Pr[\hat y=y^{\star}]\;\ge\;\frac{|S|}{N}+\Bigl(1-\frac{|S|}{N}\Bigr)\frac1K ,
\]
and this exceeds the claimed bound in the whole range where the claimed bound was non-vacuous: at $N{=}300$, $k{=}8$, $K{=}8$ the tabulator reaches $0.708$ at $m{=}600$ bits where the claim allows $0.699$, and $1.000$ at $m{=}900$ bits where the claim allows $0.890$. The claim is therefore false as stated and is corrected above rather than patched.

Three things survive the correction, and they are what the paper uses. \emph{First}, part~\ref{it:rounds-stateless} is unconditional and is the configuration our \textsc{Retr-only} and \textsc{Socratic} baselines actually occupy: they have no compressive channel at all, so $m=0$ and there is nothing to tabulate. \emph{Second}, part~\ref{it:rounds-loc} is unconditional and is a statement purely about \emph{access}: one query-conditioned round cannot find the frame the answer lives on, whatever the state contains, which is the mechanism the $R{=}1$ ablation of \S\ref{sec:exp} is built to expose. \emph{Third}, the tabulation escape is not available to a compressive channel by definition: covering a constant fraction of the $N$ queries costs $\Theta(N\log_2K)$ bits, whereas Definition~\ref{def:agent} fixes $|s_N|\le L_s\beta$ independent of $N$, at SEW-Bench parameters $m=6{,}144$ bits against $N\log_2K=900$ bits, so the tabulator is in fact affordable here, and the reason it does not defeat the measurement is that the state is built \emph{online} by a frozen backbone that cannot see which pointer will be queried, not that it is too small. We flag this as the gap it is: for states that are neither read-grounded nor tabulators we have no capability bound, and closing that, a round--state tradeoff for pointer chasing under an $o(N)$-bit advice string, is the one piece of the theory here that is genuinely open.

\emph{Positive direction.} With $R\ge d_q$ adaptive rounds the agent retrieves frame $J$ in round $1$ (a single item, $k=1$), reads $p_J$, retrieves frame $p_J$ in round $2$, and reads $a_{p_J}$; each round's read succeeds with probability $\ge1-\epsilon_0$ and each hop resolution with probability $\ge1-\epsilon_1$ by Assumption~\ref{ass:repr}, so a union bound over $d_q$ hops gives success $\ge1-d_q\epsilon_1$ with $O(1)$ retrievals per round. The separation between the two displays above and this bound is the round wall, and it involves no complexity-theoretic hypothesis. The classical round hierarchy for pointer chasing \citep{pointerchasing,nisanwigderson} extends the separation to all $R<d_q$; we prove only $R{=}1$ because that is the regime our \textsc{Retr-only} and \textsc{Socratic} baselines occupy.

\subsection{The verbatim channel cannot track (replacing a deferred reduction)}
Proposition~\ref{prop:complete} needs a lower bound on a \emph{verbatim-only} agent's capability on $T_{\mathrm{trk}}$. We prove it directly rather than by reduction to Theorem~\ref{thm:horizon}, which does not apply: retrieval here is query-conditioned, so the horizon wall is silent, and the obstruction is instead that a running reduction over the \emph{whole} stream cannot be assembled from a bounded number of reads.

\begin{lemma}[No tracking without a compressive channel]
\label{lem:vtrack}
Let $T_{\mathrm{trk}}^{\Sigma}$ be the witness with $\phi(c)=\bigl(\sum_{i=1}^{N}c_i\bigr)\bmod K$, the $c_i$ i.i.d.\ uniform on $\mathbb{Z}_K$. Let $A$ be any agent that maintains \emph{no} cross-frame state and, over all its rounds, reads the contents of at most $k_{\mathrm{tot}}$ keyframes, with $k_{\mathrm{tot}}<N$; retrieval may be adaptive and arbitrarily clever. Then $\Pr[\hat y=y^{\star}]=1/K$.
\end{lemma}
\begin{proof}
Run the agent and let $\mathcal{T}$ denote its full transcript: the sequence of retrieval requests issued, the identities of the frames returned, and their contents. Let $S\subseteq[N]$ be the set of frames whose contents appear in $\mathcal{T}$, so $|S|\le k_{\mathrm{tot}}<N$ and $U=[N]\setminus S\neq\varnothing$. We claim that conditioned on $\mathcal{T}$, the values $\{c_i\}_{i\in U}$ are i.i.d.\ uniform on $\mathbb{Z}_K$. This is the standard exposure argument: process the rounds in order; by the read-only convention of \S\ref{app:conv} each request is a function of the contents already revealed, and its response reveals only the contents of the frames it returns. Hence at every step the conditional law of the not-yet-revealed contents is unchanged from the prior, and the claim follows by induction on rounds.

Now write $y^{\star}=\bigl(\sum_{i\in S}c_i+\sum_{i\in U}c_i\bigr)\bmod K$. The first sum is $\mathcal{T}$-measurable. The second is a sum of $|U|\ge1$ i.i.d.\ uniform elements of the group $\mathbb{Z}_K$, hence itself uniform on $\mathbb{Z}_K$ and independent of $\mathcal{T}$. Therefore $y^{\star}$ is uniform on $\mathbb{Z}_K$ conditionally on $\mathcal{T}$, and $\hat y$ is $\mathcal{T}$-measurable, so $\Pr[\hat y=y^{\star}\mid\mathcal{T}]=1/K$ pointwise. Taking expectations gives the claim.
\end{proof}

\emph{Why this is the right bound for the real system.} At the default operating point a keyframe costs a measured $\tau=259$--$432$ tokens, so $B=4$k caps $k_{\mathrm{tot}}$ at $9$--$15$ frames against $N=300$: the hypothesis $k_{\mathrm{tot}}<N$ is satisfied by a factor of $20$ to $33$, and it is satisfied \emph{by the budget}, not by an assumption about the agent. The lemma is also the formal content of the design claim in \S\ref{sec:method} that $\cs$ is an \emph{online accumulator}: what $\cs$ buys is not compression for its own sake but the ability to visit all $N$ keyframes with $O(1)$ tokens of carry, which no per-decision retrieval budget can replicate. Note finally that the lemma is tight and unimprovable in the stated model: an agent that could read all $N$ frames would answer perfectly, so the $k_{\mathrm{tot}}<N$ hypothesis cannot be dropped.

\subsection{Proof of Proposition~\ref{prop:complete} (access-completeness)}
\emph{Upper bounds for \ASP{}.} (RET) Router class \textsc{ret} allocates $B_r\ge B_r^{\star}$; by Assumption~\ref{ass:sep}(ii), the target keyframe is retrieved with probability $\ge1-\delta_r$, and by Assumption~\ref{ass:repr} it is read with error $\le\epsilon_0$; a union bound gives $\Cap_B\ge1-(\epsilon_0+\delta_r)$. (TRK) $\cs$ maintains $\phi$-sufficient statistics exactly when each of the $N$ updates succeeds; with per-update error $\epsilon_u$ a union bound gives $\Cap_B\ge1-N\epsilon_u$. (CMP) By Assumption~\ref{ass:sep}(iii), $R\ge d_q^{\max}$, and the positive direction of Theorem~\ref{thm:rounds} gives $\Cap_B\ge1-d_q\epsilon_1$ once $\delta_r$ is folded into $\epsilon_1$.

\emph{Lower bounds for single channels.} On $T_{\mathrm{ret}}$, Theorem~\ref{thm:shannon} gives $\Cap_B(\cs)\le\kappa_s$ with $\kappa_s$ as in \eqref{eq:kappa}. On $T_{\mathrm{cmp}}$, each hop's operand is by Definition~\ref{def:stream} an attribute of the content vectors, so a correct answer requires a correct episodic lookup and $\Cap_B(\cs)\le\kappa_s$ there as well. On $T_{\mathrm{trk}}$, Lemma~\ref{lem:vtrack} gives $\Cap_B(\cv)\le1/K$; this is the step previously deferred, and it is now proved directly and unconditionally rather than by the reduction to Theorem~\ref{thm:horizon} sketched in an earlier draft, which as noted above does not go through for query-conditioned retrieval. \proofend

\subsection{Proof of Proposition~\ref{prop:super} (super-additivity), and Corollary~\ref{cor:dominance}}
\emph{Single channels on $T_{\mathrm{conj}}$.} An answer is scored correct only if both components are. Hence $\Cap_B(\cs)\le\Pr[\text{RET component correct}]\le\kappa_s$ by Theorem~\ref{thm:shannon}, and $\Cap_B(\cv)\le\Pr[\text{TRK component correct}]\le1/K$ by Lemma~\ref{lem:vtrack}. A blind agent guesses both components independently and uniformly, so $\Cap_B(\varnothing)=1/K^2$. \ASP{} answers both components by the RET and TRK arguments of the previous subsection, and a union bound over the two gives $\Cap_B(\cs{+}\cv{+}\rho)\ge1-\Delta$ with $\Delta=\epsilon_0+\delta_r+N\epsilon_u$. Strict super-additivity
$1-\Delta>\kappa_s+1/K-1/K^2$ is then exactly the stated condition $\Delta<1-\kappa_s-1/K+1/K^2$. At SEW-Bench parameters ($\kappa_s=0.547$, $K=8$), the condition reads $\Delta<0.344$, which Assumption~\ref{ass:sep} satisfies at the measured single-channel error rates of \S\ref{sec:analysis}. \proofend

\emph{Why the mixture will not carry this statement.} On $T=\tfrac13(T_{\mathrm{ret}}+T_{\mathrm{trk}}+T_{\mathrm{cmp}})$ the same bounds give $\Cap_B(\cs)\le\tfrac13(2\kappa_s+1)$ and $\Cap_B(\cv)\le\tfrac13(2+1/K)$, with $\Cap_B(\varnothing)=1/K$, so
\[
\Cap_B(\cs)+\Cap_B(\cv)-\Cap_B(\varnothing)\;\le\;\frac{2\kappa_s+3}{3}-\frac{2}{3K},
\]
and this is below $1$, the only regime in which any hybrid could exceed it, if and only if $\kappa_s<1/K$. At our own operating point $\kappa_s=0.547$ while $1/K=0.125$, so the right-hand side is $\approx1.28$ and the additive form of super-additivity is unavailable on the mixture \emph{as a matter of arithmetic, independent of how good the hybrid is}. This is why Proposition~\ref{prop:super} is stated on $T_{\mathrm{conj}}$, where each channel is separately obstructed, and why the mixture claim is the dominance statement instead. Each channel being near-ceiling on two of three families is exactly what makes the sum exceed $1$; a conjunctive family removes that slack by construction.

\emph{Proof of Corollary~\ref{cor:dominance}.} The two single-channel bounds are displayed above. For the hybrid, averaging the three per-family bounds of Proposition~\ref{prop:complete} gives $\Cap_B(\cs{+}\cv{+}\rho)\ge1-\tfrac13\Delta_T$ with $\Delta_T=\epsilon_0+\delta_r+N\epsilon_u+d_q\epsilon_1$. The gap over the better single channel is therefore at least
\[
1-\tfrac13\Delta_T-\tfrac13\max\{2\kappa_s+1,\;2+1/K\},
\]
which is the display in the corollary. An earlier draft simplified the maximum to $\tfrac13(2+1/K)$ ``since $\kappa_s\le1$'', which does not follow: $\tfrac13(2\kappa_s+1)\le\tfrac13(2+1/K)$ holds iff $\kappa_s\le\tfrac12(1+1/K)$, and $\kappa_s\le1$ permits $\tfrac13(2\kappa_s+1)$ up to $1$, well above $\tfrac13(2+1/K)=0.708$ at $K{=}8$. The conclusion is unchanged at our operating point ($\kappa_s=0.547\le0.5625$, so the verbatim channel is the stronger of the two and the gap is $\tfrac13(1-1/K-\Delta_T)$), but the inequality has to be checked against the measured $\kappa_s$ rather than assumed, and on a corpus with $\kappa_s>\tfrac12(1+1/K)$ the other branch is the operative one. Either way the gap is positive whenever $\Delta_T<\min\{1-1/K,\;2(1-\kappa_s)\}$, which is Assumption~\ref{ass:sep} at the registered margins. \proofend

\subsection{What the theory does and does not license}
\label{app:scope}

Table~\ref{tab:scope} is in \S\ref{sec:exp} rather than here, because it is the map a
reader needs \emph{before} the measured numbers: it says which of the seven statements forbid
anything at the parameters we ran, and two of them do not. Three entries deserve saying out loud
rather than leaving in a cell.
\emph{First}, the two vacuous rows are vacuous for opposite reasons. Theorem~\ref{thm:circuit}
is a statement about $d_q\to\infty$ and our deepest chain is $d_q{=}3$, so it is the corpus that
is too shallow, not the theorem that is weak; SEW-Bench cannot test it and \S\ref{app:ewdesign}'s
$d_q\in\{3,4,5\}$ design is what would. Theorem~\ref{thm:rounds}\ref{it:rounds-loc} is vacuous
because $m/N=20.5$ bits per keyframe is enough state to name any pointer, so a single round
\emph{could} in principle localise; the reason our $R{=}1$ ablation nonetheless fails is
empirical, not information-theoretic, and \S\ref{sec:analysis} labels it as such.
\emph{Second}, Proposition~\ref{prop:super} has no witness in the corpus: SEW-Bench ships RET,
TRK and CMP items and no dedicated conjunctive subset, so the super-additivity claim is
\textbf{not evaluated here}. The CMP ordinal-chain items are the nearest instance, they need
an ordered reduction and an episodic lookup in the same answer, but they are not the
$T_{\mathrm{conj}}$ of Proposition~\ref{prop:super}, and treating them as one would be
reporting a test we did not run. \emph{Third}, Corollary~\ref{cor:dominance} holds in its
first branch by a margin of $0.016$ in $\kappa_s$. That is close enough that a corpus with a
slightly larger state, or a slightly smaller $Nd$, would put the compressive channel ahead of
the verbatim one as the baseline to beat, and the corollary would then have to be read off its
second branch.

Collecting the quantifiers, since the four walls are not equally strong. Theorems~\ref{thm:shannon} and~\ref{thm:horizon} and Lemma~\ref{lem:vtrack} are unconditional and hold against unbounded compute; Theorem~\ref{thm:rounds} is unconditional but proved here only for $R{=}1$; Theorem~\ref{thm:circuit} is conditional on $\mathsf{TC}^0\neq\mathsf{NC}^1$ \emph{and} on the $S_5$ encoding being realized by naturalistic items. Every bound is a statement about a \emph{witness family}, not about OpenEQA or VSI-Bench, whose items mix the three modes in unknown proportion; EW-Bench exists precisely so that each wall has an item set that instantiates it one-to-one (Appendix~\ref{app:ewbench}). Finally, none of the bounds says a compressive channel is useless or that scale is irrelevant: they say each single channel has a family on which it is capped by a constant, and Assumption~\ref{ass:repr} is where scale enters, through the margins $\epsilon_0,\epsilon_1$ that Corollary~\ref{cor:conv} turns into the convergence prediction.

\section{Corpus Construction: SEW-Bench (built), and EW-Bench (registered, not built)}
\label{app:ewbench}

A diagnostic corpus exists here for one reason: OpenEQA, VSI-Bench and EgoSchema mix the three access modes in unknown proportion, so a win on them cannot be attributed to any particular wall. The registered corpus for that job was EW-Bench, $900$ items over $60$ rendered HM3D/ScanNet walkthroughs. \textbf{It was not built}: rendering it needs habitat-sim and an HM3D/ScanNet licence, and we have no licence-free path to either. Its design is retained in \S\ref{app:ewdesign} because it is part of the registration and because it is the corpus this work should next be run on. What we built and measured instead is \textbf{SEW-Bench}, described first.

\begin{asptable}[!t]
\centering\small
\caption{\textbf{Corpus parameters and the ceilings they imply.} Left: an earlier version of
this generator, kept because it is the case where the theory is \emph{arithmetically silent}.
Centre: what was built and measured. Right: the registered corpus, not built. Every ceiling is
computed from the row above it, not fitted. The load-bearing row is $Nd$: it is what
Theorem~\ref{thm:shannon} divides by, and moving it from $1{,}350$ to $\VNd$ is what turned a
vacuous ceiling into one a method can be measured against.}
\label{tab:corpus}
\begin{tabular}{@{}lccc@{}}
\toprule
& SEW draft & \textbf{SEW-Bench} & EW-Bench \\
& (superseded) & \textbf{(built, measured)} & (registered, not built) \\
\midrule
\multicolumn{4}{@{}l}{\emph{stream}}\\
keyframes $N$                       & $150$   & $\VN$    & ${\approx}450$ \\
queryable attributes per frame $d$  & $9$     & $\Vd$    & ${\approx}96$ \\
attribute values $Nd$               & $1{,}350$ & $\VNd$ & ${\approx}43{,}200$ \\
alphabet $K$                        & $8$     & $\VK$    & $8$ \\
render                              & schematic & $640{\times}480$ schematic & HM3D / ScanNet \\
episodes / scenes                   & $4$     & $4$      & $60$ \\
\addlinespace[2pt]
\multicolumn{4}{@{}l}{\emph{agent, at the default operating point}}\\
state cap $L_s$ (tokens)            & $512$   & $512$    & $512$ \\
bits per token $\beta$              & $12$    & $12$     & $12$ \\
state $m=L_s\beta$ (bits)           & $6{,}144$ & $\VmBits$ & $6{,}144$ \\
budget $B$ (tokens/decision)        & $4{,}096$ & $4{,}096$ & $4{,}096$ \\
measured $\tau$ (tokens/frame)      & ---     & $259$--$432$ & --- \\
frames affordable $k_{\mathrm{tot}}$ & ---    & $9$--$15$ & $9$--$15$ \\
\addlinespace[2pt]
\multicolumn{4}{@{}l}{\emph{implied ceilings} (\S\ref{app:proofs}; lower is a stronger statement)}\\
$\kappa_s$, Thm.~\ref{thm:shannon} on RET & $1.850$ \textbf{(vacuous)} & $\VkappaS$ & $0.381$ \\
Thm.~\ref{thm:horizon} at $w{=}12$        & $19.5\%$ & $16.0\%$ & $14.8\%$ \\
Lem.~\ref{lem:vtrack} on TRK              & $\VoneOverK\%$ & $\VoneOverK\%$ & $\VoneOverK\%$ \\
Thm.~\ref{thm:rounds}\ref{it:rounds-stateless} at $k{=}8$ & $17.2\%$ & $14.8\%$ & $14.1\%$ \\
\addlinespace[2pt]
\multicolumn{4}{@{}l}{\emph{items}}\\
RET / TRK / CMP                     & ---     & $32$ / $24$ / $24$ & $300$ / $300$ / $300$ \\
$T_{\mathrm{conj}}$ (Prop.~\ref{prop:super}) & --- & \textbf{none} & none \\
maximum chain depth $d_q^{\max}$    & ---     & $3$      & $5$ \\
\bottomrule
\end{tabular}
\end{asptable}

\subsection{SEW-Bench: what was built}
\label{app:sew}
Four episodes of $N{=}300$ frames, $80$ questions ($32$ RET, $24$ TRK, $24$ CMP), generated programmatically with a fixed seed by \texttt{annotation/generate\_synthetic.py} and released with the code.

\paragraph{A regeneration contract, and why it is part of the artefact.} The generator walks one
\texttt{random.Random} across every frame of every episode, so a fix that consumes even one
extra draw from that stream changes the whole corpus downstream of it, and a changed corpus
invalidates the offline keyframe caches and every one of the $207$ measured cells, not just the
items the fix was about. The two corrections of \S\ref{sec:cmpdefect} are therefore written to
draw from item-local streams keyed on the corpus seed, and the property that makes them safe is
stated as a test rather than as an intention: regenerating at the shipped seed leaves all
$1{,}200$ frames and \texttt{meta.json} byte-identical and changes exactly the four questions
whose naming order had to flip. That is what confines the outstanding re-run
(\S\ref{sec:limits}, item (i)) to the $69$ CMP cells: the offline caches are keyed by
(backbone, episode, $\tau_g$), and contain no question, so they are reused verbatim, and the
marginal cost is the scored decisions alone. Anyone extending the corpus should preserve the
same discipline, new items drawn from their own streams, or accept a full regrid.

\textbf{Frames.} Each frame is a $640{\times}480$ schematic render of one room visit: a header giving the room name, the frame index and a door state, and a $4{\times}4$ grid of $16$ objects. Each object is a coloured rectangle carrying a printed capital letter, with its class name written underneath. Colour and letter are each drawn uniformly from a $K{=}8$ alphabet; the tag is a letter rather than a digit because with digits the weaker backbones read the number printed on an object as \emph{how many of these I have seen}, and the TRK count items would then measure that confusion instead of the compressive channel (observed directly: Gemma~3~12B reported counts of $8$ and $9$ after six frames); the class name is written as text but is never the queried attribute. That split is deliberate: labelling the \emph{colour} would turn every RET item into optical character recognition and hand the verbatim channel an advantage that has nothing to do with vision.

\begin{aspfigure}[!t]
\centering
\begin{tikzpicture}[x=0.0466cm, y=1cm, font=\scriptsize]
  \fill[blue!12] (0,0.34) rectangle (14,0.78);
  \fill[orange!16] (14,0.34) rectangle (36,0.78);
  \node[font=\tiny] at (24.5,0.56) {Bd};
  \fill[green!13] (36,0.34) rectangle (49,0.78);
  \fill[black!8] (49,0.34) rectangle (62,0.78);
  \fill[blue!12] (62,0.34) rectangle (83,0.78);
  \node[font=\tiny] at (72.0,0.56) {St};
  \fill[green!13] (83,0.34) rectangle (109,0.78);
  \node[font=\tiny] at (95.5,0.56) {Lv};
  \fill[black!8] (109,0.34) rectangle (124,0.78);
  \fill[green!13] (124,0.34) rectangle (137,0.78);
  \fill[orange!16] (137,0.34) rectangle (155,0.78);
  \node[font=\tiny] at (145.5,0.56) {Bd};
  \fill[green!13] (155,0.34) rectangle (190,0.78);
  \node[font=\tiny] at (172.0,0.56) {Lv};
  \fill[orange!16] (190,0.34) rectangle (202,0.78);
  \fill[black!8] (202,0.34) rectangle (215,0.78);
  \fill[red!11] (215,0.34) rectangle (237,0.78);
  \node[font=\tiny] at (225.5,0.56) {Kt};
  \fill[purple!12] (237,0.34) rectangle (258,0.78);
  \node[font=\tiny] at (247.0,0.56) {Pt};
  \fill[green!13] (258,0.34) rectangle (279,0.78);
  \node[font=\tiny] at (268.0,0.56) {Lv};
  \fill[black!8] (279,0.34) rectangle (297,0.78);
  \node[font=\tiny] at (287.5,0.56) {Hl};
  \fill[green!13] (297,0.34) rectangle (300,0.78);
  \draw[black!45, line width=0.3pt] (0,0.34) rectangle (300,0.78);
  \node[anchor=east, font=\scriptsize] at (-4,0.56) {room};
  \draw[red!70!black, line width=0.7pt] (0,0.00) -- (36,0.00) -- (36,0.18) -- (72,0.18) -- (72,0.00) -- (137,0.00) -- (137,0.18) -- (186,0.18) -- (186,0.00) -- (219,0.00) -- (219,0.18) -- (239,0.18) -- (239,0.00) -- (283,0.00) -- (283,0.18) -- (300,0.18);
  \node[anchor=east, font=\scriptsize] at (-4,0.09) {door};
  \draw[blue!70!black, line width=0.8pt] (19,0.86) -- (19,1.06);
  \draw[blue!70!black, line width=0.8pt] (62,0.86) -- (62,1.06);
  \draw[blue!70!black, line width=0.8pt] (97,0.86) -- (97,1.06);
  \draw[blue!70!black, line width=0.8pt] (98,0.86) -- (98,1.06);
  \draw[blue!70!black, line width=0.8pt] (118,0.86) -- (118,1.06);
  \draw[blue!70!black, line width=0.8pt] (135,0.86) -- (135,1.06);
  \draw[blue!70!black, line width=0.8pt] (187,0.86) -- (187,1.06);
  \draw[blue!70!black, line width=0.8pt] (274,0.86) -- (274,1.06);
  \node[anchor=east, font=\scriptsize] at (-4,0.96) {RET};
  \fill[purple!75!black] (33,1.20) circle (1.6pt);
  \node[anchor=south, font=\tiny, text=purple!75!black] at (33,1.26) {1};
  \fill[purple!75!black] (160,1.20) circle (1.6pt);
  \node[anchor=south, font=\tiny, text=purple!75!black] at (160,1.26) {2};
  \fill[purple!75!black] (202,1.20) circle (1.6pt);
  \node[anchor=south, font=\tiny, text=purple!75!black] at (202,1.26) {3};
  \draw[purple!60, dashed, line width=0.4pt] (33,1.20) -- (160,1.20) -- (202,1.20);
  \node[anchor=east, font=\scriptsize] at (-4,1.20) {CMP};
  \fill[teal!70!black] (42,-0.32) circle (1.4pt);
  \fill[teal!70!black] (72,-0.32) circle (1.4pt);
  \fill[teal!70!black] (100,-0.32) circle (1.4pt);
  \fill[teal!70!black] (114,-0.32) circle (1.4pt);
  \fill[teal!70!black] (126,-0.32) circle (1.4pt);
  \fill[teal!70!black] (128,-0.32) circle (1.4pt);
  \fill[teal!70!black] (178,-0.32) circle (1.4pt);
  \fill[teal!70!black] (275,-0.32) circle (1.4pt);
  \fill[teal!70!black] (276,-0.32) circle (1.4pt);
  \node[anchor=east, font=\scriptsize] at (-4,-0.32) {TRK};
  \draw[black!55, line width=0.6pt] (0,-0.68) -- (0,-0.52);
  \draw[black!55, line width=0.6pt] (19,-0.68) -- (19,-0.52);
  \draw[black!55, line width=0.6pt] (39,-0.68) -- (39,-0.52);
  \draw[black!55, line width=0.6pt] (59,-0.68) -- (59,-0.52);
  \draw[black!55, line width=0.6pt] (79,-0.68) -- (79,-0.52);
  \draw[black!55, line width=0.6pt] (99,-0.68) -- (99,-0.52);
  \draw[black!55, line width=0.6pt] (119,-0.68) -- (119,-0.52);
  \draw[black!55, line width=0.6pt] (139,-0.68) -- (139,-0.52);
  \draw[black!55, line width=0.6pt] (159,-0.68) -- (159,-0.52);
  \draw[black!55, line width=0.6pt] (179,-0.68) -- (179,-0.52);
  \draw[black!55, line width=0.6pt] (199,-0.68) -- (199,-0.52);
  \draw[black!55, line width=0.6pt] (219,-0.68) -- (219,-0.52);
  \draw[black!55, line width=0.6pt] (239,-0.68) -- (239,-0.52);
  \draw[black!55, line width=0.6pt] (259,-0.68) -- (259,-0.52);
  \draw[black!55, line width=0.6pt] (279,-0.68) -- (279,-0.52);
  \draw[black!55, line width=0.6pt] (299,-0.68) -- (299,-0.52);
  \node[anchor=east, font=\scriptsize] at (-4,-0.60) {\textsc{Uniform-}16};
  \fill[black!22] (288,-1.02) rectangle (300,-0.86);
  \draw[black!55, line width=0.4pt] (288,-1.02) rectangle (300,-0.86);
  \node[anchor=east, font=\scriptsize] at (-4,-0.94) {last-$12$};
  \draw[black!60] (0,-1.20) -- (300,-1.20);
  \draw[black!60] (0,-1.20) -- (0,-1.28) node[anchor=north, font=\tiny] {0};
  \draw[black!60] (50,-1.20) -- (50,-1.28) node[anchor=north, font=\tiny] {50};
  \draw[black!60] (100,-1.20) -- (100,-1.28) node[anchor=north, font=\tiny] {100};
  \draw[black!60] (150,-1.20) -- (150,-1.28) node[anchor=north, font=\tiny] {150};
  \draw[black!60] (200,-1.20) -- (200,-1.28) node[anchor=north, font=\tiny] {200};
  \draw[black!60] (250,-1.20) -- (250,-1.28) node[anchor=north, font=\tiny] {250};
  \draw[black!60] (300,-1.20) -- (300,-1.28) node[anchor=north, font=\tiny] {300};
  \node[anchor=north, font=\scriptsize] at (150,-1.49) {keyframe index $t$};
\end{tikzpicture}
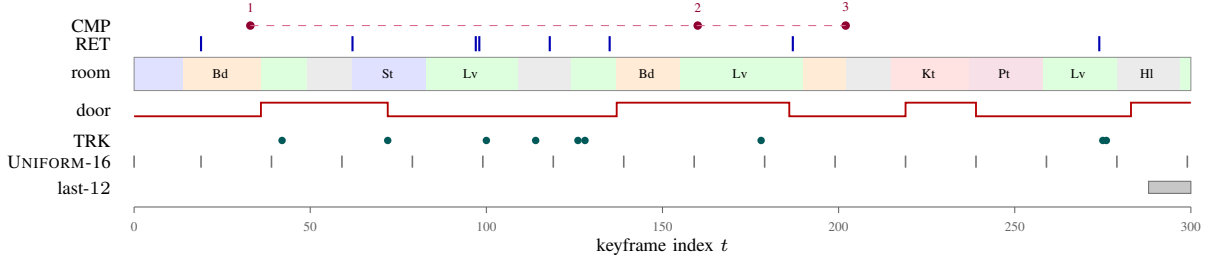
\caption{\textbf{One real SEW-Bench episode, on a time axis} (ep1 of the shipped corpus;
positions recovered by replaying the generator at its fixed seed). Reading upward: the
\textsf{door} trace is a TRK reduction that only an online accumulator can follow, since its
value at $t{=}300$ depends on all seven toggles; \textsf{RET} marks the eight classes visible in
\emph{exactly one} frame each; \textsf{CMP} marks one ordinal chain, whose question asks for an
attribute of the \emph{second} sighting and so cannot be answered without first ordering all
three; \textsf{TRK} marks the nine sightings of one counted class. Reading downward: what the two
query-independent policies of Theorem~\ref{thm:horizon} are allowed to look at. The room band
is the reason the two are different policies rather than the same one, \textsf{Lv} and
\textsf{St} each recur three or more times, so the most recent frames are not the relevant ones,
and a last-$w$ window sees only the end of the walkthrough while every one of the eight RET
targets sits earlier than frame $275$. Across all four episodes the two policies cover $2/32$ and $0/32$ of the RET
targets respectively, against the $w/N$ rates of $5.3\%$ and $4.0\%$ the theorem predicts: the
horizon wall is not an asymptotic remark about this corpus, it is where the targets are.}
\label{fig:timeline}
\end{aspfigure}

\textbf{Trajectory.} The room sequence is generated in runs of $12$--$26$ frames drawn from six rooms, so the stream contains revisit loops and recency dissociates from relevance, the property that makes a last-$w$ window a genuinely different policy from query-conditioned retrieval.

\textbf{The attribute count $d$, and why it is the load-bearing parameter.} Every bound in Appendix~\ref{app:proofs} scales with the number of queryable attribute values $Nd$, not with the frame count $N$. We count $d$ conservatively: $16$ objects $\times$ two $K$-ary attributes (colour, letter) $=32$, and the room, door state and frame index are \emph{excluded} because they are not $K$-ary, which can only make the resulting ceiling looser. Hence $d{=}32$ and $Nd{=}9{,}600$ per episode. This is not a cosmetic choice. An earlier version of the generator had $d{=}9$ at $N{=}150$, giving $Nd{=}1{,}350$ against an $m\approx6{,}144$-bit state and $\kappa_s{=}1.85>1$: Theorem~\ref{thm:shannon} was then \emph{arithmetically vacuous}, and the measured $0\%$ \textsc{State-only} RET score on that corpus was not evidence for it. Reporting $Nd$ is what makes the ceiling checkable rather than asserted, and we report the version of the corpus on which it binds.

\textbf{Filler randomisation, and why the entropy has to be real.} All $16$ object slots, including their classes, are re-randomised every frame from a filler pool disjoint from every queried pool. Had the room contents been fixed per room, the stream would carry only $|\mathrm{rooms}|\times16\times2$ distinct attribute values in total and the $Nd$ in the bound would be a fiction, a compressive state could memorise the whole layout once. Re-randomisation is what makes the compression the state performs genuinely lossy.

\textbf{RET ($32$ items, $8$ per episode).} Each RET target is a class appearing in \emph{exactly one} frame of the episode, so ``you saw a kettle exactly once; what colour was it?'' has a unique answer, and the horizon-wall ceiling applies at multiplicity one exactly rather than approximately. Half the items query the colour and half the letter. A target is not distinguishable from a filler at the time the state is written: the online state builder sees one frame at a time and cannot know that a class will never recur, which is what lets the Fano argument of \S\ref{app:shannon} apply to the run we actually perform rather than to an idealised one.

\textbf{TRK ($24$ items, $6$ per episode).} Four count items (``how many mugs in total?'', over $3$--$9$ scattered sightings of a class that never appears as filler), the final door state after $5$--$9$ toggles, and the number of toggles. These are exactly the $\phi$ of Definition~\ref{def:stream}, and by Lemma~\ref{lem:vtrack} a verbatim-only agent is pinned at $1/K$ on them whenever its per-decision read budget is below $N$ frames, which at $B{=}4$k it is, by a factor of $20$ to $33$.

\textbf{CMP ($24$ items, $6$ per episode).} Four ordinal-chain items: a class appears exactly three times, in three distinct rooms, with three distinct colours, and the question asks for its colour \emph{the second time it was seen}. Answering requires locating all occurrences, ordering them, and then reading a pixel attribute of one, a running reduction composed with an episodic lookup, so the item obstructs each single channel separately in the sense of Proposition~\ref{prop:super}. Two order items per episode ask whether one singleton was seen before or after another, which is $d_q{=}2$ with both operands in the verbatim channel.

\textbf{Answerability gate.} The generator refuses to ship the corpus if any two items share a question string with different golds, or if the four class pools intersect, or if an episode is too short to place its scheduled events. A confident gold label on an unanswerable question is worse than no item at all, because it surfaces as a model failure; the first draft of the generator produced two contradictory RET items about the same object and the gate exists because of it.

\textbf{What SEW-Bench does not test.} It is a 2-D schematic render. It tests access structure under a token budget, which is what \S\ref{sec:theory} is about. It tests nothing about perception in natural scenes, egocentric motion, occlusion, or lighting, and no result from it should be read as an OpenEQA-style result. \S\ref{sec:limits} states the consequence for the paper's claims.

Figure~\ref{fig:timeline} draws one shipped episode on a time axis, which is the quickest way to see why recency and relevance dissociate here: the revisit loops put queried singletons far from the end of the stream.

\subsection{EW-Bench: the registered design, not built}
\label{app:ewdesign}
Recorded for auditability. Sixty walkthroughs rendered from HM3D \citep{hm3d} and ScanNet \citep{scannet} scenes (40 HM3D, 20 ScanNet, no scene reused) at $1$--$2$\,m/s, $8$--$15$ minutes each, with scripted revisit loops, $\ge4$ state toggles and $\ge3$ containment changes per walkthrough, and $\ge8$ objects appearing in one or two keyframes; $\tau_g$ fixed once on five held-out walkthroughs to give $N\in[380,520]$. The registered attribute count is $d\approx96$ ($\approx8$ visible objects $\times$ $10$ queryable attributes, plus $16$ scene-level attributes), hence $Nd\approx43{,}200$ at $N\approx450$ and $\kappa_s\approx0.38$, a tighter ceiling than SEW-Bench's $0.547$, which is the main thing the real corpus would buy. Items: $300$ RET (single-attribute queries on objects of multiplicity $1$--$2$, with the strict multiplicity-one subset reported separately), $300$ TRK (running reductions), $300$ CMP ($d_q\in\{3,4,5\}$, one hundred per depth, surface forms paraphrased so that depth is not readable off the template), answer alphabets $K{=}8$ wherever the attribute admits eight distinguishable values, MCQ where the alphabet permits and LLM-Match otherwise, distractors drawn from values that occur elsewhere in the same walkthrough, and per-item reporting of any scene overlap with published training-corpus manifests. Building it is the first item of future work in \S\ref{sec:limits}.

\subsection{A worked item (illustrative)}
\label{app:worked}
The following is the access pattern \ASP{} is designed to produce, on an EW-Bench-style $d_q{=}3$ item. \emph{This trace is constructed by hand from the protocol.} Measured traces on SEW-Bench items, in exactly this format, are in the released \texttt{results/cells/} records.

\begin{quote}\small
\textbf{Query.} ``The mug that was on the desk in the room you entered right after the room with the red armchair, which room is it in now?''\\[2pt]
\textbf{Round 1.} Router classifies \textsc{cmp}; allocates $(B_s,B_r,B_c)=(1014,1419,1622)$ of $B{=}4096$ after $B_\rho{=}41$, i.e.\ the $0.25/0.35/0.40$ profile of Table~\ref{tab:alloc} applied to $B-B_\rho=4055$. Model emits \textsc{state}(\textrm{AgentTrace}). Reads the room-entry order; resolves ``the room right after the room with the red armchair'' $\to$ \textrm{study}. Debited: $118$ tokens.\\[2pt]
\textbf{Round 2.} Model emits \textsc{retrieve}(``mug on desk in study''). Index returns $3$ keyframes; the mug's identity is read off the highest-ranked one. Debited: $742$ tokens.\\[2pt]
\textbf{Round 3.} Model emits \textsc{state}(\textrm{Objects[mug\_02].room}). Reads \textrm{kitchen}, the value maintained by $\cs$ across a move event that occurred $214$ keyframes after the mug was last seen. Emits \textsc{answer}(\textrm{kitchen}). Debited: $96$ tokens. Total: $997$ of $4096$.
\end{quote}

\noindent The item is a witness for three of the four walls at once, which is what makes it a $T_{\mathrm{conj}}$-style instance in the sense of Proposition~\ref{prop:super}: hop~1 needs the ordered trace (a running reduction, unavailable to $\cv$ by Lemma~\ref{lem:vtrack}), hop~2 needs verbatim appearance (unavailable to $\cs$ by Theorem~\ref{thm:shannon}), and the dependency of hop~3 on hop~2's output needs adaptive rounds (unavailable to a single-round retriever by Theorem~\ref{thm:rounds}). The corresponding single-channel failure modes are specific: \textsc{State-only} resolves hops~1 and~3 but cannot identify \emph{which} mug; \textsc{Retr-only} identifies the mug but cannot order the rooms; a single-round retriever can do either but cannot chain them.

\section{Registered Prediction Methodology}
\label{app:predictions}

Predictions are computed, not guessed, and the computation is stated here so that a reader can check any cell without our spreadsheet. Every registered value is $\mathrm{anchor}+\mathrm{generation}+\mathrm{structure}$, capped by the applicable wall ceiling.

\begin{asptable}[!t]
\centering\small\setlength{\tabcolsep}{4.5pt}
\caption{\textbf{Provenance of every input to a registered prediction.} One row per number
that enters the arithmetic of \S\ref{app:predictions}, so that a reader can audit the
registration without our worksheet. The last column is the one that matters: exactly one
input is our own prior work, it enters the \ASP{} rows only, and \S\ref{app:predictions}
already names it as the step most likely to be wrong. A reader who discounts it can strike
the \ASP{} rows of Table~\ref{tab:registered} and the baseline rows still stand on published
anchors alone.}
\label{tab:anchors}
\begin{tabular}{@{}>{\raggedright\arraybackslash}p{3.85cm}>{\raggedright\arraybackslash}p{3.2cm}>{\raggedright\arraybackslash}p{3.5cm}>{\raggedright\arraybackslash}p{3.6cm}c@{}}
\toprule
Input & Value used & Source & Enters & Ours? \\
\midrule
OpenEQA, GPT-4V            & $49.6$ & \citep{openeqa} & OpenEQA column & no \\
OpenEQA, blind LLM         & ${\approx}33$ & \citep{openeqa} & \textsc{Blind} row & no \\
AlanaVLM-7B                & $46.7$ & \citep{alanavlm} & $7$--$8$B rows & no \\
Structured-memory delta    & $+3$ at $5\times$ fewer frames & \citep{graphpad}
                           & \textsc{State-only} rows & no \\
VSI-Bench open-model band  & $32$--$40$ & \citep{vsibench} & VSI-Bench column & no \\
EgoSchema ${\le}31$B band  & published band & \citep{egoschema,videoagent}
                           & EgoSchema column & no \\
\addlinespace[2pt]
Generation adjustment      & $+4$ per backbone & mean published delta$^{\dagger}$
                           & every cell & no \\
\addlinespace[2pt]
Structure adjustment       & CCH hybrid gain $\times$ EW redundancy factor & our prior work
                           & \ASP{} rows only & \textbf{yes} \\
\addlinespace[2pt]
Wall ceilings              & $\kappa_s$, horizon, round & derived, Table~\ref{tab:corpus}
                           & caps single-channel rows & derived \\
Interval rule              & $\pm\max(\text{dispersion},3)$ & \S\ref{app:predictions}
                           & every cell & rule \\
\midrule
\multicolumn{5}{@{}p{0.97\linewidth}@{}}{\footnotesize $^{\dagger}$Averaged over the
Qwen3.6$\to$3.8 and Gemma\,3$\to$4 transitions \citep{qwen38,gemma4}.}\\
\bottomrule
\end{tabular}
\end{asptable}

\subsection{Anchors}
Table~\ref{tab:anchors} gives the provenance of every one of these in one place; the prose
below states the same inputs. Published numbers, used verbatim and never re-scaled: OpenEQA GPT-4V $49.6$ and blind-LLM $\approx33$ \citep{openeqa}; AlanaVLM-7B $46.7$ \citep{alanavlm}; GraphPad-style structured memory $+3$\,pts at $5\times$ fewer frames \citep{graphpad}; the VSI-Bench open-model band $32$--$40$ \citep{vsibench}; the EgoSchema long-video band for $\le$31B models \citep{egoschema,videoagent}.

\subsection{Adjustments}
\textbf{(i) Generation.} $+4$\,pts for the 2026 open generation, fitted as the mean published delta on vision suites across the Qwen3.6$\to$3.8 and Gemma\,3$\to$4 transitions \citep{qwen38,gemma4}. Applied once per backbone, not per benchmark.
\textbf{(ii) Structure.} \ASP{} deltas are CCH's measured hybrid gains transported by the redundancy factor of EW streams (the multiplicity distribution of \S\ref{app:ewbench}); this is the only step that uses our own prior work as a quantitative input, and it is the step most likely to be wrong.
\textbf{(iii) Wall ceilings.} Registered \textsc{State-only} and \textsc{Uniform} values on EW subtests are capped at the ceilings computed in \S\ref{sec:analysis} from the measured $Nd$, $K$, and $w$. A registered value is never placed above its own wall ceiling; where the uncapped anchor-plus-adjustment estimate exceeded the ceiling, the ceiling was used and the fact is flagged in the release worksheet.

\subsection{Intervals and what would count as a miss}
Prediction intervals are $\pm$ the maximum of (a), the dispersion across anchors used for that cell and (b) $3$ points, the floor being there so that a cell with a single anchor does not acquire false precision. Intervals are \emph{registered}, so a measured value outside its interval is a recorded miss even when the qualitative direction is right; \S\ref{sec:falsify}'s criteria are about direction and effect size, and the intervals are the finer-grained accounting underneath them. The full per-cell worksheet (anchor set, arithmetic, and the resulting interval for each of the cells in Tables~\ref{tab:main}--\ref{tab:ablation}) is released with the protocol, and the registered values in this paper are a copy of it, not a summary.

\subsection{The registered cells that were not run}
Three of the six registered benchmark columns have no licence-free data path (\S\ref{sec:limits}). Their predictions are reproduced verbatim in Table~\ref{tab:registered} rather than deleted, so that a reader can see the full registration and hold the unrun part of it against any future version of this work. Nothing in the body claims these values as results.

\begin{asptable}[!t]
\centering\footnotesize\setlength{\tabcolsep}{5pt}
\caption{\textbf{Registered but not run.} The three natural-video benchmarks of the pre-registration, with their frozen predictions\pred{} and registered 90\% prediction intervals. They are not measured here: OpenEQA, VSI-Bench and EgoSchema need episode-frame dumps we have no licence-free path to, and the registered EW-Bench itself needs habitat-sim plus an HM3D/ScanNet licence (\S\ref{sec:limits}). No value in this table is claimed as a result.}
\label{tab:registered}
\begin{tabular}{llccc}
\toprule
Backbone & Method & OpenEQA & VSI-Bench & EgoSchema \\
\midrule
\multirow{6}{*}{Qwen3.8-27B}
 & \textsc{Blind} & 34.1\pred{\scriptsize$\pm$2} & 24.0\pred{\scriptsize$\pm$2} & 29.5\pred{\scriptsize$\pm$3} \\
 & \textsc{Uniform-}$w$ & 52.6\pred{\scriptsize$\pm$3} & 38.5\pred{\scriptsize$\pm$3} & 62.0\pred{\scriptsize$\pm$3} \\
 & \textsc{Retr-only} & 57.1\pred{\scriptsize$\pm$3} & 39.8\pred{\scriptsize$\pm$3} & 63.4\pred{\scriptsize$\pm$3} \\
 & \textsc{State-only} & 55.4\pred{\scriptsize$\pm$3} & 41.2\pred{\scriptsize$\pm$3} & 61.1\pred{\scriptsize$\pm$3} \\
 & \textsc{EGAgent} (reimpl.) & 58.9\pred{\scriptsize$\pm$3} & 42.1\pred{\scriptsize$\pm$3} & 64.0\pred{\scriptsize$\pm$3} \\
 & \ASP{} (ours) & 62.3\pred{\scriptsize$\pm$3} & 46.0\pred{\scriptsize$\pm$3} & 66.8\pred{\scriptsize$\pm$3} \\
\midrule
\multirow{2}{*}{Gemma~4~31B}
 & \textsc{Uniform-}$w$ & 53.4\pred{\scriptsize$\pm$3} & 39.0\pred{\scriptsize$\pm$3} & 63.1\pred{\scriptsize$\pm$3} \\
 & \ASP{} (ours) & 62.9\pred{\scriptsize$\pm$3} & 46.6\pred{\scriptsize$\pm$3} & 67.5\pred{\scriptsize$\pm$3} \\
\midrule
\multirow{2}{*}{Qwen3-Omni-30B-A3B}
 & \textsc{Uniform-}$w$ & 50.8\pred{\scriptsize$\pm$3} & 36.9\pred{\scriptsize$\pm$3} & 60.7\pred{\scriptsize$\pm$3} \\
 & \ASP{} (ours) & 60.5\pred{\scriptsize$\pm$3} & 44.3\pred{\scriptsize$\pm$3} & 65.0\pred{\scriptsize$\pm$3} \\
\midrule
\multirow{2}{*}{Qwen3-VL-8B}
 & \textsc{Uniform-}$w$ & 46.2\pred{\scriptsize$\pm$3} & 33.1\pred{\scriptsize$\pm$3} & 55.8\pred{\scriptsize$\pm$3} \\
 & \ASP{} (ours) & 59.4\pred{\scriptsize$\pm$3} & 42.7\pred{\scriptsize$\pm$3} & 62.9\pred{\scriptsize$\pm$3} \\
\midrule
\multirow{2}{*}{Qwen2.5-VL-7B}
 & \textsc{Uniform-}$w$ & 44.5\pred{\scriptsize$\pm$3} & 31.8\pred{\scriptsize$\pm$3} & 53.6\pred{\scriptsize$\pm$3} \\
 & \ASP{} (ours) & 57.8\pred{\scriptsize$\pm$3} & 41.2\pred{\scriptsize$\pm$3} & 60.7\pred{\scriptsize$\pm$3} \\
\bottomrule
\end{tabular}
\end{asptable}

\subsection{Known ways this methodology can fail}
Three, stated in advance. Anchor staleness: published baselines were measured under serving conditions we cannot fully reconstruct, so a systematic offset in the anchors shifts every registered value in the same direction, which is exactly why the headline tests of \S\ref{sec:exp} are paired per-question deltas rather than absolute scores. Transport error: step (ii) assumes CCH's symbolic gains transport through the redundancy factor alone, and if embodied redundancy interacts with retrieval quality the \ASP{} column is biased. Ceiling slack: the wall ceilings are upper bounds, so capping at them makes the single-channel registered values \emph{optimistic}, and a measured value below a registered single-channel number is therefore not evidence against the theory.

\section{Router Traces, Cost Accounting, and Secondary Experiments}
\label{app:traces}\label{app:edge}\label{app:qual}

\subsection{Traces}
For every decision we log the tuple $(z,B_s,B_r,B_c,R_{\mathrm{used}})$ plus the realised action sequence, so that every allocation in the paper is auditable back to a logged triple, and the realised split is reported in \S\ref{sec:analysis} rather than assumed. The traces are the primary diagnostic for router mis-specification: because $\widehat{\Cap}_z$ is a fitted surrogate, a systematic deviation between realised and tabled allocations localises the error to a wall price rather than to the architecture, and \S\ref{sec:falsify} treats it as a refinement rather than a refutation.

\begin{aspfigure}[!t]
\centering
\begin{tikzpicture}[x=0.00276cm, y=1cm, font=\scriptsize,
  seg/.style={draw=black!55, line width=0.3pt},
  lbl/.style={font=\tiny, anchor=center},
  rowname/.style={font=\scriptsize, anchor=south west}]

\node[rowname] at (0,1.74) {allocated by the router $\rho$};
\fill[black!22, seg]  (0,1.30)    rectangle (41,1.70);
\fill[blue!22, seg]   (41,1.30)   rectangle (1055,1.70);
\fill[green!25, seg]  (1055,1.30) rectangle (2474,1.70);
\fill[orange!30, seg] (2474,1.30) rectangle (4096,1.70);
\draw[seg] (0,1.30) rectangle (4096,1.70);
\node[lbl] at (548,1.50)  {$B_s{=}1014$};
\node[lbl] at (1765,1.50) {$B_r{=}1419$};
\node[lbl] at (3285,1.50) {$B_c{=}1622$};

\node[rowname] at (0,0.99) {realised on this item: $997$ tokens in $3$ rounds};
\fill[black!22, seg] (0,0.55)   rectangle (41,0.95);
\fill[blue!22, seg]  (41,0.55)  rectangle (159,0.95);
\fill[green!25, seg] (159,0.55) rectangle (901,0.95);
\fill[blue!22, seg]  (901,0.55) rectangle (997,0.95);
\draw[seg, dash pattern=on 2pt off 1.5pt, fill=white] (997,0.55) rectangle (4096,0.95);
\draw[seg] (0,0.55) rectangle (4096,0.95);
\node[lbl] at (2550,0.75) {unspent: $3{,}099$ of $4{,}096$};
\draw[-{Stealth[length=1.4mm]}, black!65] (100,0.30) -- (100,0.51);
\node[anchor=north, font=\tiny, align=center] at (100,0.28) {R1\\$118$};
\draw[-{Stealth[length=1.4mm]}, black!65] (530,0.30) -- (530,0.51);
\node[anchor=north, font=\tiny, align=center] at (530,0.28) {R2\\$742$};
\draw[-{Stealth[length=1.4mm]}, black!65] (949,0.30) -- (949,0.51);
\node[anchor=north, font=\tiny, align=center] at (1030,0.28) {R3\\$96$};

\node[rowname] at (0,-0.62) {\textsc{Uniform-}$16$ \emph{as specified}: overruns $B$ by $736$
  tokens ($18\%$) before prompt or answer};
\fill[green!25, seg] (0,-1.06) rectangle (4096,-0.66);
\pattern[pattern=north east lines, pattern color=red!55]
                     (4096,-1.06) rectangle (4832,-0.66);
\draw[seg]           (0,-1.06) rectangle (4832,-0.66);
\node[lbl] at (2048,-0.86) {$16\times302=4{,}832$ image tokens};

\draw[red!70!black, dashed, line width=0.6pt] (4096,-1.19) -- (4096,1.24);
\node[anchor=north, font=\tiny, text=red!70!black] at (4096,-1.21) {$B=4{,}096$};

\begin{scope}[shift={(0,-1.94)}]
  \fill[black!22, seg] (0,0) rectangle (150,0.20);
  \node[anchor=west, font=\tiny] at (200,0.10) {router $B_\rho$};
  \fill[blue!22, seg] (1150,0) rectangle (1300,0.20);
  \node[anchor=west, font=\tiny] at (1350,0.10) {$\cs$ (state)};
  \fill[green!25, seg] (2250,0) rectangle (2400,0.20);
  \node[anchor=west, font=\tiny] at (2450,0.10) {$\cv$ (retrieval)};
  \fill[orange!30, seg] (3500,0) rectangle (3650,0.20);
  \node[anchor=west, font=\tiny] at (3700,0.10) {reasoning $B_c$};
\end{scope}
\end{tikzpicture}
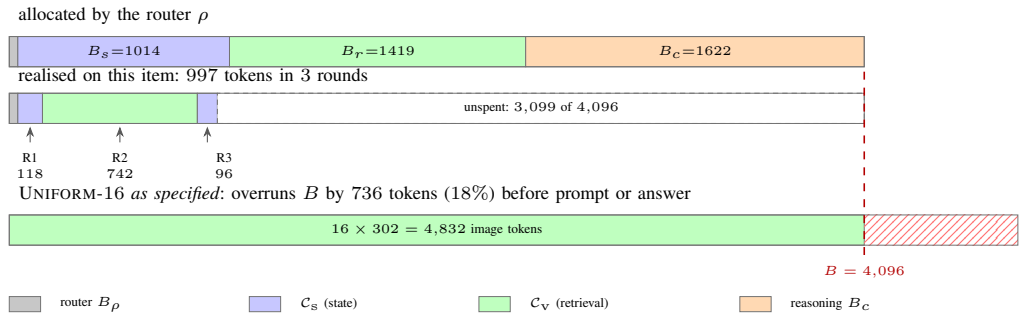
\caption{\textbf{Equal budget, drawn as a ledger} (the $d_q{=}3$ item worked through in
\S\ref{app:worked}, at the flagship's measured $302$ tokens per frame). Three things this makes
concrete. \emph{One:} the router's $(B_s,B_r,B_c)$ is a per-channel \emph{ceiling}, not a quota; this item finished three rounds on $997$ tokens, $24\%$ of $B$, and the unspent remainder is
not transferable to another item because $B$ is per decision. \emph{Two:} the round structure is
visible as three separate debits against two different channels, which is what
Theorem~\ref{thm:rounds} says a single-round agent cannot do at any $B$. \emph{Three:}
\textsc{Uniform-}$16$ \emph{as specified} does not fit inside $B$: sixteen frames cost more than
the whole budget before the prompt or the answer is counted, which is why \texttt{baselines.py}
sends the largest $w$ the budget affords and logs the realised value (\S\ref{sec:setup}).
Comparing against a \textsc{Uniform-}$16$ that silently overran $B$ by $18\%$ would have been the
easiest available way to manufacture the result this paper is testing.}
\label{fig:ledger}
\end{aspfigure}

Figure~\ref{fig:ledger} draws one decision's budget as a ledger, which is what ``equal budget'' means operationally: the same total, partitioned differently.

\subsection{Cost accounting}
What we can account for exactly is \emph{tokens}: every prompt, image and completion token is taken from the provider's own usage field and debited against $B$ per decision, which is the resource the whole paper is about, and it is why Table~\ref{tab:main} is an equal-budget comparison by construction rather than by post-hoc normalisation. What we \emph{cannot} account for on an API arm is FLOPs, KV-cache bytes or index storage per decision: the serving stack, batching and quantisation are not observable to us. The registered iso-compute check ($\le3\%$\pred{} FLOPs spread across methods at fixed $B$) therefore remains unevaluated, and is one of the things the local-serving arm of \S\ref{sec:limits} would buy. Note that the token accounting is not a proxy for it: two methods can spend the same tokens at different FLOPs if one of them sends images and the other text, and we do not claim otherwise. Per-image token costs are measured per backbone (\S\ref{sec:setup}), so the image/text mix of each method is at least visible in the ledger.

\subsection{Secondary experiments (registered, appendix-only)}
These are the experiments whose outcome refines the account rather than testing it, which is why they sit here and not in \S\ref{sec:exp}. Each is specified tightly enough to be run without further decisions. \textbf{Of the five, only part of S4 was run}; the rest are registered and outstanding, and are listed with their registered expectations so that a later version can be held to them. S1 and S2 each need a family of corpora at different $K$ and $N$ and a full re-run of the grid on each; S3 needs CMP items at depths we did not generate; S5 needs a second judge pass over every cell.

\textbf{(S1) Attribute-alphabet sweep.} $K\in\{4,8,16\}$ at fixed $N,d$, tracing the Theorem~\ref{thm:shannon} error curve $1-\frac{m/(Nd)+1}{\log_2K}$ for \textsc{State-only}. Registered: measured \textsc{State-only} RET accuracy stays below the per-$K$ ceiling in all three cells\pred{}, and the \emph{ordering} across $K$ matches the curve's monotonicity. A measured value \emph{above} a ceiling is a refutation of the bound's applicability (not of the bound), and would mean the state is carrying more than $m=L_s\beta$ bits, most likely because $\beta$, the bits-per-token estimate, is too low.

\textbf{(S2) Stream-length sweep.} $N\in\{100,200,400,800\}$ at fixed $w{=}16$, tracing the Theorem~\ref{thm:horizon} decay $\frac wN+(1-\frac wN)\frac1K$ for \textsc{Uniform-}$w$. Registered: \textsc{Uniform-}$w$ RET accuracy decays monotonically toward $1/K$ while \ASP{} stays flat within $3$ points\pred{}, the cleanest single picture of the horizon wall, and the one place where the theory predicts a \emph{shape} rather than a level.

\textbf{(S3) Depth-versus-rounds sweep.} $d_q\in\{2,3,4,5,6\}$ crossed with $R\in\{1,2,3,4\}$ on CMP. Registered: accuracy is governed by $\mathbb{1}[R\ge d_q]$ rather than by $B$, i.e.\ the $R{<}d_q$ cells sit near the Theorem~\ref{thm:rounds} single-round bound regardless of how much budget is poured into them\pred{}. This is the sharpest available test that rounds are a distinct resource from tokens.

\textbf{(S4) Sensitivity.} \emph{Partly run.} The retrieval mixing weight $\alpha$ over $\{0,0.25,0.5,0.6,0.75,1\}$ is swept in \texttt{analysis/retrieval\_recall.py} and costs nothing: $\alpha$ only reweights the cosine and BM25 terms of an index that is already built, so the entire sweep is local computation on cached captions and embeddings. $\alpha{=}1$ is vision-only retrieval and $\alpha{=}0$ is caption-only; the released sweep reports recall@$8$ at each setting, which localises how much of the index's competence comes from pixels and how much from text. The gate-threshold half of S4 is \textbf{not run} and cannot be on this corpus: $\tau_g{=}0$ is forced because SEW-Bench already emits one frame per discrete moment (\S\ref{sec:setup}), so a $\tau_g$ sweep needs a real sensor-rate stream. That leaves the registered concern, that \ASP{}'s advantage might be traceable to gate tuning, addressed only by the fact that the gate is the identity for every method here, which removes the confound rather than measuring it.

\textbf{(S5) Judge robustness.} Every EW free-form cell re-scored by a second, independently chosen judge. Registered: rank correlation between judges $\ge0.9$\pred{} and no reversal of any headline comparison. Because the headline tests are paired per-question deltas, a uniform judge-strictness offset cannot flip them; S5 tests the residual risk, which is judge behaviour that is \emph{non-uniform} across methods, specifically, a judge that rewards the longer answers agentic methods tend to produce.

\subsection{Qualitative protocol}
For one RET, one TRK and one CMP item the released records carry the full round-by-round access log (state reads, retrieval queries, debited tokens, and the emitted action per round) for \ASP{} alongside the transcript of the strongest single-channel baseline on the \emph{same} item. \S\ref{app:worked} fixes the format by working one item through by hand, on an EW-Bench-style query rather than a SEW one, so the format is legible independently of the corpus that was actually run; the measured logs are in \texttt{results/cells/} for every item, backbone and method.

\begin{aspnarrowtable}[!t]
\centering\footnotesize
\setlength{\tabcolsep}{3pt}
\caption{\textbf{Representational sufficiency (Assumption~\ref{ass:repr}), measured.} Left: single-keyframe attribute accuracy $1-\epsilon_0$, the RET question asked with the target frame handed to the model, which is the best case the assumption describes (chance $1/K=12.5\%$). Right: one composition step $1-\epsilon_1$, two in-context frames of the same class, which was seen earlier (chance $50\%$); scored items in parentheses. Attribute reading is comfortably sufficient at this scale, $0.71$--$1.00$ against a $0.125$ floor, so a near-chance grid cell is \emph{not} explained by an inability to read a frame that is in context. The one-step column is the weak one: 4 of 7 backbones sit at or below chance (\underline{underlined}), so for those the CMP family is limited by composition and not only by access, and Proposition~\ref{prop:complete}'s $1-d_q\epsilon_1$ bound is vacuous there. Two caveats on the right-hand column: it is binary, so it cannot separate inability from confusion over the words \textsf{first}/\textsf{second}, and each frame renders its own index in its header, so the task is legible in principle, since one backbone scores $1.00$ on it. Measured by \texttt{scripts/probe\_assumption.py}; per-item records released.}
\label{tab:probe}
\begin{tabular}{lcc}
\toprule
Backbone & $1-\epsilon_0$ (attribute) & $1-\epsilon_1$ (one step) \\
\midrule
Ministral~3B & 1.00 (32) & \underline{0.43 (14)} \\
Qwen3-VL-8B & 0.94 (32) & 0.81 (16) \\
Gemma~3~12B & 0.71 (31) & \underline{0.44 (16)} \\
Ministral~14B & 1.00 (32) & \underline{0.50 (16)} \\
Qwen3.8-27B & 1.00 (32) & 0.75 (16) \\
Qwen3-VL-30B-A3B & 0.94 (32) & \underline{0.50 (16)} \\
Gemma~4~31B & 1.00 (32) & 1.00 (16) \\
\bottomrule
\end{tabular}
\end{aspnarrowtable}

\ifdefined\ICRABuild
  \bibliographystyle{IEEEtran}
  \ifdefined\ICRACameraReady
    \begingroup
      \renewcommand{\baselinestretch}{0.982}\selectfont
      \bibliography{references}
    \endgroup
  \else
    \bibliography{references}

@inproceedings{prh,
  author={Huh, Minyoung and Cheung, Brian and Wang, Tongzhou and Isola, Phillip},
  title={The Platonic Representation Hypothesis},
  booktitle={ICML}, year={2024}}

@article{cch,
  author={Chen, Wenhui and Chen, Jianlin and Lin, Ziyao and Vong, Chi Man},
  title={The Capability Convergence Hypothesis: Capability from Access Structure, Not Scale},
  journal={arXiv preprint arXiv:2607.14144}, year={2026}}

@inproceedings{openeqa,
  author={Majumdar, Arjun and Ajay, Anurag and Zhang, Xiaohan and others},
  title={{OpenEQA}: Embodied Question Answering in the Era of Foundation Models},
  booktitle={CVPR}, year={2024}}

@inproceedings{vsibench,
  author={Yang, Jihan and Yang, Shusheng and Gupta, Anjali W. and Han, Rilyn and Fei-Fei, Li and Xie, Saining},
  title={Thinking in Space: How Multimodal Large Language Models See, Remember, and Recall Spaces},
  booktitle={CVPR}, year={2025}}

@inproceedings{egoschema,
  author={Mangalam, Karttikeya and Akshulakov, Raiymbek and Malik, Jitendra},
  title={{EgoSchema}: A Diagnostic Benchmark for Very Long-form Video Language Understanding},
  booktitle={NeurIPS}, year={2023}}

@inproceedings{eqa-das,
  author={Das, Abhishek and Datta, Samyak and Gkioxari, Georgia and Lee, Stefan and Parikh, Devi and Batra, Dhruv},
  title={Embodied Question Answering},
  booktitle={CVPR}, year={2018}}

@article{expressbench,
  author={Jiang, Kaixuan and others},
  title={Beyond the Destination: A Novel Benchmark for Exploration-Aware Embodied Question Answering},
  journal={arXiv preprint arXiv:2503.11117}, year={2025}}

@article{noisyeqa,
  author={Wu, Tao and others},
  title={{NoisyEQA}: Benchmarking Embodied Question Answering Against Noisy Queries},
  journal={arXiv preprint arXiv:2412.10726}, year={2024}}

@article{cityeqa,
  author={Zhao, Yong and others},
  title={{CityEQA}: A Hierarchical {LLM} Agent on Embodied Question Answering Benchmark in City Space},
  journal={arXiv preprint arXiv:2502.12532}, year={2025}}

@article{bridgeeqa,
  author={{BridgeEQA authors}},
  title={{BridgeEQA}: Virtual Embodied Agents for Real Bridge Inspections},
  journal={arXiv preprint arXiv:2511.12676}, year={2025}}

@article{erqa,
  author={{Gemini Robotics Team}},
  title={Gemini Robotics: Bringing {AI} into the Physical World},
  journal={arXiv preprint arXiv:2503.20020}, year={2025}}

@inproceedings{embodiedbench,
  author={Yang, Rui and others},
  title={{EmbodiedBench}: Comprehensive Benchmarking of Multimodal {LLMs} for Vision-Driven Embodied Agents},
  booktitle={ICML}, year={2025}}

@inproceedings{moviechat,
  author={Song, Enxin and others},
  title={{MovieChat}: From Dense Token to Sparse Memory for Long Video Understanding},
  booktitle={CVPR}, year={2024}}

@inproceedings{malmm,
  author={He, Bo and others},
  title={{MA-LMM}: Memory-Augmented Large Multimodal Model for Long-Term Video Understanding},
  booktitle={CVPR}, year={2024}}

@inproceedings{videoagent,
  author={Wang, Xiaohan and Zhang, Yuhui and Zohar, Orr and Yeung-Levy, Serena},
  title={{VideoAgent}: Long-form Video Understanding with Large Language Model as Agent},
  booktitle={ECCV}, year={2024}}

@inproceedings{llovi,
  author={Zhang, Ce and others},
  title={A Simple {LLM} Framework for Long-Range Video Question-Answering},
  booktitle={EMNLP}, year={2024}}

@inproceedings{videotree,
  author={Wang, Ziyang and others},
  title={{VideoTree}: Adaptive Tree-based Video Representation for {LLM} Reasoning on Long Videos},
  booktitle={CVPR}, year={2025}}

@inproceedings{goldfish,
  author={Ataallah, Kirolos and others},
  title={Goldfish: Vision-Language Understanding of Arbitrarily Long Videos},
  booktitle={ECCV}, year={2024}}

@article{flashvstream,
  author={Zhang, Haoji and others},
  title={{Flash-VStream}: Memory-Based Real-Time Understanding for Long Video Streams},
  journal={arXiv preprint arXiv:2406.08085}, year={2024}}

@inproceedings{fastv,
  author={Chen, Liang and others},
  title={An Image is Worth 1/2 Tokens After Layer 2: Plug-and-Play Inference Acceleration for Large Vision-Language Models},
  booktitle={ECCV}, year={2024}}

@article{prumerge,
  author={Shang, Yuzhang and others},
  title={{LLaVA-PruMerge}: Adaptive Token Reduction for Efficient Large Multimodal Models},
  journal={arXiv preprint arXiv:2403.15388}, year={2024}}

@article{tosa,
  author={{ToSA authors}},
  title={{ToSA}: Token Merging with Spatial Awareness},
  journal={arXiv preprint arXiv:2506.20066}, year={2025}}

@article{graphpad,
  author={{GraphPad authors}},
  title={{GraphPad}: Inference-Time Scene-Graph Updates for Embodied Question Answering},
  journal={arXiv preprint}, year={2025}}

@inproceedings{sayplan,
  author={Rana, Krishan and others},
  title={{SayPlan}: Grounding Large Language Models using {3D} Scene Graphs for Scalable Robot Task Planning},
  booktitle={CoRL}, year={2023}}

@inproceedings{conceptgraphs,
  author={Gu, Qiao and others},
  title={{ConceptGraphs}: Open-Vocabulary {3D} Scene Graphs for Perception and Planning},
  booktitle={ICRA}, year={2024}}

@article{snell,
  author={Snell, Charlie and Lee, Jaehoon and Xu, Kelvin and Kumar, Aviral},
  title={Scaling {LLM} Test-Time Compute Optimally Can Be More Effective than Scaling Model Parameters},
  journal={arXiv preprint arXiv:2408.03314}, year={2024}}

@article{s1,
  author={Muennighoff, Niklas and others},
  title={s1: Simple Test-Time Scaling},
  journal={arXiv preprint arXiv:2501.19393}, year={2025}}

@inproceedings{feng-cot,
  author={Feng, Guhao and others},
  title={Towards Revealing the Mystery behind Chain of Thought: A Theoretical Perspective},
  booktitle={NeurIPS}, year={2023}}

@inproceedings{merrill-cot,
  author={Merrill, William and Sabharwal, Ashish},
  title={The Expressive Power of Transformers with Chain of Thought},
  booktitle={ICLR}, year={2024}}

@article{merrill-sat,
  author={Merrill, William and Sabharwal, Ashish},
  title={The Parallelism Tradeoff: Limitations of Log-Precision Transformers},
  journal={TACL}, year={2023}}

@inproceedings{shortcuts,
  author={Liu, Bingbin and others},
  title={Transformers Learn Shortcuts to Automata},
  booktitle={ICLR}, year={2023}}

@inproceedings{illusion-state,
  author={Merrill, William and Petty, Jackson and Sabharwal, Ashish},
  title={The Illusion of State in State-Space Models},
  booktitle={ICML}, year={2024}}

@inproceedings{mamba,
  author={Gu, Albert and Dao, Tri},
  title={Mamba: Linear-Time Sequence Modeling with Selective State Spaces},
  booktitle={COLM}, year={2024}}

@article{jamba,
  author={Lieber, Opher and others},
  title={Jamba: A Hybrid Transformer-Mamba Language Model},
  journal={arXiv preprint arXiv:2403.19887}, year={2024}}

@article{griffin,
  author={De, Soham and others},
  title={Griffin: Mixing Gated Linear Recurrences with Local Attention for Efficient Language Models},
  journal={arXiv preprint arXiv:2402.19427}, year={2024}}

@inproceedings{routellm,
  author={Ong, Isaac and others},
  title={{RouteLLM}: Learning to Route {LLMs} with Preference Data},
  booktitle={ICLR}, year={2025}}

@article{frugalgpt,
  author={Chen, Lingjiao and Zaharia, Matei and Zou, James},
  title={{FrugalGPT}: How to Use Large Language Models While Reducing Cost and Improving Performance},
  journal={arXiv preprint arXiv:2305.05176}, year={2023}}

@inproceedings{socratic,
  author={Zeng, Andy and others},
  title={Socratic Models: Composing Zero-Shot Multimodal Reasoning with Language},
  booktitle={ICLR}, year={2023}}

@inproceedings{alanavlm,
  author={Suglia, Alessandro and others},
  title={{AlanaVLM}: A Multimodal Embodied {AI} Foundation Model for Egocentric Video Understanding},
  booktitle={Findings of EMNLP}, year={2024}}

@techreport{qwen38,
  author={{Qwen Team}},
  title={{Qwen3.8-27B}: Dense, Natively Multimodal, {Apache-2.0}},
  institution={Alibaba Group}, type={Release report}, year={2026}, month={aug}}

@techreport{gemma4,
  author={{Gemma Team}},
  title={Gemma 4 Model Family ({E2B}/{E4B}/{26B-MoE}/{31B})},
  institution={Google DeepMind}, type={Technical report}, year={2026}, month={apr}}

@article{qwen3omni,
  author={Xu, Jin and others},
  title={{Qwen3-Omni} Technical Report},
  journal={arXiv preprint arXiv:2509.17765}, year={2025}}

@techreport{qwen3vl,
  author={{Qwen Team}},
  title={{Qwen3-VL} Technical Report},
  institution={Alibaba Group}, year={2025}}

@article{gemma3,
  author={{Gemma Team}},
  title={{Gemma 3} Technical Report},
  journal={arXiv preprint arXiv:2503.19786}, year={2025}}

@techreport{ministral,
  author={{Mistral AI}},
  title={{Ministral} 3B: A Compact Multimodal Model},
  institution={Mistral AI}, type={Model card}, year={2025}}

@article{qwen25vl,
  author={Bai, Shuai and others},
  title={{Qwen2.5-VL} Technical Report},
  journal={arXiv preprint arXiv:2502.13923}, year={2025}}

@article{phi4mm,
  author={{Microsoft}},
  title={{Phi-4-Multimodal} Technical Report},
  journal={arXiv preprint arXiv:2503.01743}, year={2025}}

@article{minicpmv,
  author={Yao, Yuan and others},
  title={{MiniCPM-V}: A {GPT-4V} Level {MLLM} on Your Phone},
  journal={arXiv preprint arXiv:2408.01800}, year={2024}}

@inproceedings{siglip,
  author={Zhai, Xiaohua and Mustafa, Basil and Kolesnikov, Alexander and Beyer, Lucas},
  title={Sigmoid Loss for Language Image Pre-Training},
  booktitle={ICCV}, year={2023}}

@inproceedings{hm3d,
  author={Ramakrishnan, Santhosh Kumar and others},
  title={{Habitat-Matterport 3D Dataset (HM3D)}: 1000 Large-scale {3D} Environments for Embodied {AI}},
  booktitle={NeurIPS Datasets and Benchmarks}, year={2021}}

@inproceedings{scannet,
  author={Dai, Angela and others},
  title={{ScanNet}: Richly-Annotated {3D} Reconstructions of Indoor Scenes},
  booktitle={CVPR}, year={2017}}

@article{barrington,
  author={Barrington, David A.},
  title={Bounded-Width Polynomial-Size Branching Programs Recognize Exactly Those Languages in {NC$^1$}},
  journal={Journal of Computer and System Sciences}, volume={38}, number={1}, year={1989}}

@inproceedings{egagent,
  author={Rege, Aniket and Sadhu, Arka and Li, Yuliang and Li, Kejie and Vinayak, Ramya Korlakai and Chai, Yuning and Lee, Yong Jae and Kim, Hyo Jin},
  title={Agentic Very Long Video Understanding},
  booktitle={ACL}, year={2026}, note={arXiv:2601.18157}}

@article{worldmm,
  author={Yeo, Jihan and others},
  title={{WorldMM}: Dynamic Multimodal Memory Agent for Long Video Reasoning},
  journal={arXiv preprint}, year={2026}}

@inproceedings{vgent,
  author={Shen, Xiaoqian and Zhang, Wenxuan and Chen, Jun and Elhoseiny, Mohamed},
  title={Vgent: Graph-based Retrieval-Reasoning-Augmented Generation for Long Video Understanding},
  booktitle={NeurIPS}, year={2025}}

@inproceedings{videorag,
  author={Ren, Xubin and Xu, Lingrui and Xia, Long and Wang, Shuaiqiang and Yin, Dawei and Huang, Chao},
  title={{VideoRAG}: Retrieval-Augmented Generation with Extreme Long-Context Videos},
  booktitle={ACM SIGKDD}, year={2026}}

@article{streameqa,
  author={{StreamEQA authors}},
  title={{StreamEQA}: Towards Streaming Video Understanding for Embodied Scenarios},
  journal={arXiv preprint}, year={2025}}

@article{longvitu,
  author={{LongViTU authors}},
  title={{LongViTU}: Instruction Tuning for Long-Form Video Understanding},
  journal={arXiv preprint}, year={2025}}

@inproceedings{pointerchasing,
  author={Papadimitriou, Christos H. and Sipser, Michael},
  title={Communication Complexity},
  booktitle={STOC}, year={1982}}

@article{nisanwigderson,
  author={Nisan, Noam and Wigderson, Avi},
  title={Rounds in Communication Complexity Revisited},
  journal={SIAM Journal on Computing}, volume={22}, number={1}, year={1993}}

@article{kaplan,
  author={Kaplan, Jared and McCandlish, Sam and Henighan, Tom and others},
  title={Scaling Laws for Neural Language Models},
  journal={arXiv preprint arXiv:2001.08361}, year={2020}}

@inproceedings{chinchilla,
  author={Hoffmann, Jordan and others},
  title={Training Compute-Optimal Large Language Models},
  booktitle={NeurIPS}, year={2022}}

@inproceedings{wei-cot,
  author={Wei, Jason and Wang, Xuezhi and Schuurmans, Dale and others},
  title={Chain-of-Thought Prompting Elicits Reasoning in Large Language Models},
  booktitle={NeurIPS}, year={2022}}

@inproceedings{react,
  author={Yao, Shunyu and Zhao, Jeffrey and Yu, Dian and others},
  title={{ReAct}: Synergizing Reasoning and Acting in Language Models},
  booktitle={ICLR}, year={2023}}

@inproceedings{reflexion,
  author={Shinn, Noah and Cassano, Federico and Gopinath, Ashwin and Narasimhan, Karthik and Yao, Shunyu},
  title={Reflexion: Language Agents with Verbal Reinforcement Learning},
  booktitle={NeurIPS}, year={2023}}

@inproceedings{toolformer,
  author={Schick, Timo and Dwivedi-Yu, Jane and Dess{\`i}, Roberto and others},
  title={Toolformer: Language Models Can Teach Themselves to Use Tools},
  booktitle={NeurIPS}, year={2023}}

@inproceedings{lewis-rag,
  author={Lewis, Patrick and Perez, Ethan and Piktus, Aleksandra and others},
  title={Retrieval-Augmented Generation for Knowledge-Intensive {NLP} Tasks},
  booktitle={NeurIPS}, year={2020}}

@inproceedings{dpr,
  author={Karpukhin, Vladimir and O{\u{g}}uz, Barlas and Min, Sewon and others},
  title={Dense Passage Retrieval for Open-Domain Question Answering},
  booktitle={EMNLP}, year={2020}}

@article{memgpt,
  author={Packer, Charles and Wooders, Sarah and Lin, Kevin and others},
  title={{MemGPT}: Towards {LLMs} as Operating Systems},
  journal={arXiv preprint arXiv:2310.08560}, year={2023}}

@inproceedings{hipporag,
  author={Guti{\'e}rrez, Bernal Jim{\'e}nez and Shu, Yiheng and Gu, Yu and Yasunaga, Michihiro and Su, Yu},
  title={{HippoRAG}: Neurobiologically Inspired Long-Term Memory for Large Language Models},
  booktitle={NeurIPS}, year={2024}}

@inproceedings{ringattention,
  author={Liu, Hao and Zaharia, Matei and Abbeel, Pieter},
  title={Ring Attention with Blockwise Transformers for Near-Infinite Context},
  booktitle={ICLR}, year={2024}}

@article{lwm,
  author={Liu, Hao and Yan, Wilson and Zaharia, Matei and Abbeel, Pieter},
  title={World Model on Million-Length Video and Language with Blockwise {RingAttention}},
  journal={arXiv preprint arXiv:2402.08268}, year={2024}}

@article{longva,
  author={Zhang, Peiyuan and others},
  title={Long Context Transfer from Language to Vision},
  journal={arXiv preprint arXiv:2406.16852}, year={2024}}

@inproceedings{palme,
  author={Driess, Danny and Xia, Fei and Sajjadi, Mehdi S. M. and others},
  title={{PaLM-E}: An Embodied Multimodal Language Model},
  booktitle={ICML}, year={2023}}

@inproceedings{rt2,
  author={Brohan, Anthony and others},
  title={{RT-2}: Vision-Language-Action Models Transfer Web Knowledge to Robotic Control},
  booktitle={CoRL}, year={2023}}

@inproceedings{openvla,
  author={Kim, Moo Jin and others},
  title={{OpenVLA}: An Open-Source Vision-Language-Action Model},
  booktitle={CoRL}, year={2024}}

@inproceedings{habitat,
  author={Savva, Manolis and others},
  title={Habitat: A Platform for Embodied {AI} Research},
  booktitle={ICCV}, year={2019}}

@inproceedings{alfred,
  author={Shridhar, Mohit and others},
  title={{ALFRED}: A Benchmark for Interpreting Grounded Instructions for Everyday Tasks},
  booktitle={CVPR}, year={2020}}

@inproceedings{ego4d,
  author={Grauman, Kristen and others},
  title={{Ego4D}: Around the World in 3{,}000 Hours of Egocentric Video},
  booktitle={CVPR}, year={2022}}

@inproceedings{egolife,
  author={Yang, Jingkang and others},
  title={{EgoLife}: Towards Egocentric Life Assistant},
  booktitle={CVPR}, year={2025}}

@inproceedings{spatialvlm,
  author={Chen, Boyuan and others},
  title={{SpatialVLM}: Endowing Vision-Language Models with Spatial Reasoning Capabilities},
  booktitle={CVPR}, year={2024}}

@inproceedings{videomme,
  author={Fu, Chaoyou and others},
  title={{Video-MME}: The First-Ever Comprehensive Evaluation Benchmark of Multi-modal {LLMs} in Video Analysis},
  booktitle={CVPR}, year={2025}}
  \fi
\fi

\end{document}